\documentclass[12pt,a4paper]{article}

\usepackage{amsmath,amsfonts,amssymb,amsthm}
\usepackage{geometry}
\usepackage{booktabs}
\usepackage{array}
\usepackage{xcolor}
\usepackage{hyperref}
\usepackage{tcolorbox}
\usepackage{mathtools}
\usepackage{bm}
\usepackage{natbib}     
\usepackage{algorithm}
\usepackage{algpseudocode}
\usepackage{graphicx}
\usepackage{subfig}
\usepackage[section]{placeins}

\tcbuselibrary{theorems,skins,breakable}

\theoremstyle{plain}
\newtheorem{theorem}{Theorem}
\newtheorem{lemma}{Lemma}
\newtheorem{proposition}{Proposition}

\newtheorem{assumption}{Assumption}

\theoremstyle{remark}
\newtheorem{remark}{Remark}

\hypersetup{
  colorlinks = true,
  linkcolor  = blue!60!black,
  citecolor  = green!50!black,
  urlcolor   = blue!70!black
}

\title{Invariant Stochastic Filtering on $SE(3)$ for Serial Rigid
  Manipulators: Modular Design with Physically Separated
  IMU and Process Noise}

\author{%
  S.~Yaqubi%
  \thanks{This work was supported by the Research Council of
    Finland under the project ``Nonlinear PDE-model-based control
    of flexible manipulators'' (Grant No.~355664).
    (Corresponding author: S.~Yaqubi.)}
  \thanks{S.~Yaqubi, M.~Hejrati and J.~Mattila are with the Department of
    Automation Technology and Mechanical Engineering, Tampere
    University, Korkeakoulunkatu~6, 33720~Tampere, Finland
    (e-mails: \texttt{sadeq.yaqubi@tuni.fi}, \texttt{mahdi.hejrati@tuni.fi}, \texttt{sadeq.yaqubi@tuni.fi},
    \texttt{jouni.mattila@tuni.fi}).}
      \and
      M.~Hejrati
  \and
  J.~Mattila
}

\date{}

\begin{document}

\maketitle

\begin{abstract}
An invariant extended Kalman filter (IEKF) is developed for
state estimation of serial rigid manipulators with an arbitrary
number of links, formulated entirely within the Lie group $SE(3)$.
The group-affine property of the kinematic equations makes the
linearised error dynamics autonomous, so the Riccati equation
governs the true error covariance rather than a local
approximation. A physically separated noise model treats
gyroscope and accelerometer channels independently: the
accelerometer provides translational twist via gravity-compensated
integration, yielding a measurement covariance that scales with
the sample interval in exact analogy with process noise
discretisation; a state-dependent Coriolis noise term captures
gyroscope noise propagating through the nonlinear dynamics,
vanishing at rest and growing with twist magnitude. The filter is
structured as a modular chain of per-link IEKFs in which the
predicted covariance of each link depends on its predecessor only
through the Adjoint-transformed posterior, giving linear
computational cost in link count. Exponential ultimate
boundedness in mean square is established via a Lie algebra
Lyapunov function, with per-link bounds chained through the
Adjoint operator norm to yield a stability certificate that is
modular and scalable to arbitrary chain length. The state is augmented with gyroscope biases without loss of the autonomous error dynamics or the modular structure. The design is
validated in a three-dimensional simulation of a two-link chain with
encoder and IMU fusion, in a software-in-the-loop study of an industrial
hydraulic rammer-boom manipulator in which the encoder row is replaced
by a gravity-referenced virtual encoder for IMU-only estimation.
\end{abstract}

\section{Introduction}
\label{sec:1}

State estimation for serial rigid manipulators requires
geometric consistency and stochastic rigour simultaneously:
the configuration of each link is an element of the Lie group
$SE(3)$, not a vector, yet the classical Kalman filter operates
in a flat vector space. Precise closed-loop control depends on
accurate knowledge of the body-fixed twist and pose of each link,
estimated from noisy IMU and joint encoder data in real time.
The nonlinear geometry of $SE(3)$ invalidates coordinate-based
linearisation at large angles, and the heterogeneous physical
origins of IMU and process noise require a treatment that most
existing filter formulations do not provide.

Euler-angle and quaternion EKFs are well
established~\citep{thrun2005,welch1995} but introduce
unavoidable deficiencies: gimbal lock, normalisation constraints,
and a linearisation consistency error that does not vanish at
convergence~\citep{barrau2017,Bhat2000}. Geometric approaches
have been pursued, including multiplicative quaternion
filters~\citep{Wu2025,Mitikiri2021}, dual-quaternion
formulations~\citep{Khalifa2026}, motor-algebra
EKFs~\citep{Bayro-Corrochano2000}, and Lie-group signal
processing~\citep{Kumar2024,Li2024}, but these address single
bodies and do not propagate uncertainty modularly across chains.

Filters on Lie groups place geometric intuition within a
rigorous stochastic framework. Probability densities on Lie
groups and kinematic state estimation via the exponential map
were developed in~\citet{chirikjian2012,Park2008}. The invariant
EKF (IEKF)~\citep{barrau2017,barrau2018} exploits a
group-affine condition under which linearised error dynamics are
autonomous, so the Riccati equation governs the true covariance.
The IEKF has since been extended to matrix Lie
groups~\citep{Phogat2020}, iterated and equivariant
variants~\citep{Lu2025,Ge2026}, resilient distributed
estimation~\citep{Zhang2024}, nonholonomic
systems~\citep{Rigo2023}, and unscented
formulations~\citep{Li2025}, with applications in legged
odometry~\citep{hartley2020}, inertial
navigation~\citep{heo2018,forster2017}, and parallel
mechanisms~\citep{Yue2025}. Despite this breadth, extension
to serial multibody chains --- where link states couple through
kinematic constraints and uncertainty must propagate modularly
--- has received limited attention. Existing work treats either
a single body or the full chain as a monolithic state, the
latter incurring cubic update cost in link count.

A further gap concerns IMU noise. Gyroscope noise enters the
rotational twist channel directly; accelerometer noise does not,
since the accelerometer observes the time derivative of
translational velocity. Recovery by integration introduces a
sample-interval scaling on the effective noise covariance absent
in the gyroscope channel, and gyroscope noise propagates through
the Coriolis term creating a state-dependent process noise
contribution that existing geometric filters do not address.

This paper makes four contributions to geometric stochastic
estimation for serial rigid manipulators: a physically separated IMU noise model with state-dependent Coriolis
noise and gyroscope bias estimation~(C1); a
modular chain of per-link IEKFs with linear total cost via
Adjoint covariance propagation~(C2); derivation of the complete
IEKF including the mixed pose-twist observation Jacobian~(C3);
and a chained exponential ultimate boundedness certificate via
a Lie algebra Lyapunov function~(C4). Because the encoder enters only through one row of the observation
model, the same filter yields IMU-only estimation when that row is
replaced by a joint angle synthesized from the accelerometers of
adjacent links; this configuration is demonstrated in a
software-in-the-loop study of an industrial hydraulic manipulator.

The paper is organised as follows. Section~\ref{sec:2} presents
the system model and sensor measurement models.
Section~\ref{sec:3} gives the stochastic plant and observation models.
Section~\ref{sec:4} defines the left-invariant error, and
Section~\ref{sec:5} derives its linearized dynamics and noise channels.
Section~\ref{sec:6} derives the predicted covariance and the modular
synthesis. Section~\ref{sec:7} derives the Kalman update with geometric
retraction. Section~\ref{sec:8} establishes exponential ultimate boundedness.
Section~\ref{sec:9} presents a simulation and a software-in-the-loop
example.
\section{Problem Statement, Notation, and Geometric Preliminaries}
\label{sec:2}

\subsection{System Description}
\label{sec:2.1}

A serial rigid manipulator with $n$ links is considered, where
each link is connected to adjacent links via revolute joints.
Unlike the flexible case studied
in~\citep{Yaqubi2026,Yaqubi20262}, each link $i$ is treated as
a perfectly rigid body undergoing large-amplitude motion on the
Lie group $SE(3)$. The interest here is the design and analysis
of a geometric stochastic filter for the coupled rigid-body
state, motivated by the practical requirement that physically
consistent state estimation --- respecting the $SE(3)$ manifold
structure --- precedes model-based control.

\subsection{Notation and Geometric Preliminaries}
\label{sec:2.2}

The notation and geometric conventions follow~\citep{Yaqubi2026}.
Each link $i$ carries a body-fixed frame
${}^{i}\!\mathcal{F}$. The rigid-body configuration is
$g_i \in SE(3)$, with rotation matrix $\mathbf{R}_{oi} \in
SO(3)$ and position $\mathbf{p}_i \in \mathbb{R}^3$ describing
the pose of ${}^{i}\!\mathcal{F}$ relative to the inertial
frame ${}^{0}\!\mathcal{F}$

\begin{align}
  g_i =
  \begin{bmatrix}
    \mathbf{R}_{oi} & \mathbf{p}_i \\
    \mathbf{0}^\top & 1
  \end{bmatrix} \in SE(3).
  \label{eq:g}
\end{align}

The body-fixed twist of link $i$ is an element of the Lie
algebra $\mathfrak{se}(3)$~\citep{murray1994,lynch2017}:

\begin{align}
  \mathbf{V}_i =
  \begin{bmatrix}
    \boldsymbol{\omega}_i \\
    \mathbf{v}_i
  \end{bmatrix}
  \in \mathfrak{se}(3) \cong \mathbb{R}^6,
  \label{eq:twist}
\end{align}

where $\boldsymbol{\omega}_i \in \mathbb{R}^3$ is the
body-fixed angular velocity and $\mathbf{v}_i \in \mathbb{R}^3$
is the body-fixed linear velocity. The configuration kinematics
on $SE(3)$ are governed by~\citep{murray1994,lynch2017}

\begin{align}
  \dot{g}_i = g_i \mathbf{V}_i,
  \label{eq:gdot}
\end{align}

where the product $g_i\mathbf{V}_i$ uses the standard
convention that $\mathbf{V}_i$ denotes its matrix Lie algebra
representative when multiplied by a group
element~\citep{murray1994,lynch2017}. Throughout this paper
$\hat{(\cdot)}$ denotes filter estimates only. The rotation
matrix $\mathbf{R}_{oi}$ transforms vectors from
${}^{i}\!\mathcal{F}$ to ${}^{0}\!\mathcal{F}$, chaining as
$\mathbf{R}_{oi} = \mathbf{R}_{o(i-1)}\mathbf{R}_{(i-1)i}$.

The Adjoint representation $\mathbf{Ad}_{ba} \in
\mathbb{R}^{6\times6}$ transforms twists between frames, and
the co-Adjoint $\mathbf{Ad}_{ba}^{-\top}$ transforms
wrenches~\citep{murray1994,lynch2017}

\begin{align}
  \mathbf{Ad}_{ba} =
  \begin{bmatrix}
    \mathbf{R}_{ba} & \mathbf{0} \\
    [{}^{b}\mathbf{r}_a]^\times\,\mathbf{R}_{ba} &
    \mathbf{R}_{ba}
  \end{bmatrix} \in \mathbb{R}^{6\times6}.
  \label{eq:Ad}
\end{align}

The small adjoint $\mathrm{ad}_{\mathbf{V}_i} \in
\mathbb{R}^{6\times6}$ is the matrix representation of the Lie
bracket on $\mathfrak{se}(3)$~\citep{murray1994}

\begin{align}
  \mathrm{ad}_{\mathbf{V}_i} =
  \begin{bmatrix}
    [\boldsymbol{\omega}_i]^\times & \mathbf{0} \\
    [\mathbf{v}_i]^\times         & [\boldsymbol{\omega}_i]^\times
  \end{bmatrix},
  \label{eq:adv}
\end{align}

where $[\cdot]^\times \in \mathfrak{so}(3)$ denotes the
skew-symmetric matrix of a vector. Wrenches
$\boldsymbol{\mathcal{W}}_i \in \mathfrak{se}^*(3)$ pair with
twists via the natural power product
$\boldsymbol{\mathcal{W}}_i^\top\mathbf{V}_i =
\boldsymbol{\tau}_i^\top\boldsymbol{\omega}_i +
\mathbf{F}_i^\top\mathbf{v}_i$, representing instantaneous
mechanical power. Inter-link interaction wrenches
$\boldsymbol{\mathcal{W}}_{J,i} \in \mathfrak{se}^*(3)$
represent the constraint forces transmitted through joint $i$
and satisfy the Newton third-law pairing via the co-Adjoint
map~\citep{Yaqubi20262}.

The twist of link $i$ propagates from link $i-1$ through the
Adjoint map of the joint configuration~\citep{lynch2017}

\begin{align}
  \mathbf{V}_i =
  \mathbf{Ad}_{i(i-1)}\mathbf{V}_{i-1}
  + \mathbf{V}_{i,\mathrm{rel}},
  \label{eq:Vprop}
\end{align}

where $\mathbf{V}_{i,\mathrm{rel}} \in \mathfrak{se}(3)$ is
the relative joint twist. For a revolute joint with axis
$\hat{\mathbf{z}}_i$, this reduces to
$\mathbf{V}_{i,\mathrm{rel}} =
[\hat{\mathbf{z}}_i^\top,\,\mathbf{0}^\top]^\top\dot{q}_i$.

\subsection{Rigid-Body Dynamic Model}
\label{sec:2.3}

The deterministic Newton--Euler equation for rigid link $i$ on
$SE(3)$ is~\citep{featherstone2008}

\begin{align}
  \mathbf{M}_i\dot{\mathbf{V}}_i
  + \mathrm{ad}_{\mathbf{V}_i}^\top\mathbf{M}_i\mathbf{V}_i
  + \boldsymbol{\mathcal{W}}_{J,i}
  = \boldsymbol{\mathcal{W}}_i,
  \label{eq:newton_euler}
\end{align}

where $\mathbf{M}_i \in \mathbb{R}^{6\times6}$ is the symmetric
positive-definite spatial inertia matrix

\begin{align}
  \mathbf{M}_i =
  \begin{bmatrix}
    \mathbf{I}_{b_i}              & -m_i[\mathbf{r}_{c_i}]^\times \\
    m_i[\mathbf{r}_{c_i}]^\times  & m_i\mathbf{I}_3
  \end{bmatrix} \succ 0,
  \label{eq:Mi}
\end{align}

with $m_i$ the link mass, $\mathbf{I}_{b_i}$ the rotational
inertia tensor expressed in ${}^{i}\!\mathcal{F}$, and
$\mathbf{r}_{c_i} \in \mathbb{R}^3$ the centre-of-mass
position in ${}^{i}\!\mathcal{F}$. The term
$\mathrm{ad}_{\mathbf{V}_i}^\top\mathbf{M}_i\mathbf{V}_i$
collects all Coriolis and centripetal effects arising from the
adjoint action on $\mathfrak{se}(3)$, and
$\boldsymbol{\mathcal{W}}_i$ is the applied control
wrench~\citep{lynch2017}.

\begin{remark}
  The interaction wrench $\boldsymbol{\mathcal{W}}_{J,i}$ is a holonomic
  constraint force determined by the joint kinematics and the applied
  wrenches; it carries no stochastic component, and joint imperfections
  are absorbed into the disturbance wrench of Section~\ref{sec:3}.
\end{remark}

\section{Stochastic Dynamics and Observation Model}
\label{sec:3}

This section derives the stochastic dynamic model of the manipulator and the observation models for the IMU and joint encoder. Together with the geometric preliminaries of Section~\ref{sec:2}, these constitute the complete problem statement for the filter design.

\subsection{Stochastic Plant Model}
\label{sec:3.1}

\subsubsection*{Sensor Roles and State Observability}

The filter state $(g_i, \mathbf{V}_i)$ has two components
requiring different sensors.

\textbf{Body-fixed twist $\mathbf{V}_i$} is measured directly
by the IMU on link $i$. The gyroscope measures the rotational
component $\boldsymbol{\omega}_i$ at each step. The
accelerometer, after gravity compensation and integration over
$[t_k, t_{k+1}]$, provides the change in translational
velocity $\Delta\mathbf{v}_i$. Together they yield a noisy
measurement of the full twist $\mathbf{V}_i \in
\mathfrak{se}(3)$ at each time step (derived in
Section~\ref{sec:3.1.2}).

\textbf{Pose $g_i \in SE(3)$} cannot be measured directly by
the IMU. Integrating $\dot{g}_i = g_i\mathbf{V}_i$ forward
from a noisy $\mathbf{V}_i$ accumulates drift. The joint
encoder on joint $i$ measures the joint angle $q_i$, which
constrains the rotational degree of freedom of $g_i$ along
the joint axis $\hat{\mathbf{z}}_i$. The remaining pose
degrees of freedom of $g_i$ are determined by the known rigid
geometry of the chain via the kinematic prediction
equation~\eqref{eq:g_pred_chain}: given the estimated pose of
link $i-1$ and the encoder reading $q_i$, the pose $g_i$ is
algebraically determined. Any measurement of $q_i(g_i)$ can occupy this row of the observation
model. Example~2 replaces the encoder by a joint angle synthesized from
the accelerometers of the two adjacent links, which yields IMU-only
estimation without further change to the filter.

\textbf{Innovation formation.} At each filter step, the
Kalman innovation has two parts: the twist innovation
$\mathbf{y}_i^{\mathrm{IMU}} - \hat{\mathbf{V}}_i^-$
(IMU measurement minus model-predicted twist) and the joint
angle innovation $y_i^{\mathrm{enc}} - q_i(\hat{g}_i^-)$
(encoder measurement minus the joint angle predicted from the
kinematic chain). Both are derived formally in
Section~\ref{sec:7.1}.

\subsubsection{Process Noise --- Disturbance Wrench}

Unmodeled dynamics, joint friction, payload uncertainty, and
base vibration are represented as a disturbance wrench
$\mathbf{W}_i^w(t) \in \mathfrak{se}^*(3)$ acting additively
on the control channel

\begin{align}
  \mathbf{M}_i\dot{\mathbf{V}}_i
  + \mathrm{ad}_{\mathbf{V}_i}^\top\mathbf{M}_i\mathbf{V}_i
  + \boldsymbol{\mathcal{W}}_{J,i}
  = \boldsymbol{\mathcal{W}}_i + \mathbf{W}_i^w(t).
  \label{eq:stoch_eom}
\end{align}

The disturbance wrench is modelled as zero-mean
continuous-time white noise with power spectral density matrix
$\mathbf{Q}_{c,i}^w \in \mathbb{R}^{6\times6}$,
$\mathbf{Q}_{c,i}^w \succ 0$:

\begin{align}
  \mathbb{E}[\mathbf{W}_i^w(t)] &= \mathbf{0}, \notag \\
  \mathbb{E}\!\left[\mathbf{W}_i^w(t)\,
    {\mathbf{W}_i^w(s)}^\top\right]
  &= \mathbf{Q}_{c,i}^w\,\delta(t-s).
  \label{eq:Qcw}
\end{align}

The block structure of $\mathbf{Q}_{c,i}^w$ is

\begin{align}
  \mathbf{Q}_{c,i}^w =
  \begin{bmatrix}
    \mathbf{Q}_{\tau\tau,i} & \mathbf{Q}_{\tau F,i} \\
    \mathbf{Q}_{F\tau,i}   & \mathbf{Q}_{FF,i}
  \end{bmatrix},
  \label{eq:Qcw_block}
\end{align}

where $\mathbf{Q}_{\tau\tau,i},\mathbf{Q}_{FF,i} \in
\mathbb{R}^{3\times3}$ are the torque and force disturbance
spectral densities. For joints where actuator torque noise
dominates, $\mathbf{Q}_{\tau\tau,i} \gg \mathbf{Q}_{FF,i}$
and cross-terms are negligible. Inverting $\mathbf{M}_i$
in~\eqref{eq:stoch_eom}, the twist dynamics take the form:

\begin{align}
  \dot{\mathbf{V}}_i
  = \underbrace{
      \mathbf{M}_i^{-1}\!\left(
        \boldsymbol{\mathcal{W}}_i
        - \mathrm{ad}_{\mathbf{V}_i}^\top\mathbf{M}_i\mathbf{V}_i
        - \boldsymbol{\mathcal{W}}_{J,i}
      \right)
    }_{\mathbf{f}_i^V(g_i,\mathbf{V}_i)}
  + \underbrace{\mathbf{M}_i^{-1}}_{\mathbf{G}_i^w}
    \mathbf{W}_i^w(t),
  \label{eq:Vdot_stoch}
\end{align}

identifying the drift $\mathbf{f}_i^V \in \mathbb{R}^6$
and the noise input matrix $\mathbf{G}_i^w = \mathbf{M}_i^{-1}
\in \mathbb{R}^{6\times6}$. The complete SDE form --- with
$\mathbf{W}_i^w$ replaced by a Wiener increment --- is stated
after all measurement noise sources are introduced,
in~\eqref{eq:sde_g}--\eqref{eq:sde_V}.

\subsubsection{Measurement Model --- IMU}
\label{sec:3.1.2}

The IMU on link $i$ comprises a gyroscope and an
accelerometer. The gyroscope directly measures the body-fixed
angular velocity:

\begin{align}
  \mathbf{y}_i^{\mathrm{gyro}}
  = \boldsymbol{\omega}_i + \boldsymbol{\eta}_{\omega,i},
  \qquad
  \boldsymbol{\eta}_{\omega,i} \sim
  \mathcal{N}(\mathbf{0},\,\mathbf{N}_{\omega,i}),
  \label{eq:gyro_meas}
\end{align}

where $\mathbf{N}_{\omega,i} \in \mathbb{R}^{3\times3}$ is
obtained directly from the IMU datasheet. The accelerometer
measures specific force in the body-fixed frame
${}^{i}\!\mathcal{F}$; applying the transport theorem to
express the inertial acceleration $\ddot{\mathbf{p}}_i$ in
body-fixed coordinates gives:

\begin{align}
  \mathbf{y}_i^{\mathrm{acc}}
  = \dot{\mathbf{v}}_i
  + [\boldsymbol{\omega}_i]^\times\mathbf{v}_i
  - \mathbf{R}_{oi}^\top\mathbf{g}
  + \boldsymbol{\eta}_{a,i},
  \qquad
  \boldsymbol{\eta}_{a,i} \sim
  \mathcal{N}(\mathbf{0},\,\mathbf{N}_{a,i}),
  \label{eq:acc_meas}
\end{align}

where $\mathbf{g}$ is the inertial gravitational vector and
$\mathbf{R}_{oi}^\top\mathbf{g}$ rotates it into the body
frame. The accelerometer therefore observes
$\dot{\mathbf{V}}_{i,\mathrm{trans}}$, not
$\mathbf{V}_{i,\mathrm{trans}}$ directly. Integrating
\eqref{eq:acc_meas} over $[t_k, t_{k+1}]$ and adding back the
gravity term yields the velocity increment:

\begin{align}
  \Delta\mathbf{v}_i^{\mathrm{meas}}
  = \int_{t_k}^{t_{k+1}}\!
    \left(\mathbf{y}_i^{\mathrm{acc}}
    + \mathbf{R}_{oi}^\top\mathbf{g}\right)dt
  = \mathbf{v}_i(t_{k+1}) - \mathbf{v}_i(t_k)
  + \boldsymbol{\eta}_{v,i},
  \label{eq:vel_increment}
\end{align}

where $\boldsymbol{\eta}_{v,i} =
\int_{t_k}^{t_{k+1}}\boldsymbol{\eta}_{a,i}\,dt \sim
\mathcal{N}(\mathbf{0},\,\mathbf{N}_{a,i}\Delta t)$, with
the $\Delta t$ scaling arising directly from integration of
the continuous-time white noise spectral density
$\mathbf{N}_{a,i}$ over one sample interval --- exactly
analogous to the process noise discretisation
$\mathbf{Q}_{c,i}^w \to \mathbf{Q}_{c,i}^w\Delta t$.
Stacking gyroscope and reconstructed velocity:

\begin{align}
  \mathbf{y}_i^{\mathrm{IMU}}
  = \mathbf{V}_i(t_{k+1})
  + \boldsymbol{\eta}_i,
  \qquad
  \boldsymbol{\eta}_i \sim
  \mathcal{N}(\mathbf{0},\,\mathbf{N}_i^{\mathrm{IMU}}),
  \label{eq:imu_meas}
\end{align}

with:

\begin{align}
  \mathbf{N}_i^{\mathrm{IMU}} =
  \begin{bmatrix}
    \mathbf{N}_{\omega,i} & \mathbf{0} \\[2pt]
    \mathbf{0}            & \mathbf{N}_{a,i}\Delta t
  \end{bmatrix}.
  \label{eq:Nimu}
\end{align}

\begin{remark}
  The gravity compensation in~\eqref{eq:vel_increment} uses
  the rotation estimate $\hat{\mathbf{R}}_{oi}$ in place of
  the true $\mathbf{R}_{oi}$. The resulting error
  $(\mathbf{R}_{oi}^\top - \hat{\mathbf{R}}_{oi}^\top)\mathbf{g}$
  is second order in the filter error and is neglected in the
  nominal noise model, but can be absorbed into $\mathbf{N}_{a,i}$
  as a conservative inflation when rotation errors are large.
  \label{rem:gravity}
\end{remark}

\subsubsection{Measurement Model --- Joint Encoder}
\label{sec:3.1.3}

The revolute joint encoder on joint $i$ measures the joint
angle with additive quantization noise:

\begin{align}
  y_i^{\mathrm{enc}} = q_i + n_i^{\mathrm{enc}},
  \qquad
  n_i^{\mathrm{enc}} \sim
  \mathcal{N}(0,\,\sigma_{\mathrm{enc},i}^2),
  \label{eq:enc_meas}
\end{align}

where $\sigma_{\mathrm{enc},i}^2$ is the encoder
quantization variance. The encoder serves three purposes
in the filter: (i) it provides the measurement $q_i$
that directly constrains the rotational pose DOF along
$\hat{\mathbf{z}}_i$, entering the innovation via the
joint-angle row of $\mathbf{H}_{\mathrm{obs},i}$; (ii) it
determines $\mathbf{Ad}_{i(i-1)}$ (the Adjoint of the
relative joint transform) used in the twist
prediction~\eqref{eq:V_pred_chain} and pose
prediction~\eqref{eq:g_pred_chain}; and (iii) it provides
the relative joint velocity $\dot{\hat{q}}_i$ and hence
$\hat{\mathbf{V}}_{i,\mathrm{rel}}$ in~\eqref{eq:V_pred_chain}.

\subsection{Stochastic Model of System Dynamics}
\label{sec:3.2}

The filter operates on the Lie algebra error
$\tilde{\boldsymbol{\xi}}_i \in \mathbb{R}^{12}$ of
Section~\ref{sec:4.2}, not on $(g_i,\mathbf{V}_i)$ directly.

The complete stochastic model is the SDE--observation
triple with all noise sources now defined. The plant
equations are~\citep{chirikjian2012,oksendal2003,murray1994,lynch2017}:

\begin{align}
  dg_i &= g_i\mathbf{V}_i\,dt,
  \label{eq:sde_g}\\[4pt]
  d\mathbf{V}_i &= \mathbf{f}_i^V(g_i,\mathbf{V}_i)\,dt
  + \mathbf{G}_i^w\,d\mathbf{W}_i^w,
  \quad
  d\mathbf{W}_i^w \sim \mathcal{N}(\mathbf{0},\mathbf{Q}_{c,i}^w\,dt),
  \label{eq:sde_V}
\end{align}

where $\mathbf{f}_i^V$ and $\mathbf{G}_i^w = \mathbf{M}_i^{-1}$
are identified in~\eqref{eq:Vdot_stoch}. The pose equation
contains no direct Wiener increment --- $g_i$ is driven
stochastically only through $\mathbf{V}_i$; the velocity
equation is a linear It\^{o} SDE in $\mathbb{R}^6$.

The discrete-time measurement equation expresses
$\mathbf{y}_i$ as a function of the true state
$(g_i, \mathbf{V}_i)$:

\begin{align}
  \mathbf{y}_i =
  \begin{bmatrix}
    \mathbf{y}_i^{\mathrm{IMU}} \\[2pt]
    y_i^{\mathrm{enc}}
  \end{bmatrix}
  =
  \begin{bmatrix}
    \mathbf{V}_i \\[2pt]
    q_i(g_i)
  \end{bmatrix}
  +
  \boldsymbol{\nu}_i,
  \qquad
  \boldsymbol{\nu}_i \sim \mathcal{N}(\mathbf{0},\mathbf{N}_i^{\mathrm{obs}}),
  \label{eq:meas_full}
\end{align}

where $q_i(g_i)$ is the joint angle extracted from the
pose $g_i$ via the encoder (measured directly as
$q_i + n_i^{\mathrm{enc}}$), and the stacked measurement
noise covariance is:

\begin{align}
  \mathbf{N}_i^{\mathrm{obs}} =
  \begin{bmatrix}
    \mathbf{N}_{\omega,i}    & \mathbf{0}               & \mathbf{0} \\[2pt]
    \mathbf{0}               & \mathbf{N}_{a,i}\Delta t & \mathbf{0} \\[2pt]
    \mathbf{0}               & \mathbf{0}               & \sigma_{\mathrm{enc},i}^2
  \end{bmatrix} \in \mathbb{R}^{7\times7}.
  \label{eq:Nobs}
\end{align}

The observation function $h_i(g_i,\mathbf{V}_i) =
[\mathbf{V}_i^\top, q_i(g_i)]^\top$ is linear in
$\mathbf{V}_i$ but nonlinear in $g_i$ through $q_i(g_i)$.
In the filter, this function is linearized with respect to
the Lie algebra error $\tilde{\boldsymbol{\xi}}_i$,
giving the observation matrix $\mathbf{H}_{\mathrm{obs},i}
\in \mathbb{R}^{7\times12}$ defined in
Section~\ref{sec:7.1}. The inter-link kinematic coupling
via~\eqref{eq:Vprop} is deterministic --- all stochasticity
is contained in $\mathbf{Q}_{c,i}^w$ (process) and
$\mathbf{N}_i^{\mathrm{obs}}$ (measurement).

\begin{table}[htbp]
\centering
\caption{Noise parameters, physical origins, and
  identification methods.}
\label{tab:noise}
\renewcommand{\arraystretch}{1.35}
\begin{tabular}{>{\small}l >{\small}l >{\small}l >{\small}l}
\toprule
\normalsize Parameter &
\normalsize Size &
\normalsize Physical origin &
\normalsize Identified from \\
\midrule
$\mathbf{Q}_{c,i}^w$ & $6\times6$ &
  Wrench disturbance &
  Innovation consistency tuning \\
  & & (friction, unmodeled dynamics) & \\[4pt]
$\mathbf{N}_{\omega,i}$ & $3\times3$ &
  Gyroscope noise &
  IMU datasheet \\[4pt]
$\mathbf{N}_{a,i}\Delta t$ & $3\times3$ &
  Accelerometer noise (integrated) &
  IMU datasheet $\times\,\Delta t$ \\[4pt]
$\sigma_{\mathrm{enc},i}^2$ & $1\times1$ &
  Encoder quantization &
  Encoder datasheet \\
\bottomrule
\end{tabular}
\end{table}

\section{Geometric Error Dynamics}
\label{sec:4}

The estimate $(\hat{g}_i,\hat{\mathbf{V}}_i)$ is kept on $SE(3)\times
\mathbb{R}^6$, while covariance propagation, gain and innovation operate
on a Lie algebra error vector; corrections are retracted to the manifold
through the exponential map (Section~\ref{sec:7.2}). This section
defines that error and its linearized dynamics.

\subsection{Left-Invariant Error State}
\label{sec:4.2}

The left-invariant pose error between the true configuration
$g_i \in SE(3)$ and the filter estimate $\hat{g}_i \in SE(3)$
is defined as:

\begin{align}
  \tilde{g}_i \triangleq \hat{g}_i^{-1} g_i \in SE(3).
  \label{eq:gerror}
\end{align}

When the estimate is exact, $\tilde{g}_i = \mathbf{I}_{4}$
(the group identity). The Lie algebra coordinates of the pose
error are obtained via the logarithmic map~\citep{murray1994}:

\begin{align}
  \tilde{\boldsymbol{\xi}}_i^g
  = \log(\tilde{g}_i)^\vee \in \mathbb{R}^6,
  \label{eq:xi_pose}
\end{align}

whose closed-form computation in terms of $\tilde{\mathbf{R}}_i$
and $\tilde{\mathbf{p}}_i$ is derived in
Appendix~\ref{app:logSE3}; see
in particular~\eqref{eq:xi_pose_expanded}. The twist error is defined in the Lie algebra directly:

\begin{align}
  \tilde{\mathbf{V}}_i
  = \mathbf{V}_i - \hat{\mathbf{V}}_i \in \mathbb{R}^6.
  \label{eq:Verror}
\end{align}

The full per-link error state is:

\begin{align}
  \tilde{\boldsymbol{\xi}}_i =
  \begin{bmatrix}
    \tilde{\boldsymbol{\xi}}_i^g \\[2pt]
    \tilde{\mathbf{V}}_i
  \end{bmatrix} \in \mathbb{R}^{12}.
  \label{eq:xi_full}
\end{align}

With the geometric error state $\tilde{\boldsymbol{\xi}}_i$
defined, the objective of the filter design is to find the
optimal estimate $(\hat{g}_i, \hat{\mathbf{V}}_i)$ that
minimises the covariance $\mathbf{P}_i =
\mathbb{E}[\tilde{\boldsymbol{\xi}}_i
\tilde{\boldsymbol{\xi}}_i^\top]$ of this error, accounting
for all noise sources acting on the system. Achieving this
via an IEKF requires a \emph{linearized} model of how
$\tilde{\boldsymbol{\xi}}_i$ evolves in time under the
combined effects of the system dynamics and noise. The
required form is the stochastic differential equation:

\begin{align}
  d\tilde{\boldsymbol{\xi}}_i
  = \underbrace{\mathbf{F}_{c,i}\tilde{\boldsymbol{\xi}}_i\,dt}_{\text{drift (linearization)}}
  + \underbrace{\tilde{\mathbf{G}}_i^w\,d\mathbf{W}_i^w
  + \tilde{\mathbf{G}}_i^\eta\,d\boldsymbol{\zeta}_i}_{\text{noise (all channels)}},
  \label{eq:error_sde_preview}
\end{align}

where $\mathbf{F}_{c,i} \in \mathbb{R}^{12\times12}$ is the
linearized error Jacobian capturing how the estimated dynamics
drive $\tilde{\boldsymbol{\xi}}_i$, and
$\tilde{\mathbf{G}}_i^w$, $\tilde{\mathbf{G}}_i^\eta \in
\mathbb{R}^{12\times6}$ are noise input matrices mapping each
noise channel into the error state.

Section~\ref{sec:5} derives all terms in~\eqref{eq:error_sde_preview}
by combining the twist dynamics~\eqref{eq:newton_euler}, the
stochastic plant model~\eqref{eq:sde_V}, and a systematic
analysis of four independent noise sources and their effect
on $\tilde{\boldsymbol{\xi}}_i$~\eqref{eq:xi_full}:
(i)~wrench disturbance $\mathbf{W}_i^w$ acting on the twist
channel through $\mathbf{M}_i^{-1}$;
(ii)~IMU measurement noise $\boldsymbol{\eta}_i$ entering
through the Coriolis linearization $\mathbf{C}_i$;
(iii)~upstream synthesis velocity error from link $i-1$,
which perturbs the predicted twist through
$\mathbf{Ad}_{i(i-1)}$ and enters the same Coriolis channel;
(iv)~encoder noise $n_i^{\mathrm{enc}}$ and upstream pose error
$\tilde{\boldsymbol{\xi}}_{i-1}^{g,+}$, which shift the
initial condition of $\tilde{\boldsymbol{\xi}}_i^g$ through
the unit joint screw $\mathbf{j}_{\mathrm{enc},i}$.
Sources~(i)--(iii) enter continuously through the velocity
block of~\eqref{eq:error_sde_preview}; source~(iv) acts as
a discrete initial-condition shift on the pose block.

\section{Group-Affine Geometric Error Model and Noise
  Characterisation}
\label{sec:5}

This section establishes the geometric foundations of the
prediction error SDE~\eqref{eq:error_sde_preview}: the
group-affine property that makes the error dynamics autonomous
(Section~\ref{sec:5.1}), the linearization that yields the
drift matrix $\mathbf{F}_{c,i}$ (Section~\ref{sec:5.2}),
and the identification of all noise input matrices and their
covariances that complete the SDE (Section~\ref{sec:5.3}).

\subsection{Group-Affine Property and Autonomous Error Dynamics}
\label{sec:5.1}

The pose kinematics $\dot{g}_i = g_i\mathbf{V}_i$
satisfy the \emph{group-affine} property~\citep{barrau2017}:
the left-invariant error $\tilde{g}_i = \hat{g}_i^{-1}g_i$
evolves autonomously, independent of the estimate $\hat{g}_i$.

\begin{proposition}
  The left-invariant error $\tilde{g}_i = \hat{g}_i^{-1}g_i$ of the
  kinematics~\eqref{eq:gdot} evolves independently of $\hat{g}_i$.
  \label{prop:group_affine}
\end{proposition}
\begin{proof}
  With $\tfrac{d}{dt}\hat{g}_i^{-1} = -\hat{g}_i^{-1}\dot{\hat{g}}_i
  \hat{g}_i^{-1}$, $\dot{\hat{g}}_i = \hat{g}_i\hat{\mathbf{V}}_i$ and
  $\dot{g}_i = g_i\mathbf{V}_i$,
  \begin{align}
    \dot{\tilde{g}}_i
    = \tilde{g}_i\mathbf{V}_i - \hat{\mathbf{V}}_i\tilde{g}_i,
    \label{eq:gerror_dot2}
  \end{align}
  which does not contain $\hat{g}_i$.
\end{proof}

\subsection{Linearized Error Dynamics --- Determining $\mathbf{F}_{c,i}$}
\label{sec:5.2}

Writing $\tilde{g}_i = \exp([\tilde{\boldsymbol{\xi}}_i^g]^\wedge)$ and
linearizing~\eqref{eq:gerror_dot2} about $\tilde{g}_i = \mathbf{I}$ with
the first-order derivative of the exponential map {\citep[Prop.~3.3]{hall2015}} gives the standard
left-invariant result~\citep{barrau2017}
\begin{align}
  \dot{\tilde{\boldsymbol{\xi}}}_i^g
  \approx
  -\mathrm{ad}_{\hat{\mathbf{V}}_i}\tilde{\boldsymbol{\xi}}_i^g
  + \tilde{\mathbf{V}}_i.
  \label{eq:xi_dot_pose}
\end{align}

The twist error dynamics follow by evaluating~\eqref{eq:sde_V} at the
true and estimated twist and subtracting. To first order in
$\tilde{\mathbf{V}}_i$ the Coriolis term perturbs by
\begin{align}
  \delta\mathbf{H}_i
  = \mathrm{ad}_{\tilde{\mathbf{V}}_i}^\top\mathbf{M}_i\hat{\mathbf{V}}_i
  + \mathrm{ad}_{\hat{\mathbf{V}}_i}^\top\mathbf{M}_i\tilde{\mathbf{V}}_i
  = \mathbf{C}_i\tilde{\mathbf{V}}_i,
  \qquad
  \mathbf{C}_i =
  \mathrm{ad}_{\hat{\mathbf{V}}_i}^\top\mathbf{M}_i
  + \mathrm{ad}^*_{\mathbf{M}_i\hat{\mathbf{V}}_i},
  \label{eq:Ci_def}
\end{align}
where, for $\mathbf{b} = [\mathbf{p}^\top,\mathbf{q}^\top]^\top$,
$\mathrm{ad}^*_{\mathbf{b}} = \left[\begin{smallmatrix}
[\mathbf{p}]^\times & \mathbf{0}\\ [\mathbf{q}]^\times & [\mathbf{p}]^\times
\end{smallmatrix}\right]$ satisfies
$\mathrm{ad}_{\mathbf{a}}^\top\mathbf{b} = \mathrm{ad}^*_{\mathbf{b}}\mathbf{a}$,
which follows from~\eqref{eq:adv} and
$[\mathbf{u}]^\times\mathbf{w} = -[\mathbf{w}]^\times\mathbf{u}$.

The twist error dynamics then follow as:

\begin{align}
  \dot{\tilde{\mathbf{V}}}_i
  = -\mathbf{M}_i^{-1}\,\delta\mathbf{H}_i
  + \mathbf{M}_i^{-1}\mathbf{W}_i^w
  = -\mathbf{M}_i^{-1}\mathbf{C}_i\tilde{\mathbf{V}}_i
  + \mathbf{M}_i^{-1}\mathbf{W}_i^w.
  \label{eq:Vdot_error}
\end{align}

Stacking~\eqref{eq:xi_dot_pose} and~\eqref{eq:Vdot_error},
the linearized error Jacobian is:

\begin{align}
  \mathbf{F}_{c,i} =
  \begin{bmatrix}
    -\mathrm{ad}_{\hat{\mathbf{V}}_i} & \mathbf{I}_6 \\[2pt]
    \mathbf{0}_{6\times6} & -\mathbf{M}_i^{-1}\mathbf{C}_i
  \end{bmatrix}
  \in \mathbb{R}^{12\times12},
  \label{eq:Fci}
\end{align}

evaluated at $(\hat{g}_i, \hat{\mathbf{V}}_i)$ and
recomputed at each step. $\mathbf{F}_{c,i}$ depends on
$\hat{\mathbf{V}}_i$ but not on $\tilde{\boldsymbol{\xi}}_i$
--- the autonomous error dynamics property of the IEKF.
The complete prediction error SDE, which adds the IMU and
upstream synthesis noise channels beyond the wrench term
already in~\eqref{eq:Vdot_error}, is assembled in
Section~\ref{sec:5.3}.

\subsection{Determining $\tilde{\mathbf{G}}_i^w$,
  $\tilde{\mathbf{G}}_i^\eta$, and $\mathbf{N}_i^{\mathrm{total}}$}
\label{sec:5.3}

This subsection identifies the noise input matrices
$\tilde{\mathbf{G}}_i^w$ and $\tilde{\mathbf{G}}_i^\eta$
and their covariances in~\eqref{eq:error_sde_preview}.
The drift matrix $\mathbf{F}_{c,i}$ and the wrench noise
input $\tilde{\mathbf{G}}_i^w = [\mathbf{0};\mathbf{M}_i^{-1}]$
were established in Sections~\ref{sec:5.1}--\ref{sec:5.2};
two further noise channels complete the SDE.

\textbf{IMU Coriolis contribution.}
The filter evaluates the predicted Coriolis term using
$\hat{\mathbf{V}}_i^- = \mathbf{Ad}_{i(i-1)}\hat{\mathbf{V}}_{i-1}^+
+ \hat{\mathbf{V}}_{i,\mathrm{rel}}$, which is constructed from
the noisy IMU measurement. At leading order, the IMU measurement
noise $\boldsymbol{\eta}_i \sim \mathcal{N}(\mathbf{0},
\mathbf{N}_i^{\mathrm{IMU}})$ contributes an error
$\tilde{\mathbf{V}}_i^{\mathrm{IMU}} \approx \boldsymbol{\eta}_i$
to the predicted twist. This propagates through the linearized
Coriolis term~\eqref{eq:Ci_def} into the velocity error channel:

\begin{align}
  d\tilde{\mathbf{V}}_i^{\mathrm{IMU}}
  = -\mathbf{M}_i^{-1}\mathbf{C}_i\,d\boldsymbol{\eta}_i,
  \qquad
  d\boldsymbol{\eta}_i \sim \mathcal{N}(\mathbf{0},
  \mathbf{N}_i^{\mathrm{IMU}}).
  \label{eq:imu_noise_sde}
\end{align}

\textbf{Upstream synthesis velocity error.}
In the serial chain the filter operates sequentially: the
predicted twist of link $i$ is built from the posterior
estimate of link $i-1$ via $\hat{\mathbf{V}}_i^- =
\mathbf{Ad}_{i(i-1)}\hat{\mathbf{V}}_{i-1}^+ +
\hat{\mathbf{V}}_{i,\mathrm{rel}}$ (see Section~\ref{sec:6.3}).
The residual velocity error of link $i-1$ at $t_k^+$
therefore enters the predicted twist of link $i$ as:

\begin{align}
  \tilde{\mathbf{V}}_i^{\mathrm{up}}
  = \mathbf{Ad}_{i(i-1)}\,\tilde{\mathbf{V}}_{i-1}^+,
  \label{eq:V_up}
\end{align}

with covariance $\mathbf{Ad}_{i(i-1)}\mathbf{P}_{i-1,VV}^+
\mathbf{Ad}_{i(i-1)}^\top$. Crucially,
$\tilde{\mathbf{V}}_i^{\mathrm{up}}$ enters the Coriolis
linearization through \emph{exactly the same channel} as the
IMU noise $\boldsymbol{\eta}_i$: both perturb the velocity
used in $\mathbf{C}_i\tilde{\mathbf{V}}_i$. They are
independent ($\tilde{\mathbf{V}}_{i-1}^+$ is determined by
past measurements of link $i-1$, independent of the current
IMU noise of link $i$), so their covariances add. Define the
\emph{total effective twist uncertainty} entering the Coriolis
channel:

\begin{align}
  \mathbf{N}_i^{\mathrm{total}}
  \triangleq
  \mathbf{Ad}_{i(i-1)}\,\mathbf{P}_{i-1,VV}^+\,
  \mathbf{Ad}_{i(i-1)}^\top
  + \mathbf{N}_i^{\mathrm{IMU}},
  \label{eq:Ntotal}
\end{align}

which unifies the synthesis propagation error and the IMU
measurement noise into a single velocity uncertainty
covariance. For link $1$ (base link), $\mathbf{P}_{0,VV}^+
= \mathbf{0}$ and $\mathbf{N}_1^{\mathrm{total}} =
\mathbf{N}_1^{\mathrm{IMU}}$, recovering the single-body case.

\textbf{All terms of~\eqref{eq:error_sde_preview} determined.}
The noise input matrices in~\eqref{eq:error_sde_preview} are:

\begin{align}
  \tilde{\mathbf{G}}_i^w
  &= \begin{bmatrix} \mathbf{0}_{6\times6} \\
     \mathbf{M}_i^{-1} \end{bmatrix}
  \in \mathbb{R}^{12\times6},
  \label{eq:Gw}\\[4pt]
  \tilde{\mathbf{G}}_i^\eta
  &= \begin{bmatrix} \mathbf{0}_{6\times6} \\
     -\mathbf{M}_i^{-1}\mathbf{C}_i \end{bmatrix}
  \in \mathbb{R}^{12\times6},
  \label{eq:Geta}
\end{align}

with $d\mathbf{W}_i^w \sim \mathcal{N}(\mathbf{0},
\mathbf{Q}_{c,i}^w\,dt)$ and $d\boldsymbol{\zeta}_i \sim
\mathcal{N}(\mathbf{0}, \mathbf{N}_i^{\mathrm{total}})$.
The pose block of both matrices is zero: noise reaches
$\tilde{\boldsymbol{\xi}}_i^g$ only through the coupling $\mathbf{I}_6$
in $\mathbf{F}_{c,i}$.

\section{It\^{o} Analysis, Covariance Propagation, and
  Multibody Synthesis}
\label{sec:6}

Integrating~\eqref{eq:error_sde_preview} over $[t_k,t_{k+1}]$ with
$\mathbf{F}_{c,i}$ frozen at the posterior gives
\begin{align}
  \tilde{\boldsymbol{\xi}}_{i,k+1}^-
  = \boldsymbol{\Phi}_{i,k}\,\tilde{\boldsymbol{\xi}}_{i,k}
  + \mathbf{w}_{i,k},
  \qquad
  \mathbf{w}_{i,k} \sim \mathcal{N}(\mathbf{0},
  \tilde{\mathbf{Q}}_{d,i}^{\mathrm{eff}}).
  \label{eq:xi_discrete}
\end{align}

\subsection{Covariance Propagation via the It\^{o} Formula}
\label{sec:6.1}

The wrench disturbance, the IMU noise and the upstream posterior error
are mutually independent: they arise from separate physical sources,
and the upstream error is determined by past measurements of link $i-1$.
Applying the It\^{o} formula to
$\tilde{\boldsymbol{\xi}}_i\tilde{\boldsymbol{\xi}}_i^\top$ therefore
gives
\begin{align}
  \dot{\mathbf{P}}_i
  = \mathbf{F}_{c,i}\mathbf{P}_i
  + \mathbf{P}_i\mathbf{F}_{c,i}^\top
  + \mathbf{D}_i,
  \qquad
  \mathbf{D}_i =
  \begin{bmatrix}
    \mathbf{0} & \mathbf{0} \\
    \mathbf{0} &
    \mathbf{M}_i^{-1}\!\left(
      \mathbf{Q}_{c,i}^w
      + \mathbf{C}_i\,\mathbf{N}_i^{\mathrm{total}}\,\mathbf{C}_i^\top
    \right)\!\mathbf{M}_i^{-\top}
  \end{bmatrix}.
  \label{eq:Di}
\end{align}

\subsection{Discretization and Complete Predicted Covariance}
\label{sec:6.2}

Over one sample interval,
\begin{align}
  \mathbf{P}_i(t_{k+1})
  = \boldsymbol{\Phi}_i\mathbf{P}_i^+\boldsymbol{\Phi}_i^\top
  + \tilde{\mathbf{Q}}_{d,i}^{\mathrm{eff}},
  \qquad
  \tilde{\mathbf{Q}}_{d,i}^{\mathrm{eff}}
  = \int_0^{\Delta t} e^{\mathbf{F}_{c,i}s}\mathbf{D}_i
    e^{\mathbf{F}_{c,i}^\top s}ds
  = \mathbf{D}_i\Delta t + O(\Delta t^2),
  \label{eq:P_exact}
\end{align}
computed exactly by the Van Loan method~\citep{vanloan1978}, with
$\boldsymbol{\Phi}_i \approx \mathbf{I}_{12} + \mathbf{F}_{c,i}\Delta t
+ \tfrac{1}{2}(\mathbf{F}_{c,i}\Delta t)^2$.

The continuous SDE~\eqref{eq:error_sde_preview} and its
discretization~\eqref{eq:P_exact} account for three of the
four noise sources through $\mathbf{D}_i$: wrench disturbance,
IMU twist noise, and upstream synthesis velocity error. The
fourth, the encoder noise $n_i^{\mathrm{enc}}$ together with the upstream
pose error $\tilde{\boldsymbol{\xi}}_{i-1}^{g,+}$, acts on the
\emph{initial condition} of $\tilde{\boldsymbol{\xi}}_i^g$
at $t_k^+$ and enters additively.

Adding to~\eqref{eq:P_exact} gives the \emph{complete and
final} predicted covariance:

\begin{align}
  \mathbf{P}_i^-
  = \boldsymbol{\Phi}_i\mathbf{P}_i^+\boldsymbol{\Phi}_i^\top
  + \tilde{\mathbf{Q}}_{d,i}^{\mathrm{eff}}
  + \mathbf{P}_i^{-,\mathrm{pose}}.
  \label{eq:P_pred_expanded}
\end{align}

Wrench disturbance, IMU noise and upstream velocity error enter through
$\mathbf{D}_i$; upstream pose error and encoder noise enter through
$\mathbf{P}_i^{-,\mathrm{pose}}$~\eqref{eq:P_upstream_gg}. For the base
link $\mathbf{P}_1^{-,\mathrm{pose}} = \mathbf{0}$.

\begin{remark}
  The Coriolis coupling $\mathbf{C}_i\mathbf{N}_i^{\mathrm{total}}
  \mathbf{C}_i^\top$ vanishes as $\hat{\mathbf{V}}_i \to
  \mathbf{0}$ since $\mathbf{C}_i \to \mathbf{0}$, recovering
  the wrench-only noise model at rest. It grows with both
  twist magnitude and upstream velocity uncertainty
  $\mathbf{P}_{i-1,VV}^+$, and is recomputed each step
  via~\eqref{eq:Ci_def}.
  \label{rem:Qad_vanish}
\end{remark}

\subsection{Modular Prediction Synthesis}
\label{sec:6.3}

The filter operates sequentially: link $1$ is predicted and
updated first, then link $2$ using the posterior of link $1$,
and so on. This section states the kinematic prediction
equations and derives the upstream pose error covariance
$\mathbf{P}_i^{-,\mathrm{pose}}$ that completes~\eqref{eq:P_pred_expanded}.

\textbf{Pose prediction.}
From the product-of-exponentials chain~\citep{murray1994}:

\begin{align}
  \hat{g}_i^- =
  \hat{g}_{i-1}^+ \cdot g_{(i-1)i}^{\mathrm{ref}} \cdot
  \exp\!\left(\hat{q}_i
    [\hat{\mathbf{z}}_i]^\wedge\right),
  \label{eq:g_pred_chain}
\end{align}

where $g_{(i-1)i}^{\mathrm{ref}} \in SE(3)$ is the known reference
configuration at $q_i = 0$ and $\hat{q}_i = y_i^{\mathrm{enc}}$ is the
encoder reading~\eqref{eq:enc_meas}. The true pose satisfies the same
chain with $g_{i-1}$ in place of $\hat{g}_{i-1}^+$ and
$q_i = y_i^{\mathrm{enc}} + n_i^{\mathrm{enc}}$. The prediction error
$\tilde{g}_i^-$ is therefore the conjugation of the upstream error
$\tilde{g}_{i-1}^+$ by the nominal relative transform
$h = g_{(i-1)i}^{\mathrm{ref}}\exp(y_i^{\mathrm{enc}}[\hat{\mathbf{z}}_i]^\wedge)$,
followed by the factor $\exp(n_i^{\mathrm{enc}}[\hat{\mathbf{z}}_i]^\wedge)$.
To first order, by $\log(h^{-1}\tilde{g}h)^\vee = \mathbf{Ad}_{h^{-1}}
\log(\tilde{g})^\vee$ and the Baker--Campbell--Hausdorff (BCH)
formula~\citep{hall2015,barfoot2017},
\begin{align}
  \tilde{\boldsymbol{\xi}}_i^{g,-}
  \approx
  \mathbf{Ad}_{i(i-1)}\,\tilde{\boldsymbol{\xi}}_{i-1}^{g,+}
  + \mathbf{j}_{\mathrm{enc},i}\,n_i^{\mathrm{enc}},
  \qquad
  \mathbf{j}_{\mathrm{enc},i} = [\hat{\mathbf{z}}_i^\top,\,\mathbf{0}^\top]^\top.
  \label{eq:xi_pose_prop}
\end{align}

The two terms are independent, so the pose block of the predicted
covariance receives
\begin{align}
  \mathbf{P}_i^{-,\mathrm{pose}} =
  \begin{bmatrix}
    \mathbf{Ad}_{i(i-1)}\mathbf{P}_{i-1}^{gg,+}\mathbf{Ad}_{i(i-1)}^\top
    + \sigma_{\mathrm{enc},i}^2\,
      \mathbf{j}_{\mathrm{enc},i}\mathbf{j}_{\mathrm{enc},i}^\top
    & \mathbf{0} \\ \mathbf{0} & \mathbf{0}
  \end{bmatrix}.
  \label{eq:P_upstream_gg}
\end{align}

\textbf{Twist prediction.}
From the Adjoint propagation~\eqref{eq:Vprop}:

\begin{align}
  \hat{\mathbf{V}}_i^- =
  \mathbf{Ad}_{i(i-1)}\hat{\mathbf{V}}_{i-1}^+
  + \hat{\mathbf{V}}_{i,\mathrm{rel}},
  \label{eq:V_pred_chain}
\end{align}

where $\hat{\mathbf{V}}_{i,\mathrm{rel}} =
[\hat{\mathbf{z}}_i^\top,\mathbf{0}^\top]^\top\dot{\hat{q}}_i$.
The velocity prediction error
$\tilde{\mathbf{V}}_i^- = \mathbf{Ad}_{i(i-1)}
\tilde{\mathbf{V}}_{i-1}^+$~\eqref{eq:V_up} is absorbed
into $\mathbf{N}_i^{\mathrm{total}}$~\eqref{eq:Ntotal}
and $\mathbf{D}_i$~\eqref{eq:Di}.

\textbf{Information flow.}
The complete predicted covariance $\mathbf{P}_i^-$
from~\eqref{eq:P_pred_expanded} requires from link $i-1$
only: $\hat{g}_{i-1}^+$ (kinematic chain),
$\hat{\mathbf{V}}_{i-1}^+$ (twist prediction and
$\mathbf{N}_i^{\mathrm{total}}$), and $\mathbf{P}_{i-1}^+$
(velocity block for $\mathbf{N}_i^{\mathrm{total}}$, pose
block for $\mathbf{P}_i^{-,\mathrm{pose}}$). No information
from links $i+1,\ldots,n$ is required:

\begin{align}
  (\hat{g}_{i-1}^+,\hat{\mathbf{V}}_{i-1}^+,
  \mathbf{P}_{i-1}^+)
  \;\xrightarrow{\text{\S\ref{sec:6.2}}}\;
  \mathbf{P}_i^-
  \;\xrightarrow{\text{\S\ref{sec:7.1}}}\;
  \mathbf{K}_i
  \;\xrightarrow{\text{\S\ref{sec:7.2}}}\;
  (\hat{g}_i^+,\hat{\mathbf{V}}_i^+,\mathbf{P}_i^+).
  \label{eq:info_flow}
\end{align}
\section{Kalman Update}
\label{sec:7}

With $\mathbf{P}_i^-$ fully determined by~\eqref{eq:P_pred_expanded},
the IEKF update corrects the predicted state using the
innovation from both the IMU twist measurement and the
encoder joint-angle measurement, then retracts the correction
back onto the manifold $SE(3)$.

\subsection{Kalman Gain and Innovation}
\label{sec:7.1}

With $\mathbf{P}_i^-$ given by~\eqref{eq:P_pred_expanded},
the standard IEKF update proceeds in the Lie algebra error
space $\tilde{\boldsymbol{\xi}}_i \in \mathbb{R}^{12}$.

The observation function $h_i(g_i, \mathbf{V}_i) =
[\mathbf{V}_i^\top, q_i(g_i)]^\top \in \mathbb{R}^7$
from~\eqref{eq:meas_full} is linearized with respect to
$\tilde{\boldsymbol{\xi}}_i$ at the predicted estimate
$(\hat{g}_i^-, \hat{\mathbf{V}}_i^-)$. Rows 1--6 correspond
to the IMU twist measurement: since $h_i^{\mathrm{IMU}} =
\mathbf{V}_i$ is linear in $\mathbf{V}_i$ and independent
of $g_i$, these rows contribute $[\mathbf{0}_{6\times6},
\mathbf{I}_6]$ directly. Row 7 corresponds to the encoder
joint-angle measurement: $q_i(g_i)$ is nonlinear in $g_i$
but its Lie algebra linearization with respect to the pose
error $\tilde{\boldsymbol{\xi}}_i^g$ is the unit joint
screw $\mathbf{j}_{q,i}^\top = [\hat{\mathbf{z}}_i^\top,
\mathbf{0}_{1\times3}]$ (the joint axis selects the
rotational component of the pose error aligned with
$\hat{\mathbf{z}}_i$, identical to $\mathbf{j}_{\mathrm{enc},i}^\top$
in~\eqref{eq:xi_pose_prop}). The full observation matrix is:

\begin{align}
  \mathbf{H}_{\mathrm{obs},i} =
  \begin{bmatrix}
    \mathbf{0}_{6\times6}      & \mathbf{I}_6 \\[2pt]
    \mathbf{j}_{q,i}^\top      & \mathbf{0}_{1\times6}
  \end{bmatrix}
  \in \mathbb{R}^{7\times12}.
  \label{eq:Hobs_explicit}
\end{align}

The innovation, innovation covariance, and Kalman gain are:

\begin{align}
  \boldsymbol{\nu}_i &=
  \mathbf{y}_i - h_i(\hat{g}_i^-, \hat{\mathbf{V}}_i^-)
  \in \mathbb{R}^7,
  \label{eq:innovation}\\[4pt]
  \mathbf{S}_i &=
  \mathbf{H}_{\mathrm{obs},i}\,\mathbf{P}_i^-\,
  \mathbf{H}_{\mathrm{obs},i}^\top
  + \mathbf{N}_i^{\mathrm{obs}}
  \in \mathbb{R}^{7\times7},
  \label{eq:S_i}\\[4pt]
  \mathbf{K}_i &=
  \mathbf{P}_i^-\,\mathbf{H}_{\mathrm{obs},i}^\top\,
  \mathbf{S}_i^{-1}
  \in \mathbb{R}^{12\times7}.
  \label{eq:K_i}
\end{align}

Here $\mathbf{N}_i^{\mathrm{obs}} \in \mathbb{R}^{7\times7}$
is the stacked measurement noise covariance~\eqref{eq:Nobs},
combining the IMU noise ($\mathbf{N}_{\omega,i}$ and
$\mathbf{N}_{a,i}\Delta t$) and encoder quantization
($\sigma_{\mathrm{enc},i}^2$). This is distinct from
$\mathbf{N}_i^{\mathrm{IMU}} \in \mathbb{R}^{6\times6}$
(gyroscope and accelerometer noise only, used in the
Coriolis noise term) and $\mathbf{N}_i^{\mathrm{total}}
\in \mathbb{R}^{6\times6}$ (combined twist uncertainty
used in $\mathbf{D}_i$~\eqref{eq:Di}).

\subsection{Modular Per-Link Update}
\label{sec:7.2}

With the Kalman gain $\mathbf{K}_i \in \mathbb{R}^{12\times7}$
computed from~\eqref{eq:K_i}, the correction vector in
the Lie algebra error space is:

\begin{align}
  \delta\tilde{\boldsymbol{\xi}}_i =
  \mathbf{K}_i\,\boldsymbol{\nu}_i =
  \begin{bmatrix}
    \delta\boldsymbol{\xi}_i^g \\[2pt]
    \delta\mathbf{V}_i
  \end{bmatrix}
  \in \mathbb{R}^{12},
  \label{eq:correction}
\end{align}

where $\delta\boldsymbol{\xi}_i^g \in \mathbb{R}^6$ is
the pose correction and $\delta\mathbf{V}_i \in \mathbb{R}^6$
is the twist correction. These are applied via retraction
back to the manifold and vector space respectively:

\begin{align}
  \hat{g}_i^+ &=
  \hat{g}_i^- \cdot
  \exp\!\bigl([\delta\boldsymbol{\xi}_i^g]^\wedge\bigr)
  \in SE(3),
  \label{eq:g_update_full}\\[4pt]
  \hat{\mathbf{V}}_i^+ &=
  \hat{\mathbf{V}}_i^- + \delta\mathbf{V}_i
  \in \mathbb{R}^6.
  \label{eq:V_update_full}
\end{align}

The exponential map keeps $\hat{g}_i^+$ on $SE(3)$ without
normalization. The posterior covariance is updated using the Joseph form
for numerical stability:

\begin{align}
  \mathbf{L}_i &= \mathbf{I}_{12}
  - \mathbf{K}_i\mathbf{H}_{\mathrm{obs},i},
  \label{eq:L_i}\\[4pt]
  \mathbf{P}_i^+ &=
  \mathbf{L}_i\,\mathbf{P}_i^-\,\mathbf{L}_i^\top
  + \mathbf{K}_i\,\mathbf{N}_i^{\mathrm{obs}}\,\mathbf{K}_i^\top
  \in \mathbb{R}^{12\times12}.
  \label{eq:P_update}
\end{align}

\begin{proposition}[Filter Modularity]
  The per-link filter for link $i$ requires at each step only the local
  measurements $\mathbf{y}_i$ and the upstream posterior
  $(\hat{g}_{i-1}^+,\hat{\mathbf{V}}_{i-1}^+,\mathbf{P}_{i-1}^+)$. The
  cost per link is independent of $n$, and the cost of the chain is
  $O(n)$.
  \label{prop:modularity}
\end{proposition}
\begin{proof}
  The prediction~\eqref{eq:P_pred_expanded}--\eqref{eq:V_pred_chain}
  depends on link $i-1$ only through its posterior, by~\eqref{eq:info_flow};
  the update~\eqref{eq:g_update_full}--\eqref{eq:P_update} uses only local
  quantities; and since the noise sources are independent across links
  (Section~\ref{sec:6.1}), no cross-link covariance blocks are driven.
\end{proof}

\subsubsection{Complete Per-Link Filter Algorithm}

The complete filter for link $i$ at each time step $k$
is summarised in Algorithm~\ref{alg:perlinkIEKF}.

\begin{algorithm}[h]
\caption{Per-Link IEKF for Link $i$}
\label{alg:perlinkIEKF}
\begin{algorithmic}[1]
\Require $(\hat{g}_{i-1,k}^+, \hat{\mathbf{V}}_{i-1,k}^+,
  \mathbf{P}_{i-1,k}^+)$,\;
  $(\hat{g}_{i,k}^+, \hat{\mathbf{V}}_{i,k}^+,
  \mathbf{P}_{i,k}^+)$,\;
  measurement $\mathbf{y}_{i,k+1}$
\Ensure $(\hat{g}_{i,k+1}^+, \hat{\mathbf{V}}_{i,k+1}^+,
  \mathbf{P}_{i,k+1}^+)$
\State \textbf{Predict pose:}
  $\hat{g}_{i,k+1}^- \gets \hat{g}_{i-1,k}^+ \cdot
  g_{(i-1)i}^{\mathrm{ref}} \cdot
  \exp(\hat{q}_{i,k+1}[\hat{\mathbf{z}}_i]^\wedge)$
  \hfill\eqref{eq:g_pred_chain}
\State \textbf{Predict twist:}
  $\hat{\mathbf{V}}_{i,k+1}^- \gets
  \mathbf{Ad}_{i(i-1)}\hat{\mathbf{V}}_{i-1,k}^+ +
  \hat{\mathbf{V}}_{i,\mathrm{rel},k+1}$
  \hfill\eqref{eq:V_pred_chain}
\State \textbf{Predict covariance:}
  $\mathbf{P}_{i,k+1}^- \gets
  \boldsymbol{\Phi}_{i,k}\mathbf{P}_{i,k}^+\boldsymbol{\Phi}_{i,k}^\top
  + \tilde{\mathbf{Q}}_{d,i,k}^{\mathrm{eff}}
  + \mathbf{P}_{i,k+1}^{-,\mathrm{pose}}$
  \hfill\eqref{eq:P_pred_expanded}
\State \textbf{Innovation:}
  $\boldsymbol{\nu}_{i,k+1} \gets
  \mathbf{y}_{i,k+1} -
  h_i(\hat{g}_{i,k+1}^-, \hat{\mathbf{V}}_{i,k+1}^-)$
  \hfill\eqref{eq:innovation}
\State \textbf{Innovation covariance:}
  $\mathbf{S}_{i,k+1} \gets
  \mathbf{H}_{\mathrm{obs},i}\mathbf{P}_{i,k+1}^-
  \mathbf{H}_{\mathrm{obs},i}^\top + \mathbf{N}_i^{\mathrm{obs}}$
  \hfill\eqref{eq:S_i}
\State \textbf{Kalman gain:}
  $\mathbf{K}_{i,k+1} \gets
  \mathbf{P}_{i,k+1}^-\mathbf{H}_{\mathrm{obs},i}^\top
  \mathbf{S}_{i,k+1}^{-1}$
  \hfill\eqref{eq:K_i}
\State \textbf{Update pose:}
  $\hat{g}_{i,k+1}^+ \gets \hat{g}_{i,k+1}^- \cdot
  \exp([\mathbf{K}_{i,k+1}\boldsymbol{\nu}_{i,k+1}]_{1:6}^\wedge)$
  \hfill\eqref{eq:g_update_full}
\State \textbf{Update twist:}
  $\hat{\mathbf{V}}_{i,k+1}^+ \gets \hat{\mathbf{V}}_{i,k+1}^-
  + [\mathbf{K}_{i,k+1}\boldsymbol{\nu}_{i,k+1}]_{7:12}$
  \hfill\eqref{eq:V_update_full}
\State \textbf{Update covariance (Joseph):}
  $\mathbf{L} \gets \mathbf{I}_{12} -
  \mathbf{K}_{i,k+1}\mathbf{H}_{\mathrm{obs},i}$;\;
  $\mathbf{P}_{i,k+1}^+ \gets
  \mathbf{L}\mathbf{P}_{i,k+1}^-\mathbf{L}^\top +
  \mathbf{K}_{i,k+1}\mathbf{N}_i^{\mathrm{obs}}\mathbf{K}_{i,k+1}^\top$
  \hfill\eqref{eq:P_update}
\end{algorithmic}
\end{algorithm}

\noindent Steps 1--3 constitute the prediction step;
steps 4--9 the update step. The algorithm requires
$(\hat{g}_{i-1,k}^+, \hat{\mathbf{V}}_{i-1,k}^+,
\mathbf{P}_{i-1,k}^+)$ from link $i-1$ and the local
posterior $(\hat{g}_{i,k}^+, \hat{\mathbf{V}}_{i,k}^+,
\mathbf{P}_{i,k}^+)$ carried from the previous time step.
The algorithm is identical for every link $i = 1,\ldots,n$,
with link $0$ being the known base configuration
$g_0 = \mathbf{I}_4$, $\mathbf{V}_0 = \mathbf{0}$,
$\mathbf{P}_0^+ = \mathbf{0}$.

\begin{remark}
  The filter extends to gyroscope biases by augmenting the state with
  $\mathbf{b}_{\omega,i}$; the augmentation preserves the autonomous error
  dynamics, the modular structure and Theorems~\ref{thm:euubs}
  and~\ref{thm:chained}, and is given in Appendix~\ref{app:bias}.
  \label{rem:bias}
\end{remark}

\section{Stability and Boundedness Analysis}
\label{sec:8}

This section establishes exponential ultimate boundedness
(EUB) in mean square of the per-link estimation error
$\tilde{\boldsymbol{\xi}}_i \in \mathbb{R}^{12}$
from~\eqref{eq:xi_full}, and chains the per-link results
across the serial manipulator. The analysis follows the
framework of Reif et al.~\citep{reif1999} adapted to the
Lie group setting, exploiting the group-affine error
dynamics of Proposition~\ref{prop:group_affine}. Since $SE(3)$ admits no bi-invariant metric, the Adjoint map is not an
isometry, and the chained bound accounts for its operator norm.

\subsection{Discrete-Time Closed-Loop Error Dynamics}
\label{sec:8.1}

The discrete prediction error~\eqref{eq:xi_discrete},
$\tilde{\boldsymbol{\xi}}_{i,k+1}^- =
\boldsymbol{\Phi}_{i,k}\tilde{\boldsymbol{\xi}}_{i,k}
+ \mathbf{w}_{i,k}$,
is corrected by the Kalman update~\eqref{eq:correction}
to give the closed-loop error recursion.
The BCH formula used below is the discrete composition
identity $\log(\exp(-A)\exp(B))^\vee = B^\vee - A^\vee -
\frac{1}{2}[A,B]^\vee + \cdots$, which is distinct from
the continuous-time linearization~\eqref{eq:xi_dot_pose}
of Section~\ref{sec:5.2}.

\textbf{Step 1 --- Kalman update and geometric retraction.}
The update correction is $\delta\tilde{\boldsymbol{\xi}}_{i,k+1}
= \mathbf{K}_{i,k}\boldsymbol{\nu}_{i,k+1}
= [\delta\boldsymbol{\xi}_{i,k+1}^{g\top},
\delta\mathbf{V}_{i,k+1}^\top]^\top \in \mathbb{R}^{12}$,
applied via~\eqref{eq:g_update_full}--\eqref{eq:V_update_full}.
For the twist block the update is ordinary vector subtraction
in the flat space $\mathbb{R}^6$:
\begin{align}
  \tilde{\mathbf{V}}_{i,k+1}^+
  = \tilde{\mathbf{V}}_{i,k+1}^-
  - \delta\mathbf{V}_{i,k+1}.
  \label{eq:Vupdate_exact}
\end{align}
For the pose block, the true posterior error after the
geometric retraction~\eqref{eq:g_update_full} is:
\begin{align}
  \tilde{g}_{i,k+1}^+
  = \bigl(\hat{g}_{i,k+1}^+\bigr)^{-1}g_{i,k+1}
  = \exp\!\bigl(-[\delta\boldsymbol{\xi}_{i,k+1}^g]^\wedge\bigr)
    \cdot\tilde{g}_{i,k+1}^-.
  \label{eq:gerror_update}
\end{align}
Taking the logarithm via the discrete BCH composition
identity:
\begin{align}
  \tilde{\boldsymbol{\xi}}_{i,k+1}^{g,+}
  = \tilde{\boldsymbol{\xi}}_{i,k+1}^{g,-}
    - \delta\boldsymbol{\xi}_{i,k+1}^g
    - \tfrac{1}{2}
      \bigl[\delta\boldsymbol{\xi}_{i,k+1}^g,\,
       \tilde{\boldsymbol{\xi}}_{i,k+1}^{g,-}\bigr]
    + \cdots,
  \label{eq:BCH_expand}
\end{align}
where $[\cdot,\cdot]$ denotes the Lie bracket on
$\mathfrak{se}(3)$ and omitted terms are higher-order
nested commutators. The remainder satisfies
$\|\frac{1}{2}[\delta\boldsymbol{\xi}^g,
\tilde{\boldsymbol{\xi}}^{g,-}]+\cdots\|
\leq C\,
\|\delta\boldsymbol{\xi}_{i,k+1}^g\|\,
\|\tilde{\boldsymbol{\xi}}_{i,k+1}^{g,-}\|$
for a constant $C>0$.
Since $\delta\boldsymbol{\xi}_{i,k+1}^g
= [\mathbf{K}_{i,k}\boldsymbol{\nu}_{i,k+1}]_{1:6}$
and the innovation satisfies
$\boldsymbol{\nu}_{i,k+1} =
\mathbf{H}_{\mathrm{obs},i}\tilde{\boldsymbol{\xi}}_{i,k+1}^-
+ \boldsymbol{\nu}_{i,k+1}^{\mathrm{noise}}$
with $\|\tilde{\boldsymbol{\xi}}_{i,k+1}^-\|
= O(\|\tilde{\boldsymbol{\xi}}_{i,k}\|)$
via~\eqref{eq:xi_discrete}, both factors are
$O(\|\tilde{\boldsymbol{\xi}}_{i,k}\|)$
and the remainder is therefore
$O(\|\tilde{\boldsymbol{\xi}}_{i,k}\|^2)$.

\begin{remark}
  The remainder is second order because the correction is retracted through
  the exponential map~\eqref{eq:g_update_full}, whereas a coordinate EKF
  incurs a first-order linearization error in its prediction. The
  group-affine property (Proposition~\ref{prop:group_affine}) makes the pose
  part of the propagation independent of the estimate, so the transition
  matrix $\boldsymbol{\Phi}_{i,k}$ is not evaluated at an erroneous
  pose~\citep{barrau2017}.
  \label{rem:BCH_order}
\end{remark}

\textbf{Step 2 --- Assembly.}
From~\eqref{eq:xi_discrete},
$\tilde{\boldsymbol{\xi}}_{i,k+1}^- =
\boldsymbol{\Phi}_{i,k}\tilde{\boldsymbol{\xi}}_{i,k}
+ \mathbf{w}_{i,k}$.
Substituting into~\eqref{eq:BCH_expand}
and~\eqref{eq:Vupdate_exact} with
$\delta\tilde{\boldsymbol{\xi}}_{i,k+1} =
\mathbf{K}_{i,k}\mathbf{H}_{\mathrm{obs},i}
\tilde{\boldsymbol{\xi}}_{i,k+1}^-$:
\begin{align*}
  \tilde{\boldsymbol{\xi}}_{i,k+1}
  &= \tilde{\boldsymbol{\xi}}_{i,k+1}^-
   - \mathbf{K}_{i,k}\mathbf{H}_{\mathrm{obs},i}
     \tilde{\boldsymbol{\xi}}_{i,k+1}^-
   + O(\|\tilde{\boldsymbol{\xi}}_{i,k}\|^2) \\
  &= (\mathbf{I}_{12} - \mathbf{K}_{i,k}\mathbf{H}_{\mathrm{obs},i})
     \tilde{\boldsymbol{\xi}}_{i,k+1}^-
   + O(\|\tilde{\boldsymbol{\xi}}_{i,k}\|^2) \\
  &= (\mathbf{I}_{12} - \mathbf{K}_{i,k}\mathbf{H}_{\mathrm{obs},i})
     (\boldsymbol{\Phi}_{i,k}\tilde{\boldsymbol{\xi}}_{i,k}
     + \mathbf{w}_{i,k})
   + O(\|\tilde{\boldsymbol{\xi}}_{i,k}\|^2),
\end{align*}
giving the closed-loop error recursion:
\begin{align}
  \tilde{\boldsymbol{\xi}}_{i,k+1}
  = \mathbf{A}_{i,k}\,\boldsymbol{\Phi}_{i,k}\,
    \tilde{\boldsymbol{\xi}}_{i,k}
  + \mathbf{A}_{i,k}\,\mathbf{w}_{i,k}
  + O(\|\tilde{\boldsymbol{\xi}}_{i,k}\|^2),
  \label{eq:error_closed_loop}
\end{align}
where $\mathbf{A}_{i,k} = \mathbf{I}_{12} -
\mathbf{K}_{i,k}\mathbf{H}_{\mathrm{obs},i}$
and $\mathbf{w}_{i,k} \sim
\mathcal{N}(\mathbf{0},\tilde{\mathbf{Q}}_{d,i}^{\mathrm{eff}})$
collects all noise contributions through
$\mathbf{D}_i$~\eqref{eq:Di}.

The following assumptions are standard for IEKF
stability~\citep{reif1999,barrau2017}:

\begin{assumption}[Bounded noise and Jacobians]
  The effective process noise satisfies
  $\mathbf{0} \prec \underline{q}\mathbf{I} \preceq
  \tilde{\mathbf{Q}}_{d,i}^{\mathrm{eff}} \preceq
  \overline{q}\mathbf{I}$
  and the measurement noise satisfies
  $\mathbf{0} \prec \underline{n}\mathbf{I} \preceq
  \mathbf{N}_i^{\mathrm{obs}} \preceq \overline{n}\mathbf{I}$
  for positive constants $\underline{q}, \overline{q},
  \underline{n}, \overline{n}$. The error Jacobian
  $\mathbf{F}_{c,i}$ is uniformly bounded:
  $\|\mathbf{F}_{c,i}\| \leq \bar{f} < \infty$.
  \label{ass:bounded}
\end{assumption}

\begin{assumption}[Uniform complete observability]
  There exist $\delta > 0$ and $N \geq 1$ such that the
  observability Gramian satisfies
  \begin{align}
    \mathbf{W}_i(k,N)
    \triangleq
    \sum_{j=k}^{k+N}
    \boldsymbol{\Phi}_{i,j:k}^\top\,
    \mathbf{H}_{\mathrm{obs},i}^\top\,
    (\mathbf{N}_i^{\mathrm{obs}})^{-1}\,
    \mathbf{H}_{\mathrm{obs},i}\,
    \boldsymbol{\Phi}_{i,j:k}
    \succeq \delta\mathbf{I}_{12}
    \label{eq:obs_gramian}
  \end{align}
  for all $k \geq 0$, where $\boldsymbol{\Phi}_{i,j:k}
  = \prod_{l=k}^{j-1}\boldsymbol{\Phi}_{i,l}$ is the
  transition matrix from step $k$ to step $j$.
  \label{ass:observability}
\end{assumption}

Under Assumptions~\ref{ass:bounded}
and~\ref{ass:observability}, the discrete Riccati
equation admits a unique bounded positive-definite
steady-state solution $\mathbf{P}_i^\infty \succ 0$,
and the gain $\mathbf{K}_{i,k} \to
\mathbf{K}_i^\infty$~\citep{reif1999}. There exist
uniform bounds $0 < \underline{p}\mathbf{I} \preceq
\mathbf{P}_{i,k}^- \preceq \overline{p}\mathbf{I}$
for all $k$.

\begin{remark}
  Assumptions~\ref{ass:bounded} and~\ref{ass:observability} admit
  time-varying measurement covariances and intermittent measurements.
  If the joint-angle row is deweighted with a bounded covariance, or
  rejected over intervals shorter than the window $N$, as by the
  validation gate of Example~2, the Gramian~\eqref{eq:obs_gramian}
  accumulates the accepted measurements and remains uniformly positive
  definite with a smaller $\delta$. Theorems~\ref{thm:euubs}
  and~\ref{thm:chained} then hold, with a residual ball
  $\beta_i/\alpha_i$ that widens as $\delta$ decreases.
  \label{rem:intermittent}
\end{remark}

\subsection{Per-Link Exponential Ultimate Boundedness}
\label{sec:8.2}

Define the per-link Lyapunov candidate as the
normalized estimation error squared (NEES):

\begin{align}
  V_i(\tilde{\boldsymbol{\xi}}_i,k)
  \triangleq
  \tilde{\boldsymbol{\xi}}_{i,k}^\top\,
  (\mathbf{P}_{i,k}^-)^{-1}\,
  \tilde{\boldsymbol{\xi}}_{i,k} \geq 0.
  \label{eq:lyapunov}
\end{align}

This is a quadratic form on $\mathbb{R}^{12}$, positive
definite by $\mathbf{P}_{i,k}^- \succ 0$. For a
consistent filter, $\mathbb{E}[V_i] = 12$ at steady
state~\citep{bar2004}.

\begin{theorem}[Per-Link EUB in Mean Square]
  Under Assumptions~\ref{ass:bounded}
  and~\ref{ass:observability}, and provided
  $\|\tilde{\boldsymbol{\xi}}_{i,0}\|$ is small enough
  that the $O(\|\tilde{\boldsymbol{\xi}}_{i,k}\|^2)$
  remainder in~\eqref{eq:error_closed_loop} is dominated
  by the linear term, the per-link estimation error
  satisfies:
  \begin{align}
    \mathbb{E}\!\left[
      \|\tilde{\boldsymbol{\xi}}_{i,k}\|^2
    \right]
    \leq c_i\,e^{-\alpha_i k}\,
      \mathbb{E}\!\left[
        \|\tilde{\boldsymbol{\xi}}_{i,0}\|^2
      \right]
    + \frac{\beta_i}{\alpha_i},
    \label{eq:exp_bound}
  \end{align}
  where $c_i = \overline{p}/\underline{p} > 0$,
  $\alpha_i \in (0,1)$, and $\beta_i > 0$ are defined
  in the proof.
  \label{thm:euubs}
\end{theorem}

\begin{proof}
  The proof follows~\citet{reif1999} and is given in
  Appendix~\ref{app:proof_eub}.
\end{proof}

\begin{remark}
  The residual ball $\beta_i/\alpha_i$ is the irreducible
  steady-state mean-square error at which noise injected
  through $\tilde{\mathbf{Q}}_{d,i}^{\mathrm{eff}}$ exactly
  balances the information gained from measurements
  $\mathbf{N}_i^{\mathrm{obs}}$. It cannot be driven to zero
  in the presence of persistent noise.
  \label{rem:residual}
\end{remark}

\subsection{Operator Norm of the Adjoint Map}
\label{sec:8.3}

Before chaining the per-link bounds, the amplification
factor of the Adjoint map must be established. Unlike
$SO(3)$ where $\mathbf{Ad}_{SO(3)} = \mathbf{R}$ is
orthogonal, the Adjoint map on $SE(3)$ is \emph{not}
norm-preserving under the standard Euclidean inner product
on $\mathbb{R}^6$, because $SE(3)$ admits no bi-invariant
metric.

\begin{lemma}[Adjoint Operator Norm]
  The Adjoint map $\mathbf{Ad}_{i(i-1)}$ defined in~\eqref{eq:Ad} satisfies
  \begin{align}
    \|\mathbf{Ad}_{i(i-1)}\mathbf{V}\| \leq \bar{\gamma}_i\,\|\mathbf{V}\|
    \quad \forall\,\mathbf{V} \in \mathbb{R}^6,
    \qquad
    \bar{\gamma}_i \triangleq
    \tfrac{1}{2}\Big(\|{}^i\mathbf{r}_{i-1}\|
    + \sqrt{\|{}^i\mathbf{r}_{i-1}\|^2 + 4}\Big) \geq 1,
    \label{eq:gamma_def}
  \end{align}
  where ${}^i\mathbf{r}_{i-1}$ is the position of frame ${}^{i-1}\mathcal{F}$
  expressed in ${}^i\mathcal{F}$. The bound is attained.
  \label{lem:Ad_norm}
\end{lemma}

\begin{proof}
  By~\eqref{eq:Ad}, $\mathbf{Ad}_{i(i-1)} = \mathbf{T}\,
  \mathrm{blkdiag}(\mathbf{R}_{i(i-1)},\mathbf{R}_{i(i-1)})$ with
  $\mathbf{T} = \left[\begin{smallmatrix}\mathbf{I}_3 & \mathbf{0}\\
  \mathbf{S} & \mathbf{I}_3\end{smallmatrix}\right]$ and
  $\mathbf{S} = [{}^i\mathbf{r}_{i-1}]^\times$. The block-diagonal factor is
  orthogonal, so $\|\mathbf{Ad}_{i(i-1)}\| = \|\mathbf{T}\|$. With the
  singular value decomposition $\mathbf{S} = \mathbf{U}\boldsymbol{\Sigma}
  \mathbf{W}^\top$, $\boldsymbol{\Sigma} = \mathrm{diag}(\|{}^i\mathbf{r}_{i-1}\|,
  \|{}^i\mathbf{r}_{i-1}\|,0)$, the orthogonal transformation
  $\mathrm{blkdiag}(\mathbf{W},\mathbf{U})^\top\,\mathbf{T}\,
  \mathrm{blkdiag}(\mathbf{W},\mathbf{U})$ decouples $\mathbf{T}$ into
  $2\times2$ blocks $\left[\begin{smallmatrix}1&0\\ \sigma&1\end{smallmatrix}\right]$,
  $\sigma \in \{\|{}^i\mathbf{r}_{i-1}\|,\,0\}$, whose largest singular value
  is $\tfrac{1}{2}(\sigma+\sqrt{\sigma^2+4})$. Its maximum over $\sigma$ is
  $\bar{\gamma}_i$.
\end{proof}

\begin{remark}
The bound is conservative and equals $1$ only for co-located frames;
  for any manipulator it is a finite constant fixed by the link geometry.
  \label{rem:gamma}
\end{remark}

\subsection{Chained Per-Link Bound}
\label{sec:8.4}

\begin{theorem}[Chained EUB in Mean Square]
  Under Assumptions~\ref{ass:bounded}
  and~\ref{ass:observability} for each link
  $i = 1,\ldots,n$, the estimation error of link $i$
  satisfies:
  \begin{align}
    \mathbb{E}\!\left[
      \|\tilde{\boldsymbol{\xi}}_{i,k}\|^2
    \right]
    \leq
    c_i\,e^{-\alpha k}\,
    \mathbb{E}\!\left[
      \|\tilde{\boldsymbol{\xi}}_{i,0}\|^2
    \right]
    + \sum_{j=1}^{i}
      \left(\prod_{l=j+1}^{i}\bar{\gamma}_l^2\right)
      \frac{\beta_j}{\alpha},
    \label{eq:chained_bound}
  \end{align}
  where $\alpha = \min_{j=1,\ldots,i}\alpha_j$,
  $c_i > 0$ accumulates per-link constants and
  geometry factors, and $\bar{\gamma}_l$ is defined
  in~\eqref{eq:gamma_def} with the convention
  $\prod_{l=i+1}^{i}(\cdot) = 1$.
  \label{thm:chained}
\end{theorem}

\begin{proof}
  \textbf{Base case ($i=1$).} The base link has no
  upstream contribution ($\mathbf{N}_1^{\mathrm{total}} =
  \mathbf{N}_1^{\mathrm{IMU}}$,
  $\mathbf{P}_1^{-,\mathrm{pose}} = \mathbf{0}$), so
  Theorem~\ref{thm:euubs} applies directly and
  gives~\eqref{eq:chained_bound} with $j=1$ and the
  empty product equal to $1$.

  \textbf{Inductive step.}
  Assume~\eqref{eq:chained_bound} holds for link $i-1$.
  For link $i$, the upstream velocity error~\eqref{eq:V_up}
  contributes to $\mathbf{N}_i^{\mathrm{total}}$~\eqref{eq:Ntotal}:

  \begin{align}
    \mathbf{N}_i^{\mathrm{total}}
    = \mathbf{Ad}_{i(i-1)}\,
      \mathbf{P}_{i-1,VV}^+\,
      \mathbf{Ad}_{i(i-1)}^\top
    + \mathbf{N}_i^{\mathrm{IMU}}.
    \label{eq:Ntotal_chain}
  \end{align}

  By Lemma~\ref{lem:Ad_norm}, the upstream velocity
  covariance entering $\mathbf{N}_i^{\mathrm{total}}$
  satisfies:

  \begin{align}
    \mathbb{E}\!\left[
      \|\tilde{\mathbf{V}}_i^-\|^2
    \right]
    \leq \bar{\gamma}_i^2\,
    \mathbb{E}\!\left[
      \|\tilde{\mathbf{V}}_{i-1}^+\|^2
    \right]
    \leq \bar{\gamma}_i^2\,
    \mathbb{E}\!\left[
      \|\tilde{\boldsymbol{\xi}}_{i-1,k}\|^2
    \right].
    \label{eq:upstream_bound}
  \end{align}

  Applying the inductive hypothesis to the right-hand side
  of~\eqref{eq:upstream_bound} and substituting into the
  effective noise $\tilde{\mathbf{Q}}_{d,i}^{\mathrm{eff}}$
  through $\mathbf{D}_i$~\eqref{eq:Di} and
  $\mathbf{N}_i^{\mathrm{total}}$~\eqref{eq:Ntotal_chain},
  the residual term for link $i$ from the upstream
  contribution of link $j < i$ carries a factor
  $\prod_{l=j+1}^{i}\bar{\gamma}_l^2$, since the
  Adjoint amplification accumulates multiplicatively
  through the chain. Bounding all per-link decay rates
  by $\alpha = \min_j\alpha_j$ and applying
  Theorem~\ref{thm:euubs} to link $i$ with the augmented
  initial condition gives~\eqref{eq:chained_bound}.
\end{proof}

\begin{remark}
  The residual ball grows with the product of the Adjoint factors along
  the chain, that is, with the link lengths, and is bounded for any finite
  chain.
  \label{rem:chained}
\end{remark}

\section{Results and discussion}
\label{sec:9}
The theoretical contributions of Sections~\ref{sec:3}--\ref{sec:8}
are validated in two examples of increasing realism.
Example~1 is a numerical simulation of a 3-DOF two-link rigid serial
manipulator executing large-amplitude three-dimensional motion. It
isolates the geometric estimation advantage of the proposed IEKF by
controlled comparison against a coordinate EKF and the raw
measurement baseline under exact ground truth, with quantitative RMSE
and NEES analysis under controlled noise statistics.
Example~2 is a software-in-the-loop study of an industrial hydraulic
rammer-boom manipulator modeled in Simulink/Simscape Multibody. It
evaluates the pure-IMU configuration, in which the encoder row is
replaced by a gravity-referenced virtual encoder and gyroscope biases
are estimated, on a heavy machine with long links and a hydraulic
start-up transient; the encoders of the multibody model serve as
ground truth only. In all examples, filter parameters are set directly from the physical
noise characterisation of Section~\ref{sec:3}: gyroscope and
accelerometer covariances from sensor datasheets, encoder variance
from the encoder resolution where an encoder is used, and the wrench
disturbance spectral density tuned by innovation
consistency~\citep{bar2004}.

\subsection*{Example 1: Three-Dimensional Two-Link Rigid Chain}

The simulated plant is a two-link rigid serial manipulator
in three-dimensional motion. Link~1 is attached to the
fixed base via a two-DOF revolute joint permitting
rotations about the body-fixed $y$- and $z$-axes, giving
the configuration $g_1 = \exp(\theta_{1y}\hat{\mathbf{y}})
\cdot\exp(\theta_{1z}\hat{\mathbf{z}}) \in SE(3)$.
Link~2 is connected to Link~1 via a single-DOF revolute
joint about the body-fixed $z$-axis, with $g_2 =
g_1\cdot g_{12}^{\mathrm{ref}}\cdot\exp(\theta_2\hat{\mathbf{z}})$.
The system is a 3-DOF chain on $SE(3)$. The motion is
genuinely three-dimensional: the two-DOF base joint
parameterised as $g_1 =
\exp(\theta_{1y}\hat{\mathbf{y}})\exp(\theta_{1z}\hat{\mathbf{z}})$
produces nonzero body-frame angular velocity on all three
axes of Link~1, and all six twist components of Link~2 are
nonzero via the Adjoint propagation~(\ref*{eq:Vprop}).
This is a genuine 3D test, not a planar reduction.

Mechanical parameters are: link masses $m_1 = 2.0$~[kg]
and $m_2 = 1.5$~[kg]; lengths $l_1 = 0.50$~[m] and
$l_2 = 0.40$~[m]; rotational inertias
$\mathbf{I}_{b_i} = \mathrm{diag}(10^{-3},\,
m_il_i^2/12,\, m_il_i^2/12)$~[kg$\cdot$m$^2$].
Noise parameters follow the physical model of
Section~\ref{sec:3}: gyroscope standard deviation
$\sigma_\omega = 0.05$~[rad/s], accelerometer
standard deviation $\sigma_a = 0.20$~[m/s$^2$],
encoder standard deviation $\sigma_{\mathrm{enc}} =
0.5^\circ = 8.7\times10^{-3}$~[rad], and wrench
disturbance spectral densities
$\mathbf{Q}_{c,1}^w =
\mathrm{diag}(0.08^2\mathbf{I}_3,\,0.03^2\mathbf{I}_3)$
and
$\mathbf{Q}_{c,2}^w =
\mathrm{diag}(0.10^2\mathbf{I}_3,\,0.04^2\mathbf{I}_3)$
[N$^2$m$^2$/Hz]. The plant is integrated at
$\Delta t = 5\times10^{-3}$~[s] for $t_f = 15$~[s] by
fourth-order Runge--Kutta applied to the Newton--Euler
equations~(\ref*{eq:newton_euler}), with wrench noise
injected at each step.

The end-effector is commanded to track a Lissajous
trajectory in the $yz$-plane:

\begin{align}
  p_{e,d,y}(t) &= A_y\sin(\omega_d t)\,\rho(t),
  \quad A_y = 0.25~[\mathrm{m}],\;\omega_d = 0.8~[\mathrm{rad/s}],
  \notag\\
  p_{e,d,z}(t) &= A_z\cos(\omega_d t)\,\rho(t),
  \quad A_z = 0.15~[\mathrm{m}],
  \label{eq:sim_traj}
\end{align}

where $\rho(t) = 1 - e^{-t/\tau}$ with $\tau = 1.5$~[s].
Desired absolute joint angles follow from the linear
inverse kinematics

\begin{align}
  \theta_{1y,d} &= -p_{e,d,z}/L, \quad
  \theta_{1z,d} = p_{e,d,y}/L, \notag\\
  \theta_{2,d}^{\mathrm{abs}}
  &= p_{e,d,y}/L
  + \frac{A_{y,2}\sin(\omega_{d,2}t)\,\rho(t)}{l_2},
  \label{eq:sim_ik}
\end{align}

where $L = l_1 + l_2 = 0.90$~[m], $A_{y,2} = 0.10$~[m],
and $\omega_{d,2} = 1.5$~[rad/s]. The independent
oscillation assigned to Link~2 at frequency $\omega_{d,2}$
ensures both revolute joints are actively driven throughout,
exercising the full kinematic chain and the upstream
Adjoint propagation~(\ref*{eq:V_pred_chain}) in the filter.
A proportional-derivative controller with gains
$\mathbf{K}_p = \mathrm{diag}(1.3,\,1.6,\,0.5)$~[Nm/rad]
and $\mathbf{K}_d = \mathrm{diag}(0.5,\,0.6,\,0.2)$~[Nm$\cdot$s/rad]
computes errors in absolute angle coordinates
consistent with~(\ref*{eq:sim_ik}).

Three estimation methods are compared. The \emph{proposed
IEKF} implements Algorithm~\ref{alg:perlinkIEKF} with the
full noise model of Section~\ref{sec:3}, including the
accelerometer $\Delta t$ scaling~(\ref*{eq:Nimu}) and the
Coriolis coupling
$\mathbf{C}_i\mathbf{N}_i^{\mathrm{total}}\mathbf{C}_i^\top$
in $\mathbf{D}_i$~(\ref*{eq:Di}); the per-link error state
has dimension 12 and the total chain filter state dimension
24. The \emph{coordinate EKF (CEKF)} operates on the flat
joint-space state $[\theta_{1y},\,\theta_{1z},\,
\theta_{2}^{\mathrm{abs}},\,\dot{\theta}_{1y},\,
\dot{\theta}_{1z},\,\dot{\theta}_{2}^{\mathrm{abs}}]^\top
\in\mathbb{R}^6$ with linearised double-integrator dynamics
and encoder-only measurements; it has no IMU and no $SE(3)$
manifold structure. The \emph{no-filter baseline} uses raw
encoder readings directly as the state estimate, providing
the sensor noise floor against which both filters are judged.

Two amplitude scenarios are evaluated: S1 (standard,
$A_y = 0.25$~[m]) and S2 (large, $A_y = 0.40$~[m]),
to assess performance under growing motion amplitude.
Time-series figures show scenario~S1; RMSE bar charts
include both scenarios.

\subsubsection*{Body-frame twist estimation}

Figure~\ref{fig:sim_twists} shows the true and
IEKF-estimated body-frame twists for both links. All
nine plotted components are nonzero, confirming genuine
three-dimensional motion. The Link~1 angular velocity
components $\omega_{x,1}$, $\omega_{y,1}$, $\omega_{z,1}$
arise simultaneously from the two-DOF base joint geometry.
All six components of $\mathbf{V}_2$ are nonzero via the
Adjoint propagation~(\ref*{eq:Vprop}); the translational
velocity $\mathbf{v}_2$ in particular is driven entirely
by kinematic coupling from Link~1, since the Link~2 joint
is revolute. The IEKF tracks the true twist closely on
all channels, with the $2\sigma$ covariance band
from~(\ref*{eq:P_pred_expanded}) consistently enveloping
the true value. The band is narrower for Link~1 (directly
observed by two encoders and an IMU) and wider for Link~2,
whose covariance incorporates upstream estimation
uncertainty through $\mathbf{N}_2^{\mathrm{total}}$~(\ref*{eq:Ntotal}).
Since $\mathbf{v}_1 \equiv \mathbf{0}$ identically for
the base-mounted link, those channels are omitted.

\begin{figure*}[!t]
\centering
\subfloat[]{%
  \includegraphics[width=0.32\textwidth]{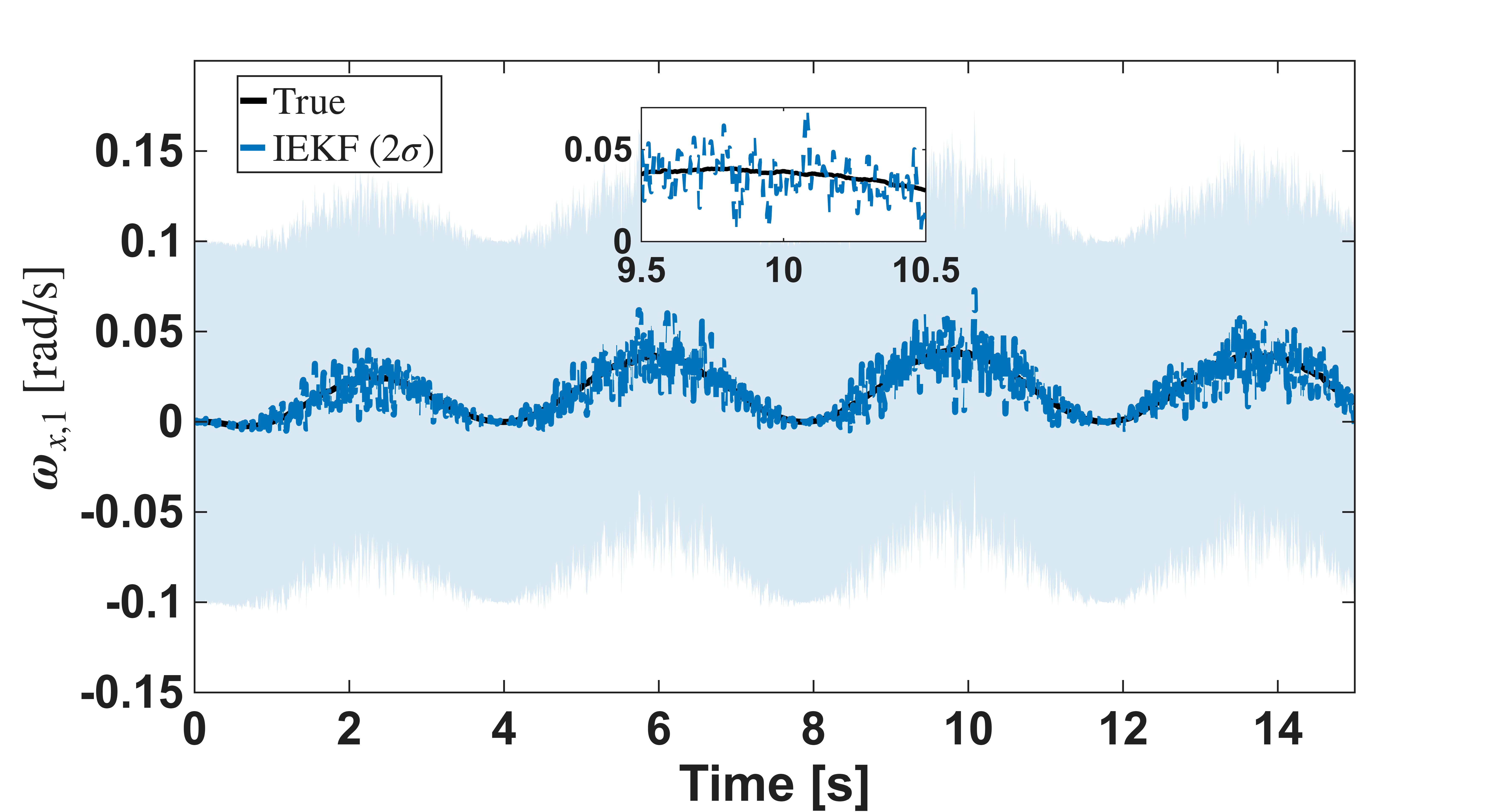}
  \label{fig:V1omx}}
\hfill
\subfloat[]{%
  \includegraphics[width=0.32\textwidth]{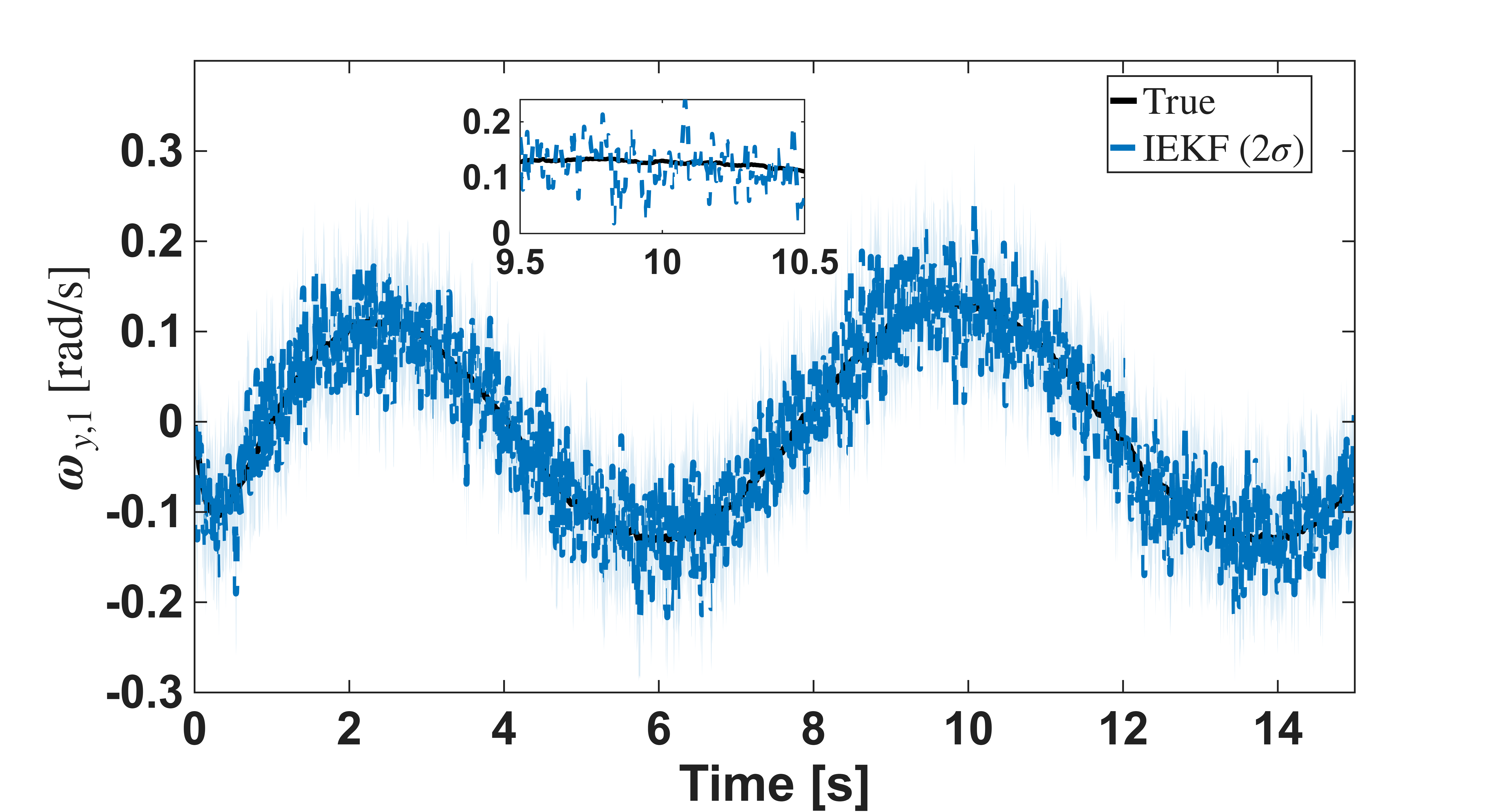}
  \label{fig:V1omy}}
\hfill
\subfloat[]{%
  \includegraphics[width=0.32\textwidth]{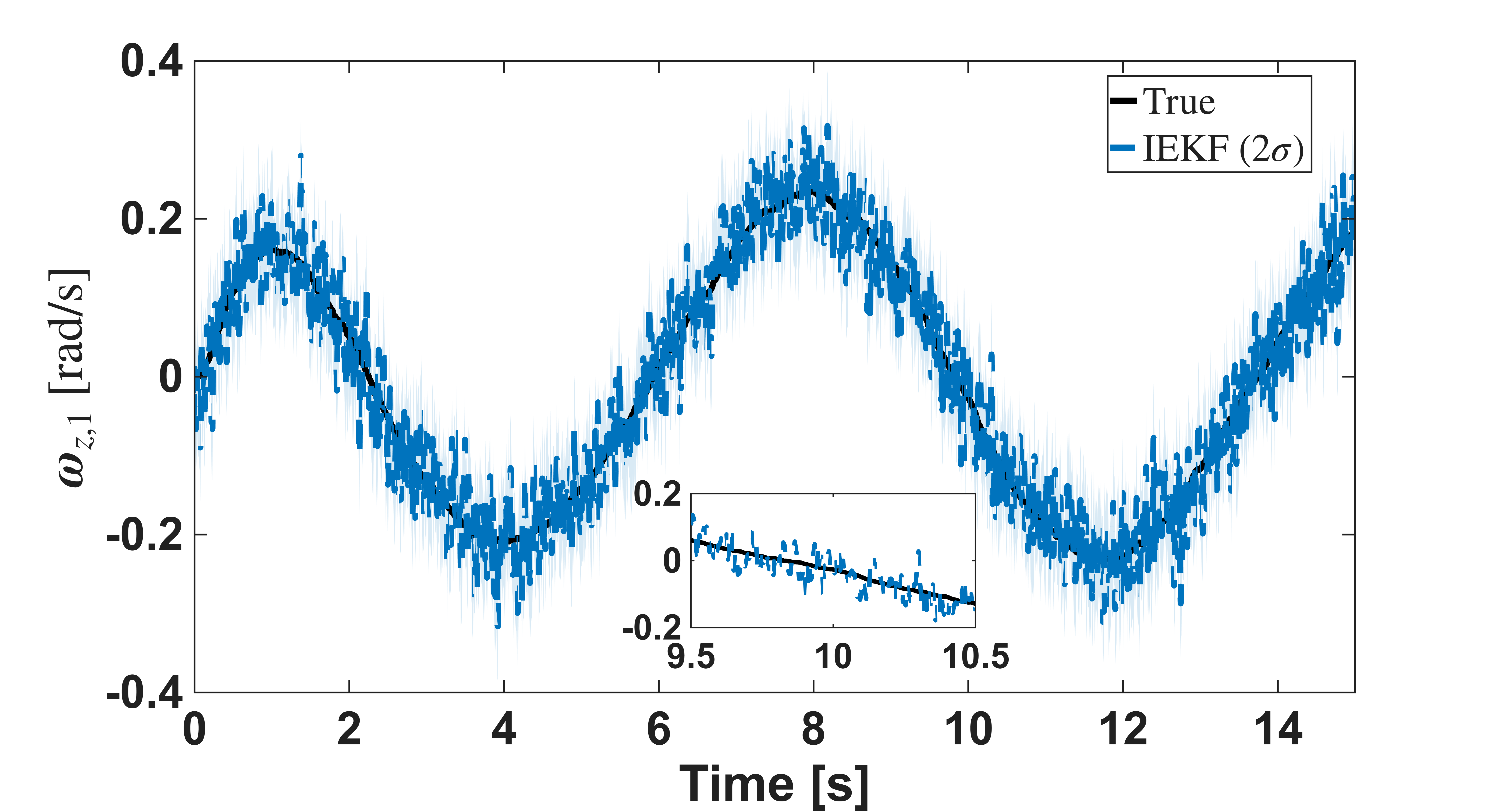}
  \label{fig:V1omz}}\\[0.5em]
\subfloat[]{%
  \includegraphics[width=0.32\textwidth]{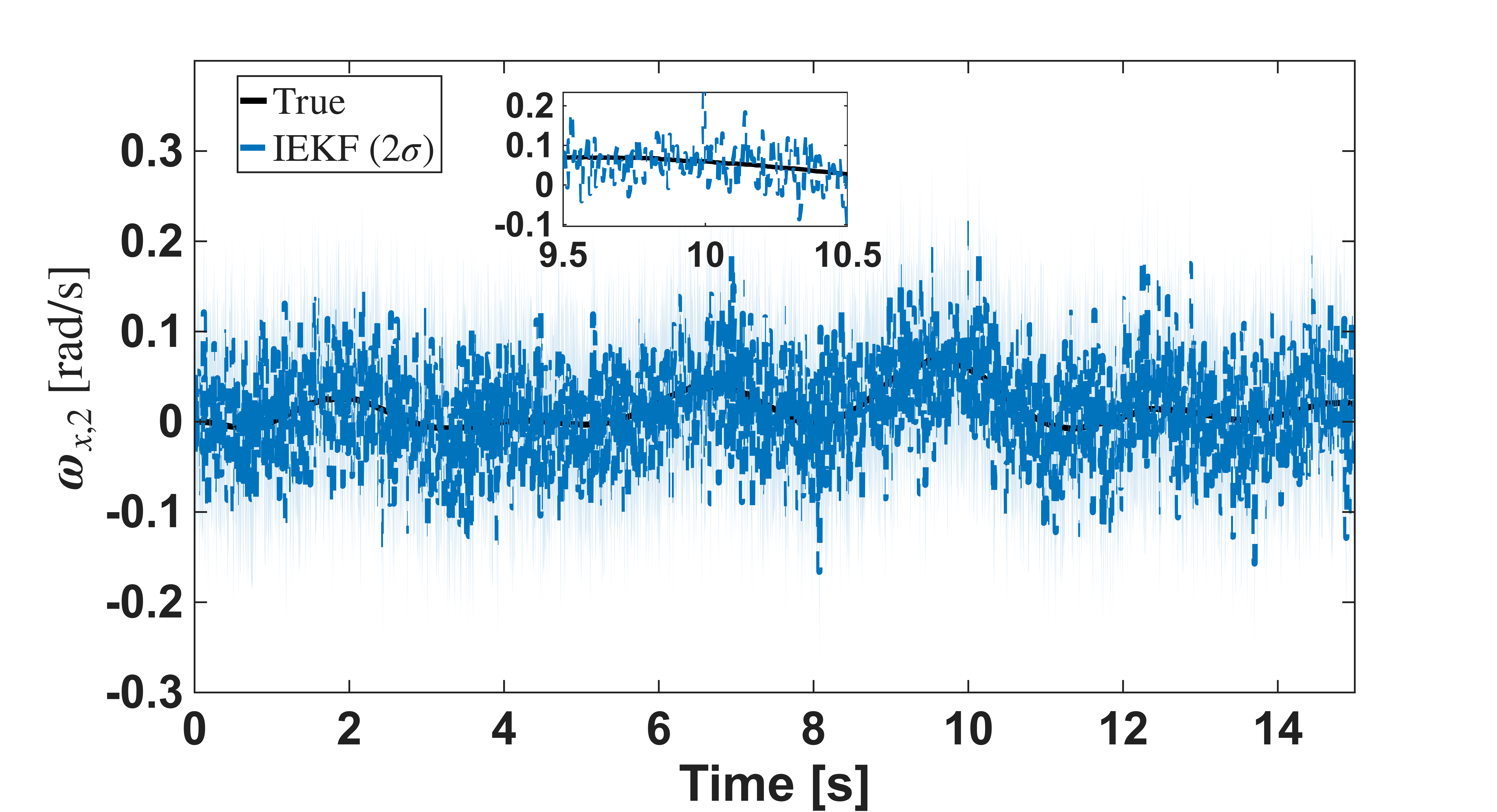}
  \label{fig:V2omx}}
\hfill
\subfloat[]{%
  \includegraphics[width=0.32\textwidth]{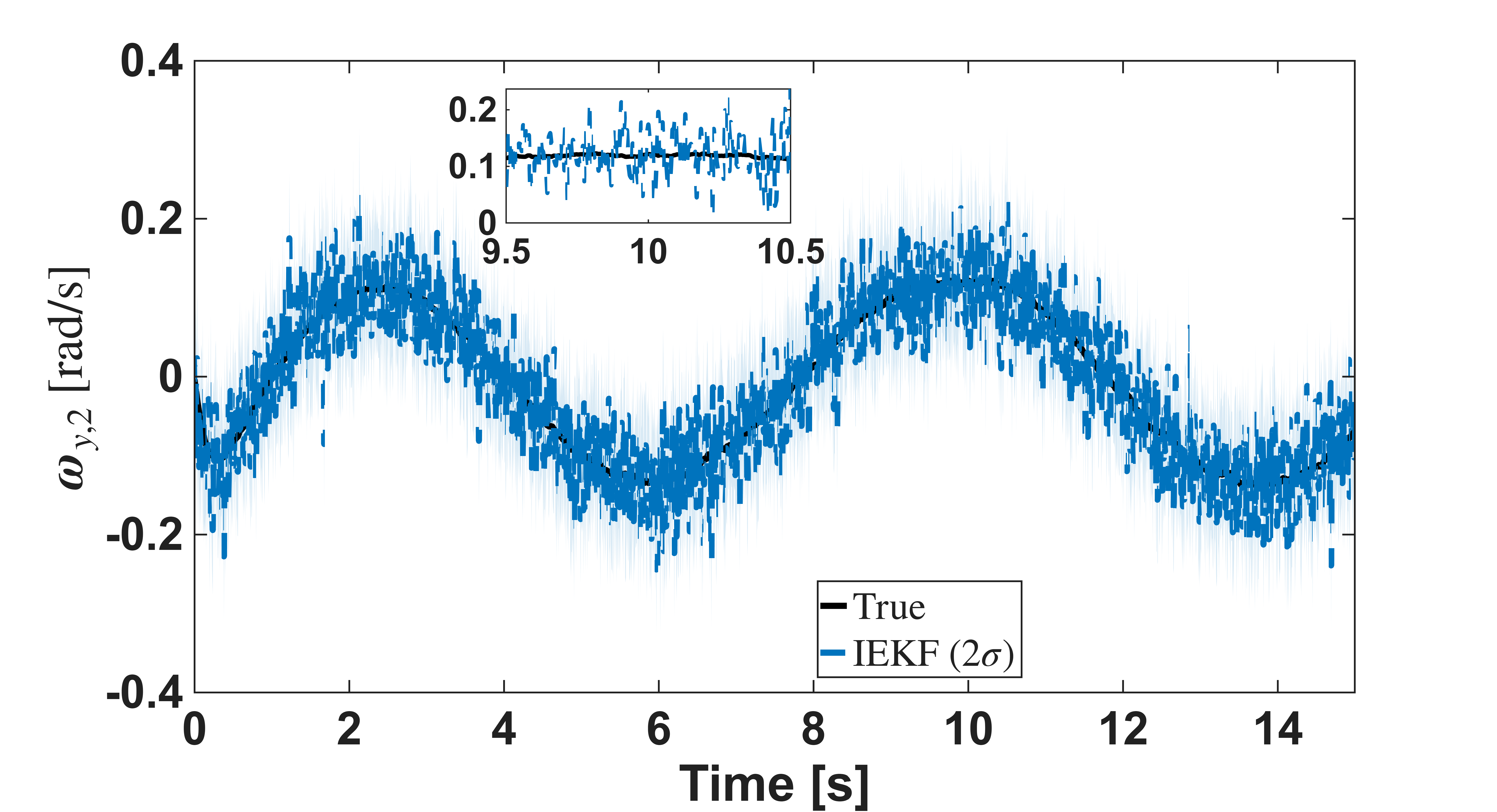}
  \label{fig:V2omy}}
\hfill
\subfloat[]{%
  \includegraphics[width=0.32\textwidth]{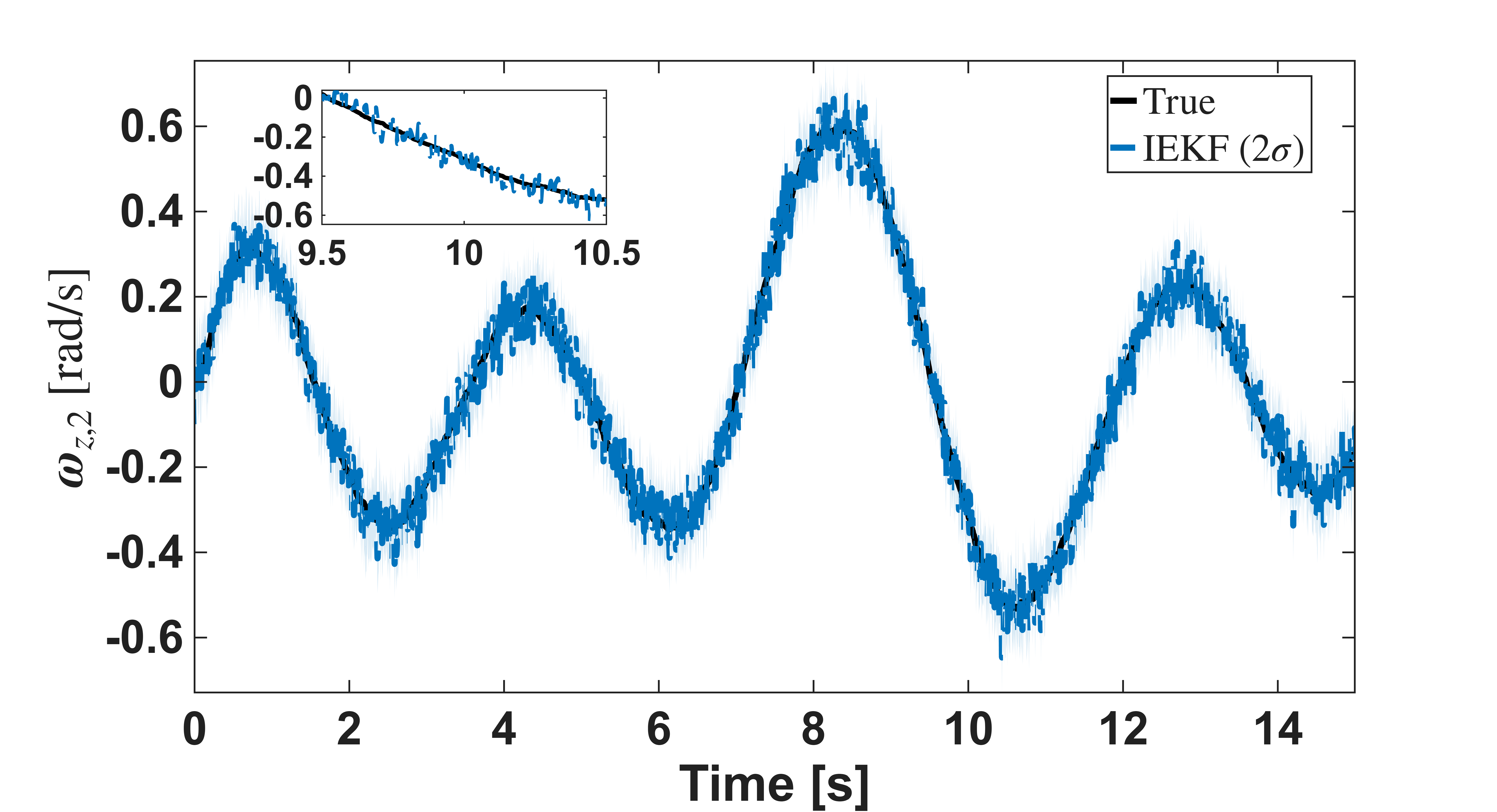}
  \label{fig:V2omz}}\\[0.5em]
\subfloat[]{%
  \includegraphics[width=0.32\textwidth]{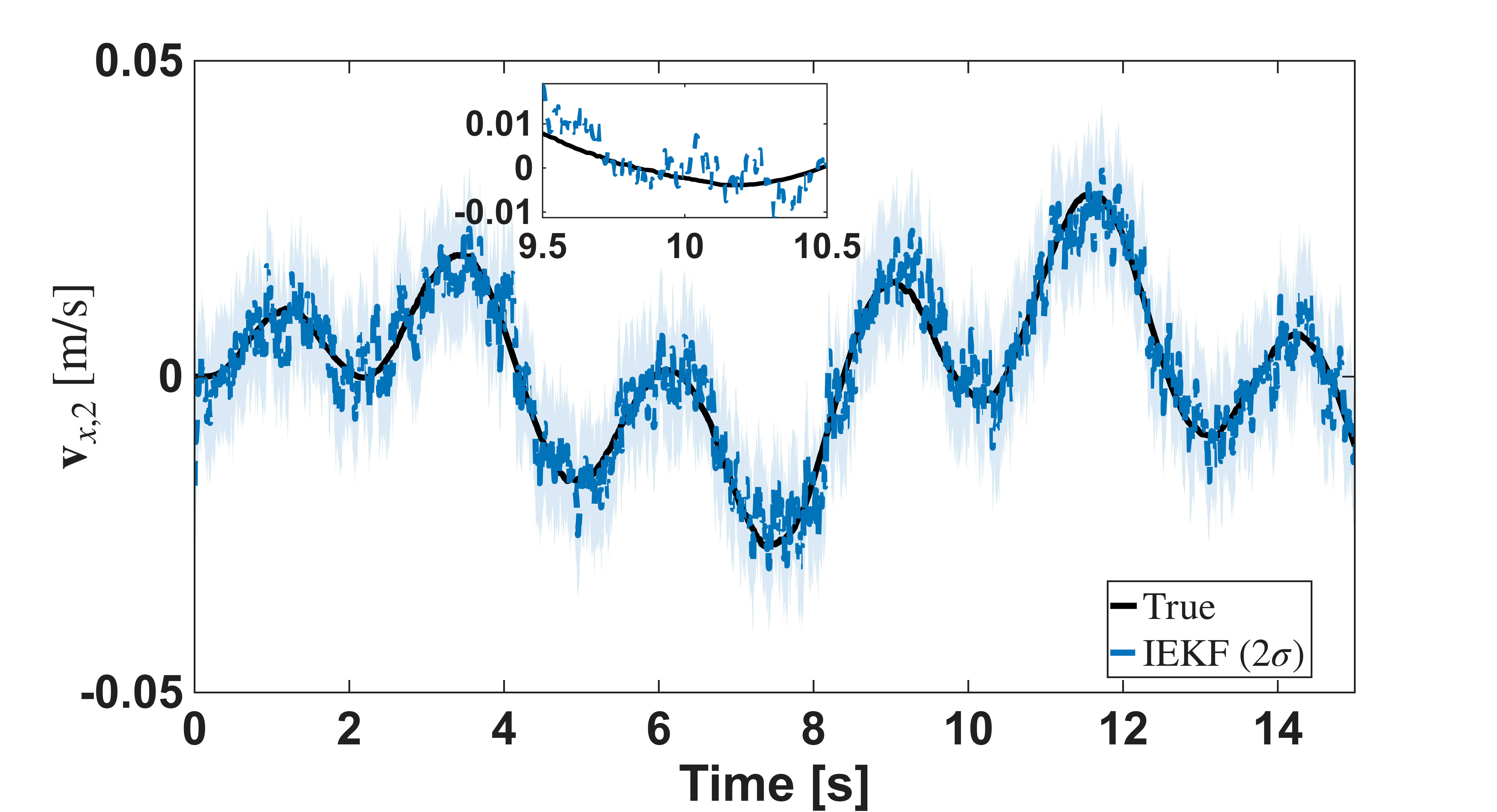}
  \label{fig:V2vx}}
\hfill
\subfloat[]{%
  \includegraphics[width=0.32\textwidth]{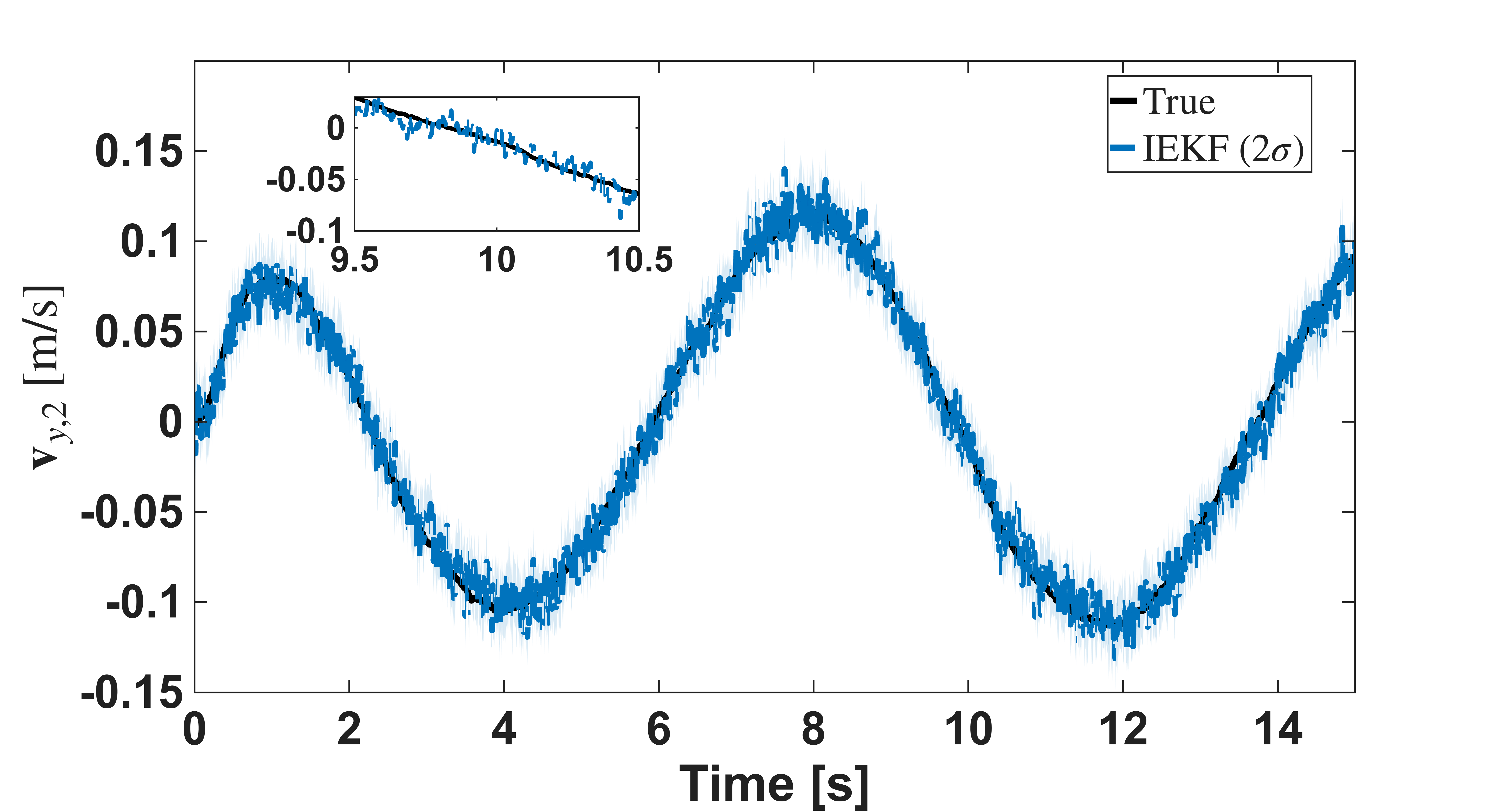}
  \label{fig:V2vy}}
\hfill
\subfloat[]{%
  \includegraphics[width=0.32\textwidth]{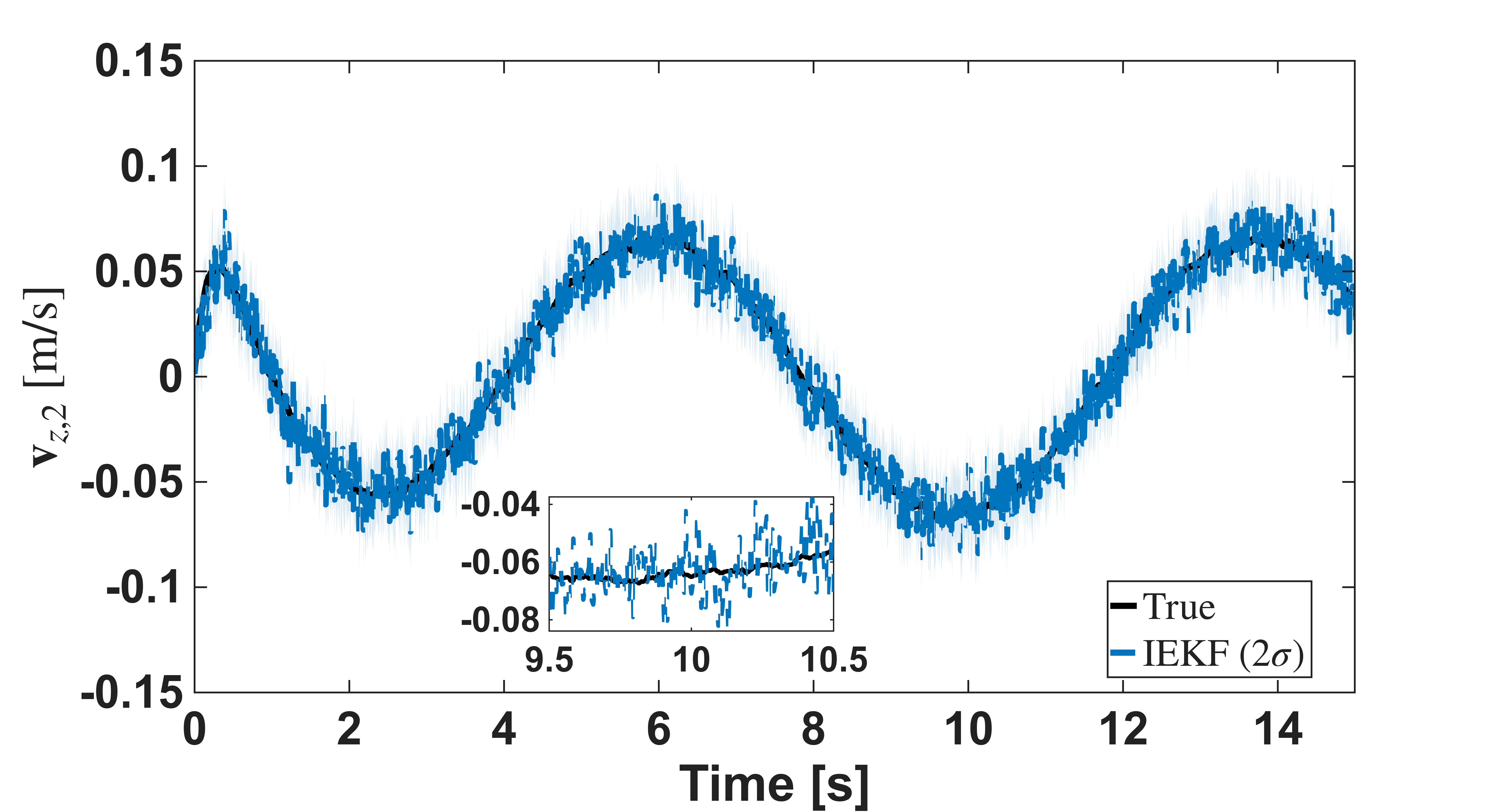}
  \label{fig:V2vz}}

\caption{Body-frame twists --- true (black solid) and
  IEKF estimate (blue dashed) with $2\sigma$ uncertainty
  band (blue shaded), scenario~S1.
  (a)~$\omega_{x,1}$, (b)~$\omega_{y,1}$,
  (c)~$\omega_{z,1}$~[rad/s] --- Link~1 angular velocities;
  (d)~$\omega_{x,2}$, (e)~$\omega_{y,2}$,
  (f)~$\omega_{z,2}$~[rad/s] --- Link~2 angular velocities;
  (g)~$v_{x,2}$, (h)~$v_{y,2}$,
  (i)~$v_{z,2}$~[m/s] --- Link~2 translational velocities,
  nonzero through kinematic coupling from Link~1.
  $\mathbf{v}_1 \equiv \mathbf{0}$ (base fixed).
  The $2\sigma$ band envelops the true value throughout,
  confirming covariance consistency~(\ref*{eq:P_pred_expanded}).}
\label{fig:sim_twists}
\end{figure*}

\subsubsection*{Accuracy and consistency}

The IEKF achieves the lowest joint-angle RMSE of the three methods on all
joints in both scenarios. On Link~1 it reduces the RMSE of the raw
encoder baseline from $0.503^{\circ}$ to $0.33^{\circ}$ on both
$\theta_{1y}$ and $\theta_{1z}$, a reduction of $33\%$, and that of the
CEKF by $24\%$; these values are the same in S1 and S2. On $\theta_2$
the reduction relative to the raw encoders is $26\%$ in S1 and $13\%$ in
S2, as the IEKF RMSE rises from $0.525^{\circ}$ to $0.621^{\circ}$. This
is consistent with the state-dependent Coriolis noise, which grows with
the twist magnitude (Remark~\ref{rem:Qad_vanish}). The CEKF, a linear
joint-space filter, is insensitive to the amplitude.

Figure~\ref{fig:sim_ex1} details Link~2, whose prediction depends on the
posterior of Link~1 through the Adjoint propagation~(\ref*{eq:V_pred_chain})
and which is therefore the most demanding link of the chain. Panels (a)
and (b) show the error time series and the RMSE of $\theta_2$ for the
three methods, and panel (c) the resulting end-effector path, which
confirms the three-dimensional motion of the chain. Panel (d) shows the
NEES of Link~2 on the $n_f = 7$ free error coordinates
$[\phi_z,\,\tilde{\mathbf{V}}_2^\top]^\top$~\citep{bar2004}: its mean of
$3.5$ lies inside the $95\%$ band $[1.7,\,16.0]$ and below the ideal value
of $7$, indicating a consistent and conservative covariance.

\begin{figure*}[!t]
\centering
\subfloat[]{%
  \includegraphics[width=0.49\textwidth]{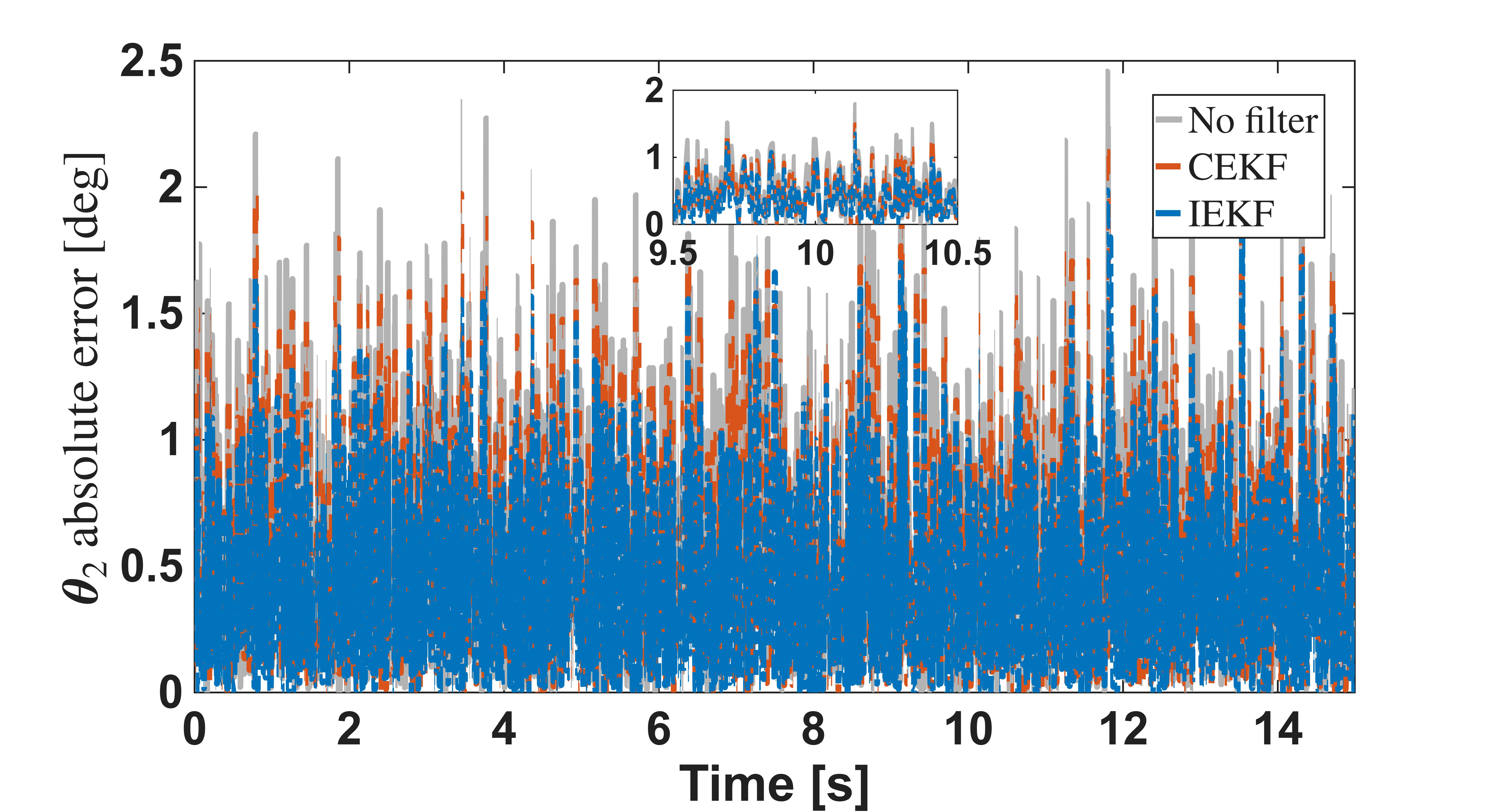}
  \label{fig:err_q2}}
\hfill
\subfloat[]{%
  \includegraphics[width=0.49\textwidth]{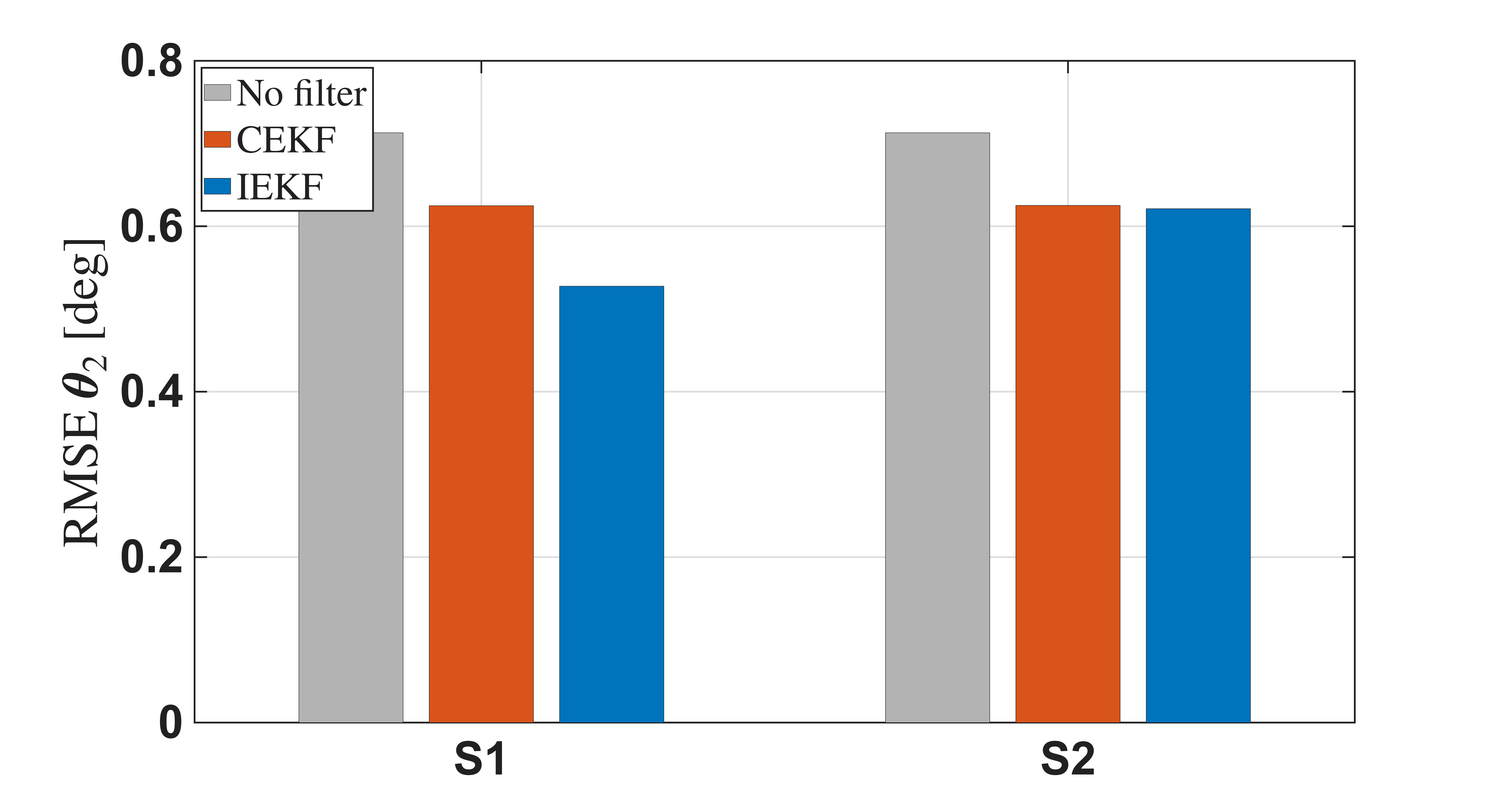}
  \label{fig:rmse_q2}}\\[0.5em]
\subfloat[]{%
  \includegraphics[width=0.49\textwidth]{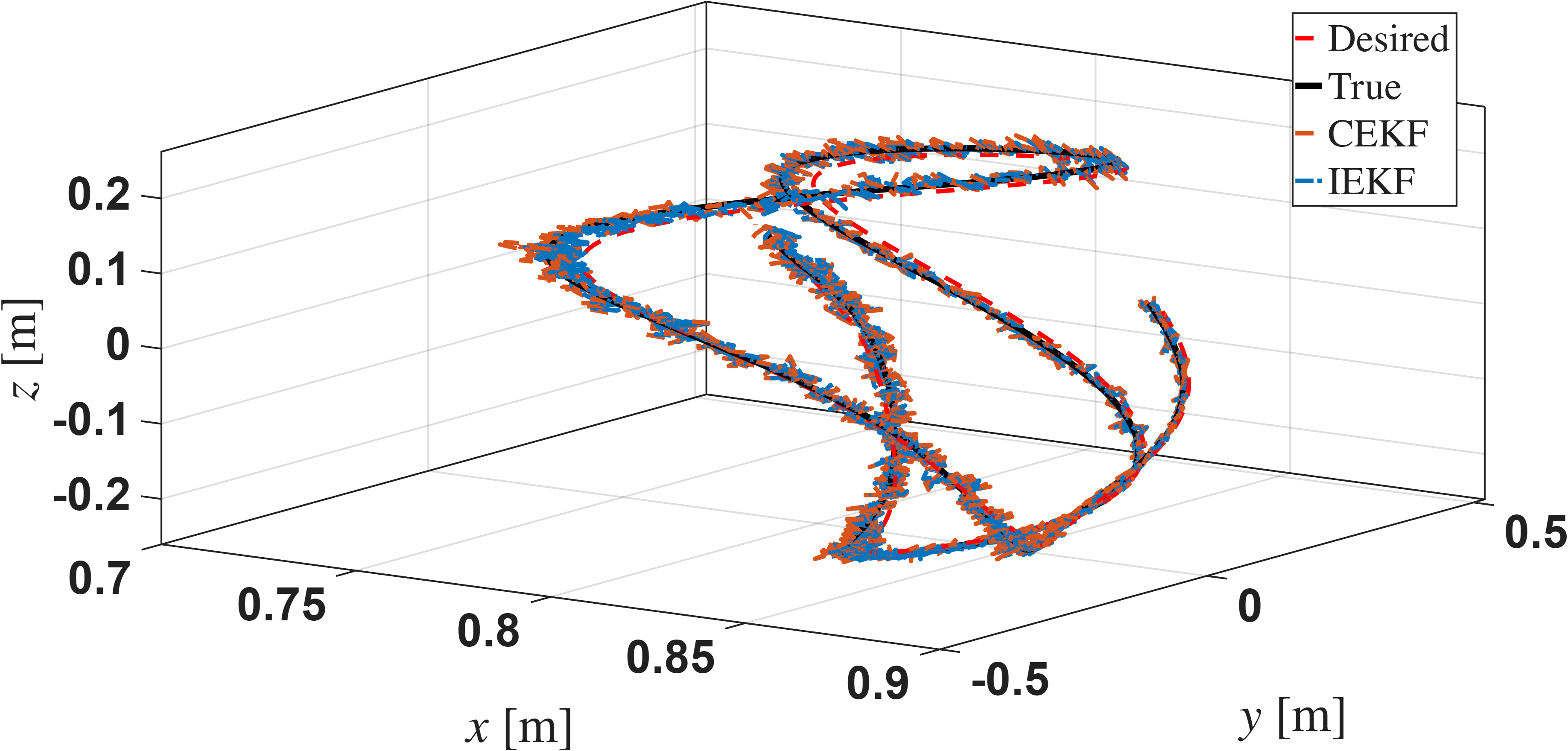}
  \label{fig:ee_3d}}
\hfill
\subfloat[]{%
  \includegraphics[width=0.49\textwidth]{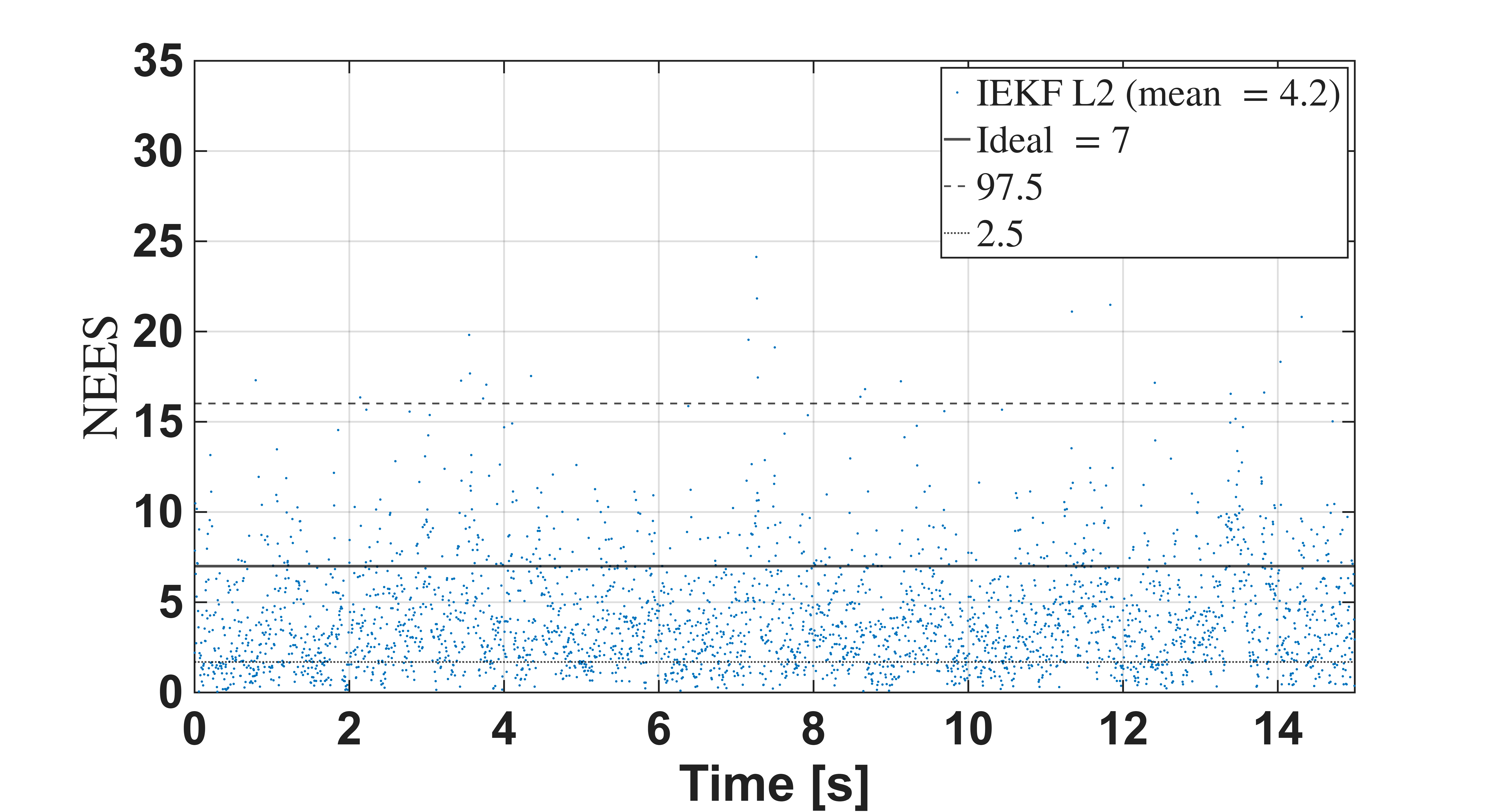}
  \label{fig:nees_l2}}
\caption{Example~1, Link~2. (a)~Absolute error
  $|\hat{\theta}_2 - \theta_2|$~[deg], scenario~S1: no filter (grey),
  CEKF (orange dashed), IEKF (blue dash-dotted). (b)~RMSE of $\theta_2$
  [deg] for scenarios S1 and S2: no filter (grey), CEKF (orange), IEKF
  (blue). (c)~End-effector trajectory, scenario~S1, $t \geq 2$~s: desired
  (red dashed), true (black solid), CEKF (orange dashed) and IEKF (blue
  dash-dotted). (d)~NEES, scenario~S1: ideal ($7$, solid), $97.5\%$
  ($16.0$, dashed) and $2.5\%$ ($1.7$, dotted) bounds.}
\label{fig:sim_ex1}
\end{figure*}
\FloatBarrier
\subsection*{Example 2: Pure-IMU estimation for a rigid three-link rammer-boom manipulator}
\label{sec:example2}

While Example~1 demonstrates the proposed filter in its general
encoder--IMU fusion configuration, the modular structure admits a more
demanding special case: \emph{proprioceptive-only} estimation, in which
the encoder channel is removed entirely
($\sigma_{\mathrm{enc},i}^2\!\to\!\infty$ in \eqref{eq:Nobs}) and the
joint-angle information is recovered exclusively from the link-mounted
IMUs. This example evaluates that configuration in a software-in-the-loop
simulation of an industrial-scale hydraulic rammer-boom manipulator.

\subsubsection*{System description}

The mechanism (Fig.~\ref{fig:rammer_setup}) is a four-link hydraulic
breaker manipulator modeled in Simulink/Simscape Multibody: a slew base,
held stationary so that the mechanism is planar with all joint axes
$\hat{\mathbf{z}}_i$ normal to a vertical plane, followed by three
actuated revolute joints, namely lift, tilt and breaker, with joint
angles $\theta_1$, $\theta_2$, $\theta_3$. The base is link~1 and the
moving links are $i = 2, 3, 4$, so that joint $J_i$ connects link $i-1$
to link $i$ and carries the angle $\theta_{i-1}$. The joint origins are
separated by ${}^{2}\mathbf{r}_{J_3}=[2.750,\,0,\,0]^\top$~[m] and
${}^{3}\mathbf{r}_{J_4}=[3.500,\,0,\,0]^\top$~[m], and the inertial parameters used in the filter prediction are listed in
Table~\ref{tab:ex2params} (the breaker link aggregates the arm and tool
bodies through the parallel-axis theorem). One six-axis IMU is rigidly
mounted on each moving link at position ${}^{i}\mathbf{r}_{S_i}$ in
${}^{i}\!\mathcal{F}$, referenced to the proximal joint
(Table~\ref{tab:ex2params}); the out-of-plane components drop out of the
planar lever kinematics, so only the in-plane components enter the
estimator. Gravity in the boom-plane coordinates is
$\mathbf{g}=[0,\,-9.81,\,0]^\top$~[m/s$^2$], and each IMU delivers the
gyroscope and accelerometer measurements
\eqref{eq:gyro_meas}--\eqref{eq:acc_meas} at its mounting point. The
boom is driven by its joint-space controller through a $60$~[s]
maneuver that starts with a hydraulic start-up transient. The filter
runs at $1$~[kHz] on the IMU signals only; the joint encoders of the
multibody model serve as ground truth and do not enter the estimator.

\begin{figure}[t]
\centering
\includegraphics[width=0.48\linewidth]{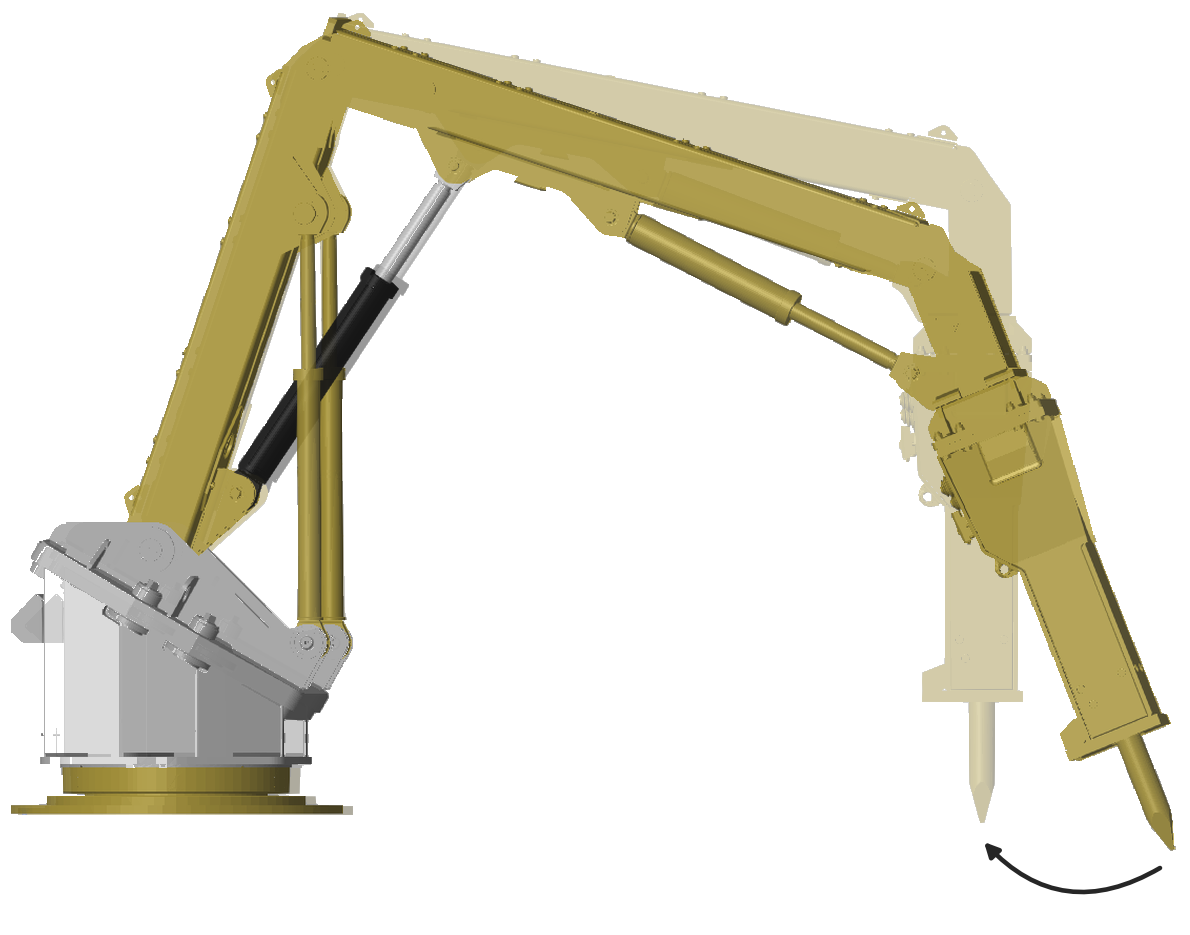}\hfill
\includegraphics[width=0.48\linewidth]{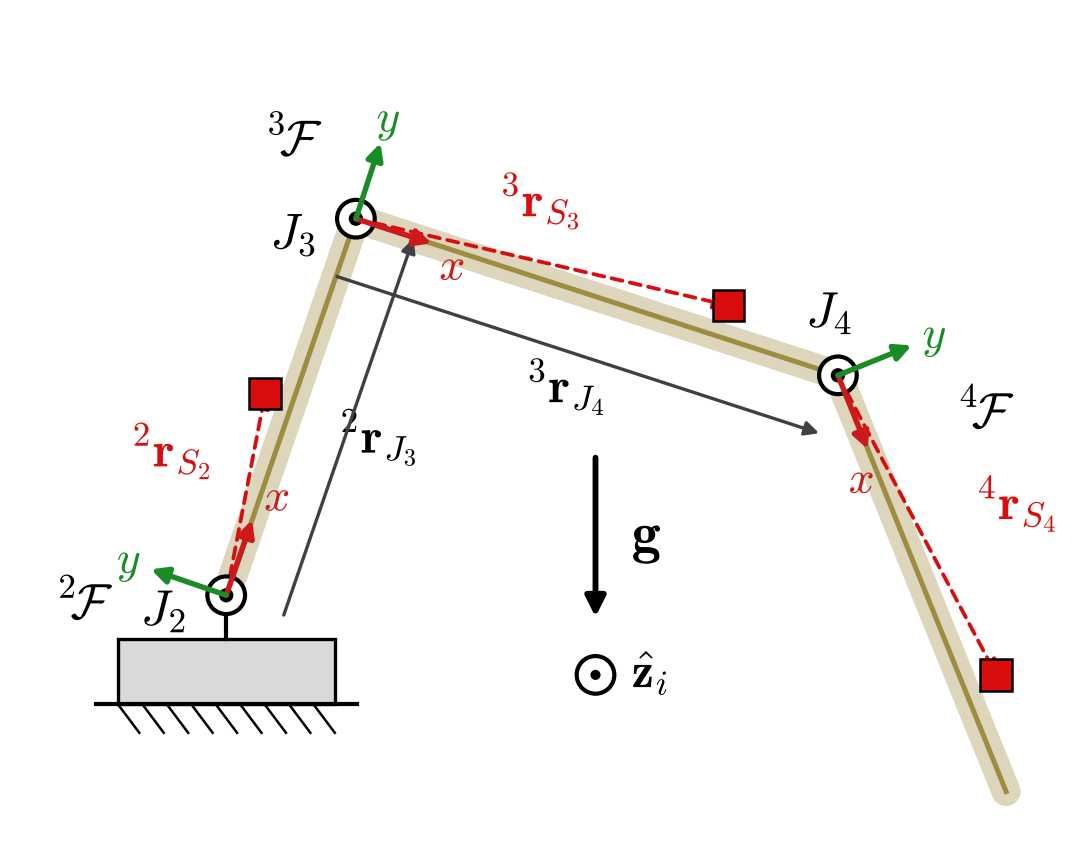}
\caption{Example 2 test system. Left: Simscape Multibody model of the
four-link hydraulic rammer-boom manipulator in double exposure: initial
configuration (solid) and final configuration (faded) of the maneuver;
the arrow indicates the motion of the tool tip.
The slew base is stationary, so the mechanism evolves in a vertical
plane. Right: kinematic schematic at the initial configuration, drawn to scale:
link frames ${}^{i}\!\mathcal{F}$ with origins at the proximal joints
$J_i$ and joint axes $\hat{\mathbf{z}}_i$ normal to the boom plane; IMU
positions ${}^{i}\mathbf{r}_{S_i}$ (red squares, in-plane components) and
joint offsets ${}^{2}\mathbf{r}_{J_3}$, ${}^{3}\mathbf{r}_{J_4}$; gravity
$\mathbf{g}$ lies in the $x$--$y$ plane of every link frame.}
\label{fig:rammer_setup}
\end{figure}

\subsubsection*{Filter configuration relative to the general design}

The estimator is the filter of
Sections~\ref{sec:4}--\ref{sec:7} per link, on the state
$(g_i,\mathbf{V}_i)$ augmented with the gyroscope bias
$\mathbf{b}_{\omega,i}\in\mathbb{R}^3$: the left-invariant propagation
driven by the plant SDE \eqref{eq:sde_g}--\eqref{eq:sde_V}, the
predicted covariance \eqref{eq:P_pred_expanded} with the modular
composition
$\mathbf{N}_i^{\mathrm{total}}
 = \mathbf{Ad}_{i(i-1)}\mathbf{P}^{VV}_{i-1}\mathbf{Ad}_{i(i-1)}^\top
 + \mathbf{N}_i^{\mathrm{IMU}}$ of \eqref{eq:Ntotal} entering the
diffusion, and the Joseph-form update \eqref{eq:P_update}. Exactly one
element of the observation model \eqref{eq:meas_full} is substituted.

\emph{Virtual encoder by gravity referencing.} The encoder row
\eqref{eq:enc_meas} is replaced by a joint-angle measurement synthesized
from the accelerometers of the two links adjacent to joint $i$. Since
two points of the same rigid link differ in specific force only through
the rigid-body transport terms, the accelerometer output
\eqref{eq:acc_meas} of link $i{-}1$ is transported along that link to
the origin $J_i$ of ${}^{i}\!\mathcal{F}$,
\begin{align}
  {}^{i-1}\mathbf{f}_{J_i} \;=\; \mathbf{y}_{i-1}^{\mathrm{acc}}
   + \Big([\dot{\boldsymbol{\omega}}_{i-1}]^\times
   + [\boldsymbol{\omega}_{i-1}]^\times[\boldsymbol{\omega}_{i-1}]^\times\Big)
     \mathbf{d}_i,
  \qquad
  \mathbf{d}_i = {}^{i-1}\mathbf{r}_{J_i}-{}^{i-1}\mathbf{r}_{S_{i-1}},
  \label{eq:parent_transport}
\end{align}
and the accelerometer output of link $i$ is transported to the same
point,
${}^{i}\mathbf{f}_{J_i}=\mathbf{y}_{i}^{\mathrm{acc}}
 - \big([\dot{\boldsymbol{\omega}}_{i}]^\times
 + [\boldsymbol{\omega}_{i}]^\times[\boldsymbol{\omega}_{i}]^\times\big)
 {}^{i}\mathbf{r}_{S_i}$.
The joint-axis component of $\dot{\boldsymbol{\omega}}$ is the
least-squares slope of the gyroscope signal over a sliding $50$~[ms]
window. The two vectors are the specific force of the \emph{same
physical point} $J_i$, expressed in ${}^{i-1}\!\mathcal{F}$ and
${}^{i}\!\mathcal{F}$ respectively, so they differ only by the relative
rotation across joint $i$; the joint angle is measured as the angle
between them about the joint axis,
\begin{align}
  y_i^{\mathrm{tilt}} \;=\; \operatorname{atan2}\!\big(
  \hat{\mathbf{z}}_i^\top\big({}^{i}\mathbf{f}_{J_i}\times{}^{i-1}\mathbf{f}_{J_i}\big),\;
  {}^{i}\mathbf{f}_{J_i}^\top\,{}^{i-1}\mathbf{f}_{J_i}\big).
  \label{eq:virtual_encoder}
\end{align}
The carrier acceleration of the joint, which is the dominant disturbance
for gravity-referenced attitude on a moving chain, is common to both
vectors and cancels \emph{exactly}, not statistically. For the first
moving link the upstream reference degenerates to gravity itself,
${}^{1}\mathbf{f}_{J_2}=-\mathbf{g}$, recovering the classical
inclinometer formula. The measurement \eqref{eq:virtual_encoder} is
low-pass filtered with a $50$~[ms] time constant and occupies the
encoder row of \eqref{eq:meas_full}, with the same observation function
$q_i(g_i)$ and the same row of $\mathbf{H}_{\mathrm{obs},i}$
(Section~\ref{sec:7.1}). Only the last entry of
$\mathbf{N}_i^{\mathrm{obs}}$ in \eqref{eq:Nobs} changes, becoming
state-dependent,
\begin{align}
  \sigma_{\mathrm{enc},i}^2 \;\longrightarrow\;
  \sigma_{\min}^2
  + \sum_{s\in\{i-1,\,i\}}
    \frac{\sigma_a^2 + \sigma_{\dot\omega,s}^2
      \big(\mathbf{n}_s^\top[\hat{\mathbf{z}}]^\times\mathbf{r}_s\big)^2}
         {\|\mathbf{f}_s\|^2}
  \;+\; \hat{\mathbf{z}}_i^\top\,\mathbf{P}^{gg}_{i-1}\,\hat{\mathbf{z}}_i,
  \label{eq:adaptive_R}
\end{align}
where $s$ runs over the parent and child sides of the measurement,
$\mathbf{f}_s$ is the transported specific force of that side,
$\mathbf{n}_s$ its in-plane unit normal, $\mathbf{r}_s$ the transport arm
($\mathbf{d}_i$ or ${}^{i}\mathbf{r}_{S_i}$), and $\sigma_a$ the
accelerometer noise. The sum is the first-order propagation of the
accelerometer noise and of the residual of the
$\dot{\boldsymbol{\omega}}$-dependent transport term into the angle
\eqref{eq:virtual_encoder}; it deweights the measurement exactly in the
configurations where the tangential transport term is aligned with the
measurement direction. The angular-acceleration uncertainty is evaluated
from the data, $\sigma_{\dot\omega,s}^2 = \sigma_{\mathrm{LS}}^2 +
\big(\hat{\dot\omega}_s(t)-\hat{\dot\omega}_s(t-T_w/2)\big)^2$, with
$\sigma_{\mathrm{LS}}$ the least-squares slope noise implied by
$\mathbf{N}_{\omega,i}$ and $T_w$ the slope window, so that it is large
during vibration and small in smooth motion. The last term of
\eqref{eq:adaptive_R} propagates the upstream attitude variance about
the joint axis in accordance with the modular chain treatment of
Section~\ref{sec:6}. When the transport terms exceed $1.2\,\|\mathbf{g}\|$
the angle \eqref{eq:virtual_encoder} carries no usable gravity
information and the row is removed from $\mathbf{H}_{\mathrm{obs},i}$;
measurements failing a $\chi^2(1)$ innovation gate are deweighted. In
summary, the general design is recovered verbatim under the
substitutions $\sigma_{\mathrm{enc},i}^2\!\to\!\infty$,
$y_i^{\mathrm{enc}}\!\to\!y_i^{\mathrm{tilt}}$ of
\eqref{eq:virtual_encoder}--\eqref{eq:adaptive_R}: the framework passes
from exteroceptive--proprioceptive fusion to pure proprioception through
a change of the measurement and its covariance alone.

\subsubsection*{Sensor error injection}

The ideal IMU signals of the multibody model are corrupted at every
sample with zero-mean white noise of standard deviation
$\sigma_\omega = 5\times10^{-3}$~[rad/s] (gyroscope) and
$\sigma_a = 5\times10^{-2}$~[m/s$^2$] (accelerometer) per axis, and with
constant gyroscope biases about the joint axes
(Table~\ref{tab:ex2params}). The gyroscope covariance of the filter is
matched, $\mathbf{N}_{\omega,i} = \sigma_\omega^2\mathbf{I}_3$. The
injected accelerometer noise has the spectral density
$\sigma_a^2\Delta t = 2.5\times10^{-6}$~[(m/s$^2$)$^2$/Hz]; the filter uses
$\mathbf{N}_{a,i} = \sigma_a^2\mathbf{I}_3$ in~\eqref{eq:Nimu}, which is
conservative by the factor $1/\Delta t$ and enters only the Coriolis
diffusion through $\mathbf{N}_i^{\mathrm{total}}$, while the
virtual-encoder covariance~\eqref{eq:adaptive_R} uses the per-sample
value $\sigma_a$. No prior bias knowledge is given to the
filter: the bias states start at zero with covariance
$(5\times10^{-3})^2\,\mathbf{I}_3$ and are released after $3$~[s], once
the attitude has settled.

\begin{table}[t]
\centering
\small
\caption{Example 2 parameters. Top: per-link quantities in the link frame
${}^{i}\!\mathcal{F}$, referenced to the proximal joint. Bottom: sensor and
filter parameters, common to all links.}
\label{tab:ex2params}
\renewcommand{\arraystretch}{1.15}
\setlength{\tabcolsep}{4pt}
\begin{tabular}{@{}l ccc@{}}
\toprule
& Lift ($i=2$) & Tilt ($i=3$) & Breaker ($i=4$) \\
\midrule
${}^{i}\mathbf{r}_{S_i}$ (IMU) [m] & $[1.405,\,0.199,\,0.116]$ & $[2.632,\,0.225,\,0.872]$ & $[2.330,\,0.239,\,0]$ \\
${}^{i}\mathbf{r}_{J_{i+1}}$ [m] & $[2.750,\,0,\,0]$ & $[3.500,\,0,\,0]$ & \\
$m_i$ [kg] & $1763.5$ & $2100.4$ & $2289.1$ \\
${}^{i}\mathbf{r}_{c_i}$ [m] & $[1.702,\,-0.158,\,0.001]$ & $[1.819,\,0.018,\,0]$ & $[0.918,\,0.004,\,0]$ \\
$\mathrm{diag}(\mathbf{I}_{b_i})$ [kg$\cdot$m$^2$] & $(127.9,\,6420.6,\,6417.1)$ & $(59.8,\,12875.8,\,12894.7)$ & $(127.7,\,1821.7,\,1859.1)$ \\
Gyro bias [mrad/s] & $2.0$ & $-3.0$ & $2.5$ \\
\bottomrule
\end{tabular}

\vspace{0.8em}

\begin{tabular}{@{}l l@{}}
\toprule
Sample time $\Delta t$ [ms] & $1$ \\
Gyroscope noise $\sigma_\omega$ (1$\sigma$/sample) [rad/s] & $5\times10^{-3}$ \\
Accelerometer noise $\sigma_a$ (1$\sigma$/sample) [m/s$^2$] & $5\times10^{-2}$ \\
Filter accelerometer density $\mathbf{N}_{a,i}$ [(m/s$^2$)$^2$/Hz] & $(5\times10^{-2})^2\,\mathbf{I}_3$ \\
Wrench disturbance $\mathbf{Q}_{c,i}^w$ [N$^2$m$^2$/Hz] & $\mathrm{diag}(10\,\mathbf{I}_3,\,\mathbf{I}_3)$ \\
Bias process noise [rad/s/$\sqrt{\mathrm{s}}$] & $2\times10^{-5}$ \\
Tilt floor $\sigma_{\min}$ [rad] & $0.05$ \\
$\dot{\boldsymbol{\omega}}$ slope window $T_w$ / tilt low-pass [ms] & $50$ / $50$ \\
Transport-term rejection threshold & $1.2\,\|\mathbf{g}\|$ \\
Filter initial configuration $(\theta_1,\theta_2,\theta_3)$ [deg] & $(71,\,-89,\,-50)$ \\
\bottomrule
\end{tabular}
\end{table}

\subsubsection*{Results}

Fig.~\ref{fig:sil_angles} compares the estimated joint angles with the
encoder reference over the $60$~[s] maneuver. During the first
$1.5$~[s] the hydraulic start-up excites oscillations with angular
rates of about $40^{\circ}$/s; the transport terms then reject or
deweight the virtual encoder, and the estimate is carried by the
gyroscope chain with errors of up to $17^{\circ}$ on the tilt joint.
Once the oscillation decays, the virtual encoder is accepted on all
three joints and the estimates converge within $1.5$~[s]. In steady
state ($t\geq15$~[s]) the joint-angle RMSE is $0.045^{\circ}$,
$0.059^{\circ}$ and $0.064^{\circ}$ for lift, tilt and breaker
(Table~\ref{tab:sil}), with no growth along the chain despite the
$2.75$ and $3.5$~[m] transport arms.

The gyroscope-bias estimates (Fig.~\ref{fig:sil_bias}) converge from
zero to the injected values; the bias about $\hat{\mathbf{z}}_i$ is
observable only through the attitude channel, that is, through
bias-induced attitude drift seen by the virtual encoder, and its
$2\sigma$ band contracts accordingly. The breaker estimate converges
most slowly.

Consistency is assessed with the encoder as reference. The NEES on the
$n_f=2$ free error coordinates $[\phi_z,\,\tilde{\omega}_z]^\top$ of
each link (Fig.~\ref{fig:sil_nees}) has steady-state means of $0.50$,
$0.56$ and $0.62$ against the ideal value of $2$, and the samples outside
the $95\%$ band lie below its lower bound: the covariance is
conservative, uniformly along the chain. The conservatism originates in
the floor $\sigma_{\min}=0.05$~[rad] of \eqref{eq:adaptive_R}, which
exceeds the observed virtual-encoder error ($0.12^{\circ}$ to
$0.19^{\circ}$ rms) also after accounting for its correlation through
the low-pass filter. The NIS of the gyroscope innovation has mean
$0.90$ with $95.5\%$ of the samples inside the $95\%$ $\chi^2(1)$ band
on all links, which confirms the gyroscope noise model
$\mathbf{N}_{\omega,i}$.

\begin{figure*}[!t]
\centering
\subfloat[]{%
  \includegraphics[width=0.32\textwidth]{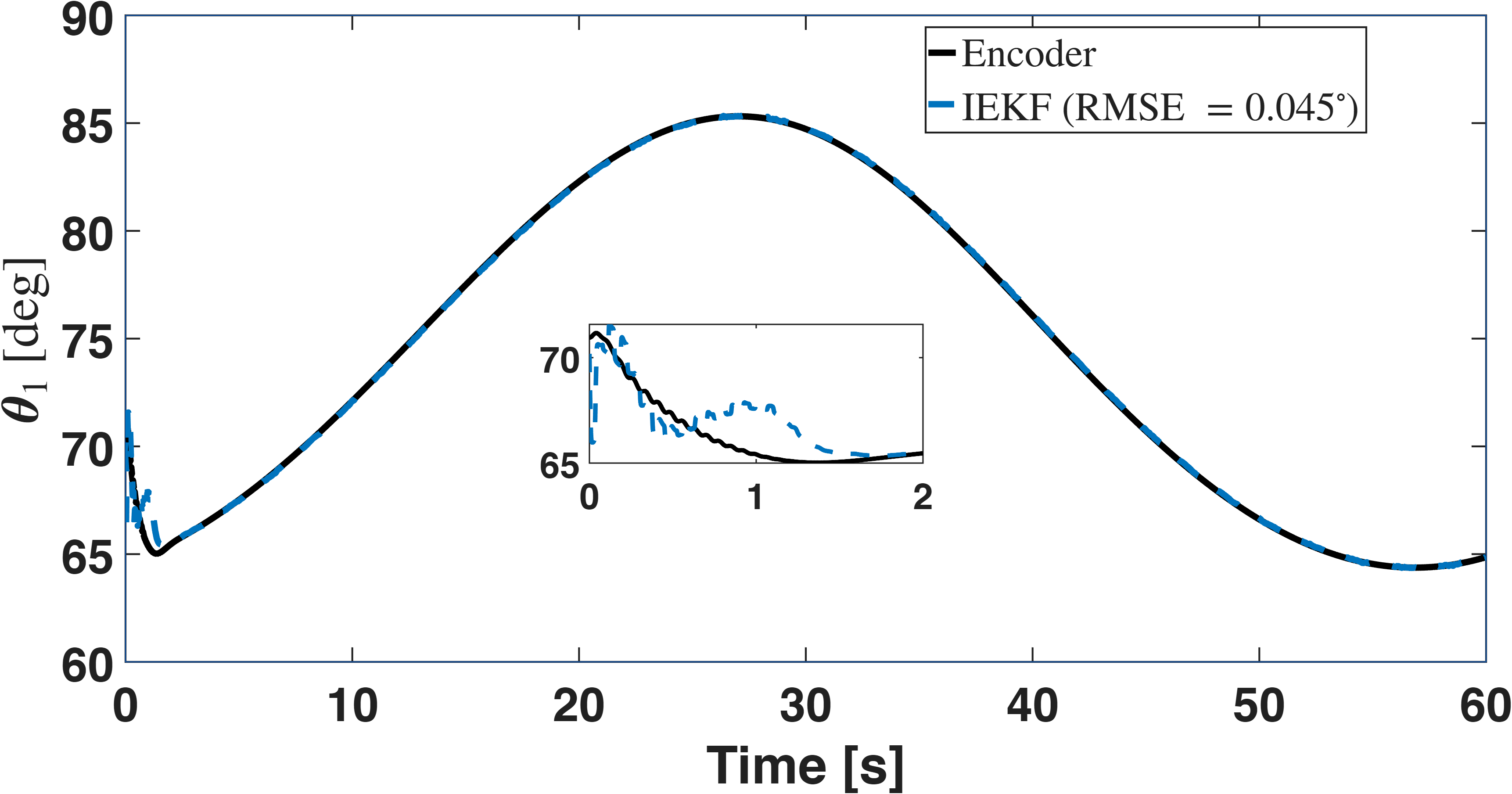}
  \label{fig:sil_q1}}
\hfill
\subfloat[]{%
  \includegraphics[width=0.32\textwidth]{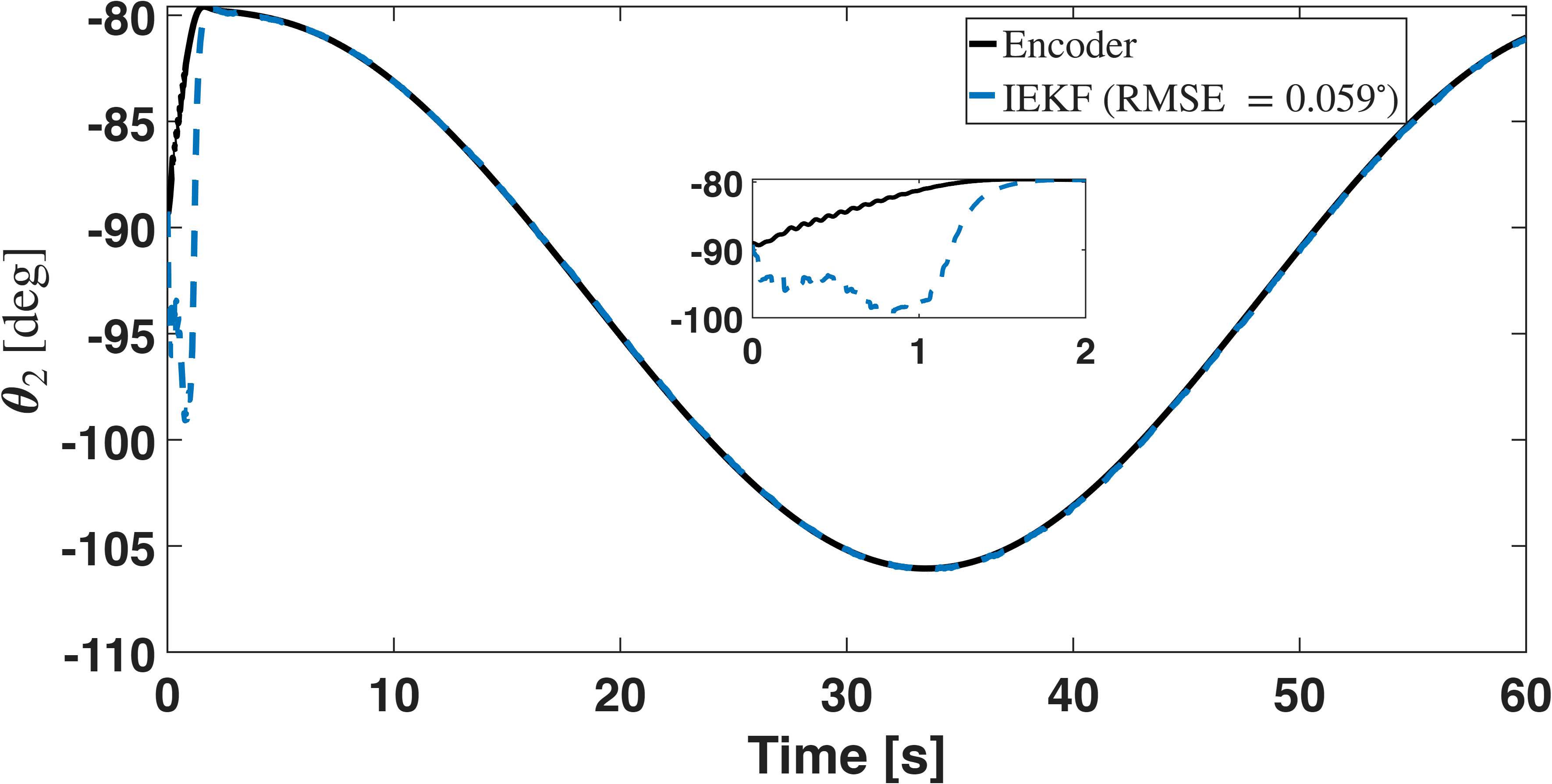}
  \label{fig:sil_q2}}
\hfill
\subfloat[]{%
  \includegraphics[width=0.32\textwidth]{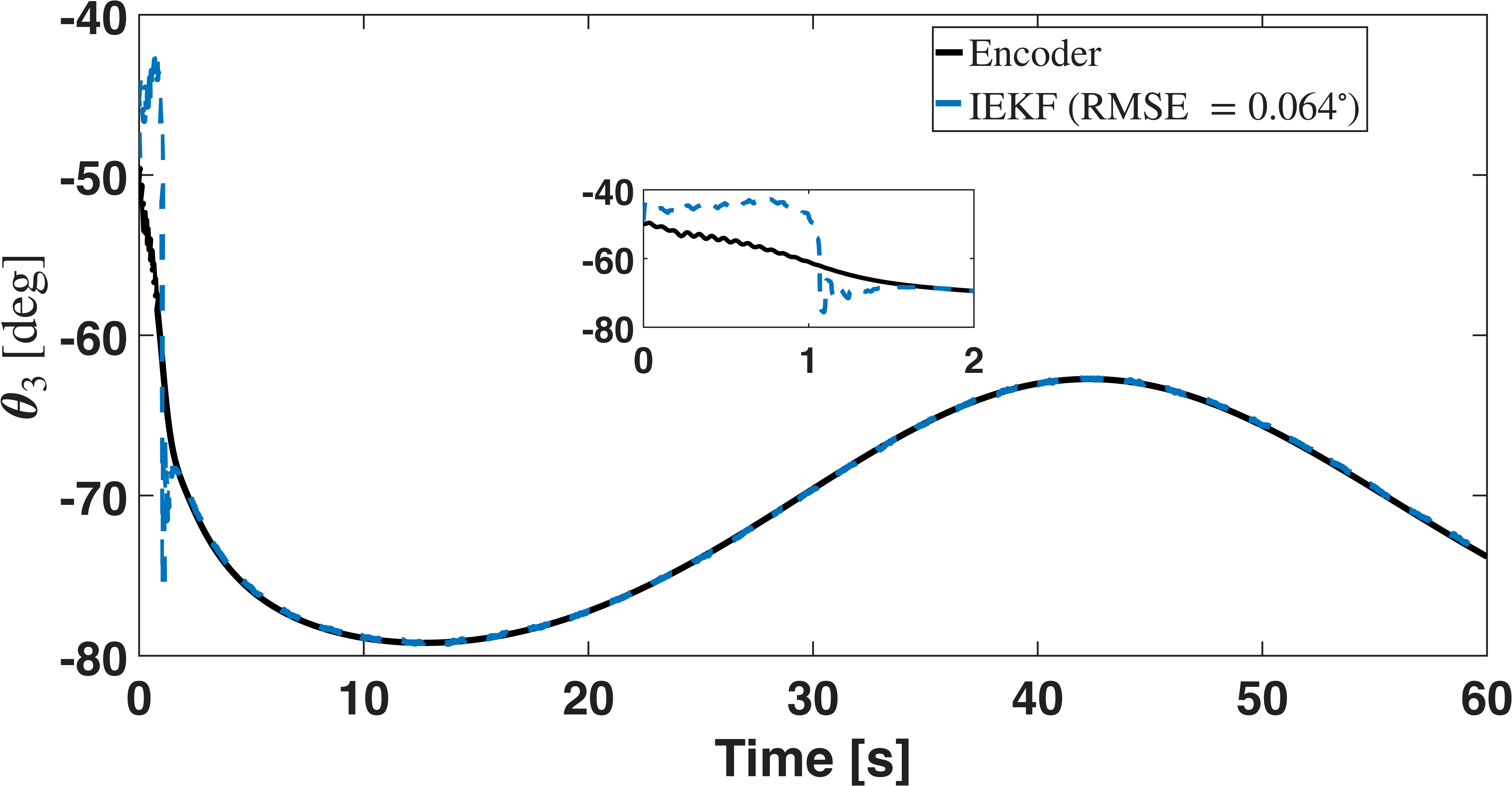}
  \label{fig:sil_q3}}
\caption{Pure-IMU joint-angle estimation, Example~2: encoder reference
  (black solid) and IEKF estimate (blue dashed), (a)~$\theta_1$ (lift),
  (b)~$\theta_2$ (tilt), (c)~$\theta_3$ (breaker); the legend gives the
  steady-state RMSE ($t \geq 15$~s) and the inset enlarges the start-up
  transient. The encoder is used for evaluation only.}
\label{fig:sil_angles}
\end{figure*}

\begin{figure*}[!t]
\centering
\subfloat[]{%
  \includegraphics[width=0.32\textwidth]{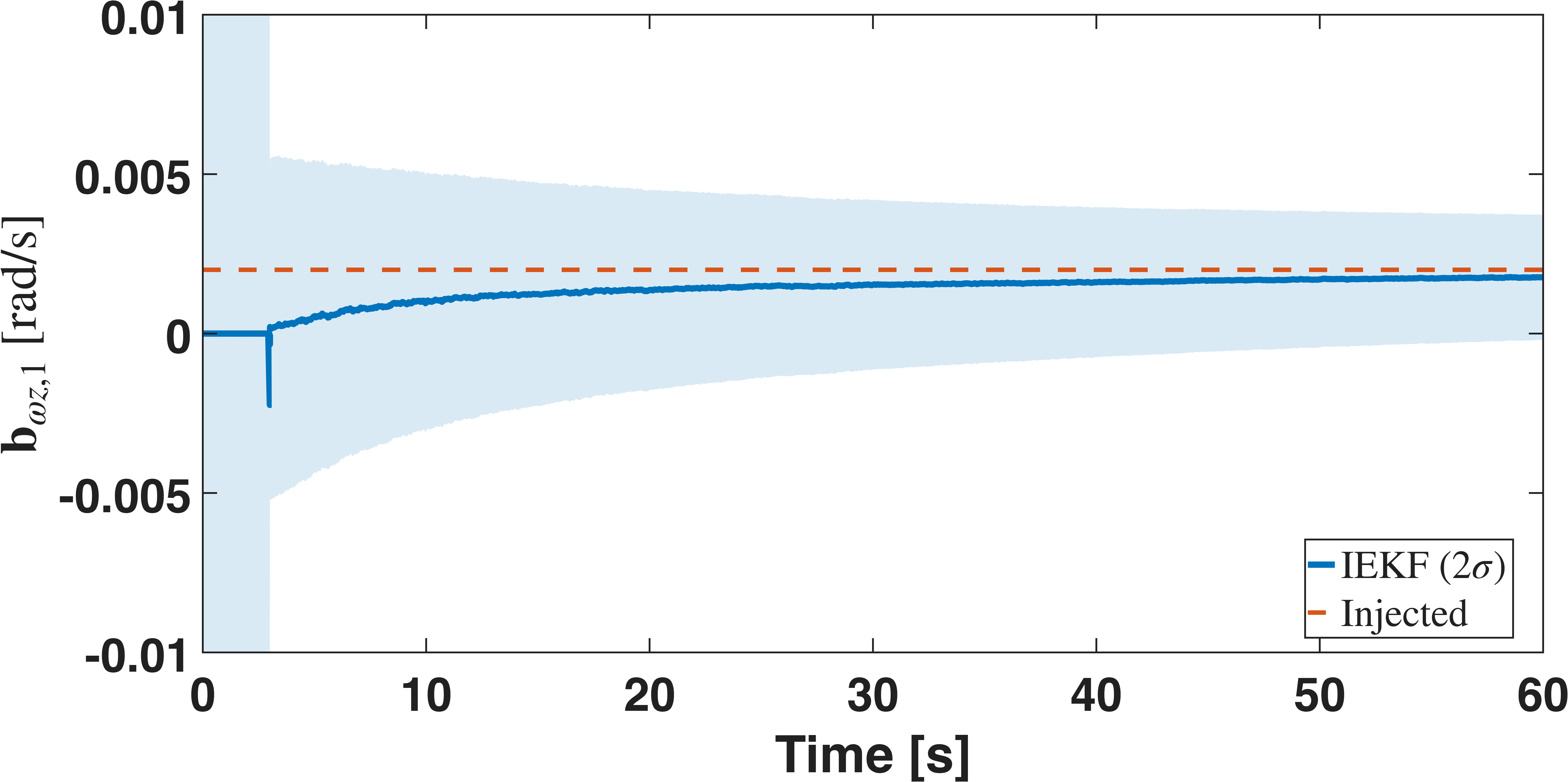}
  \label{fig:sil_b1}}
\hfill
\subfloat[]{%
  \includegraphics[width=0.32\textwidth]{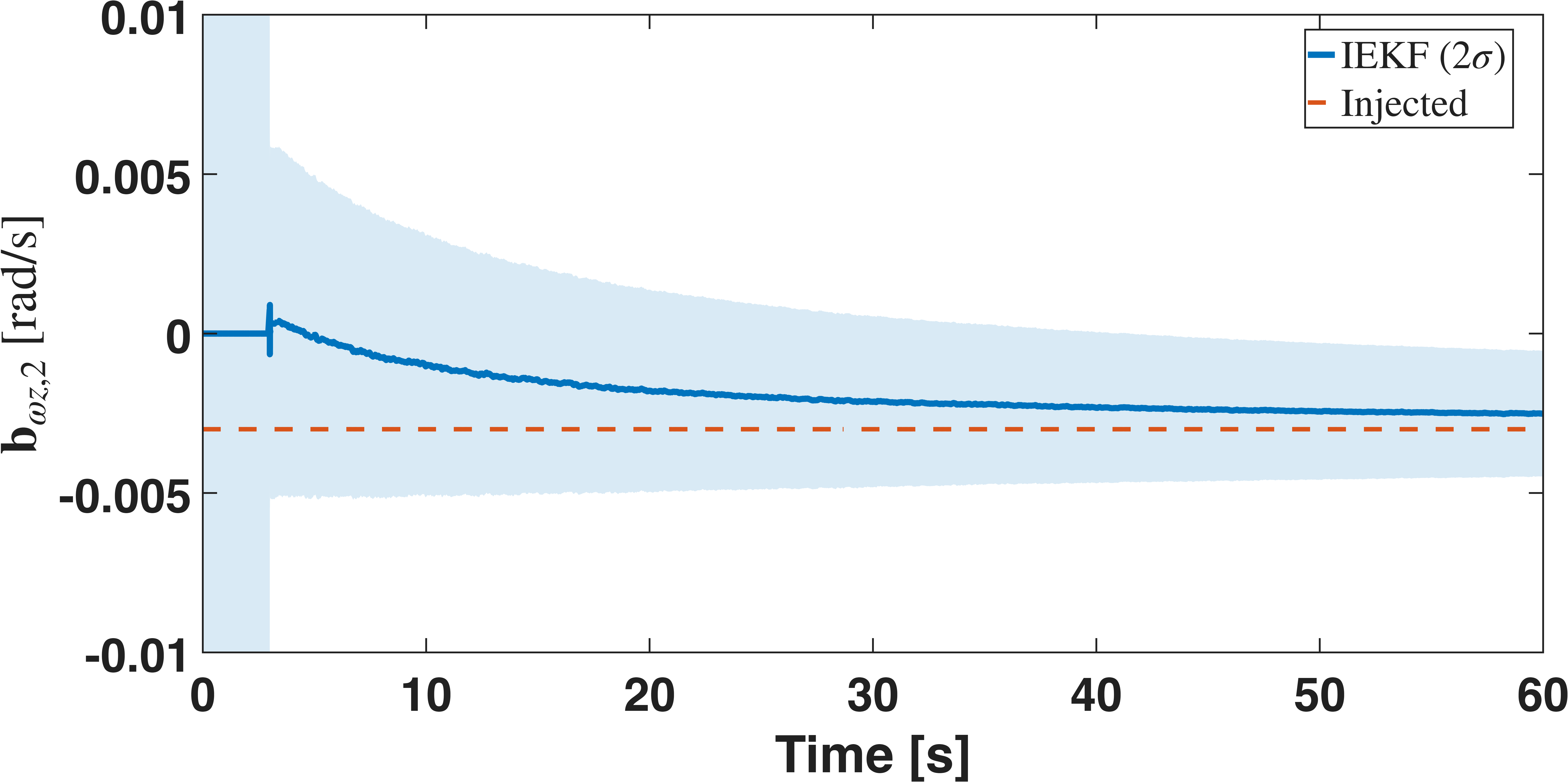}
  \label{fig:sil_b2}}
\hfill
\subfloat[]{%
  \includegraphics[width=0.32\textwidth]{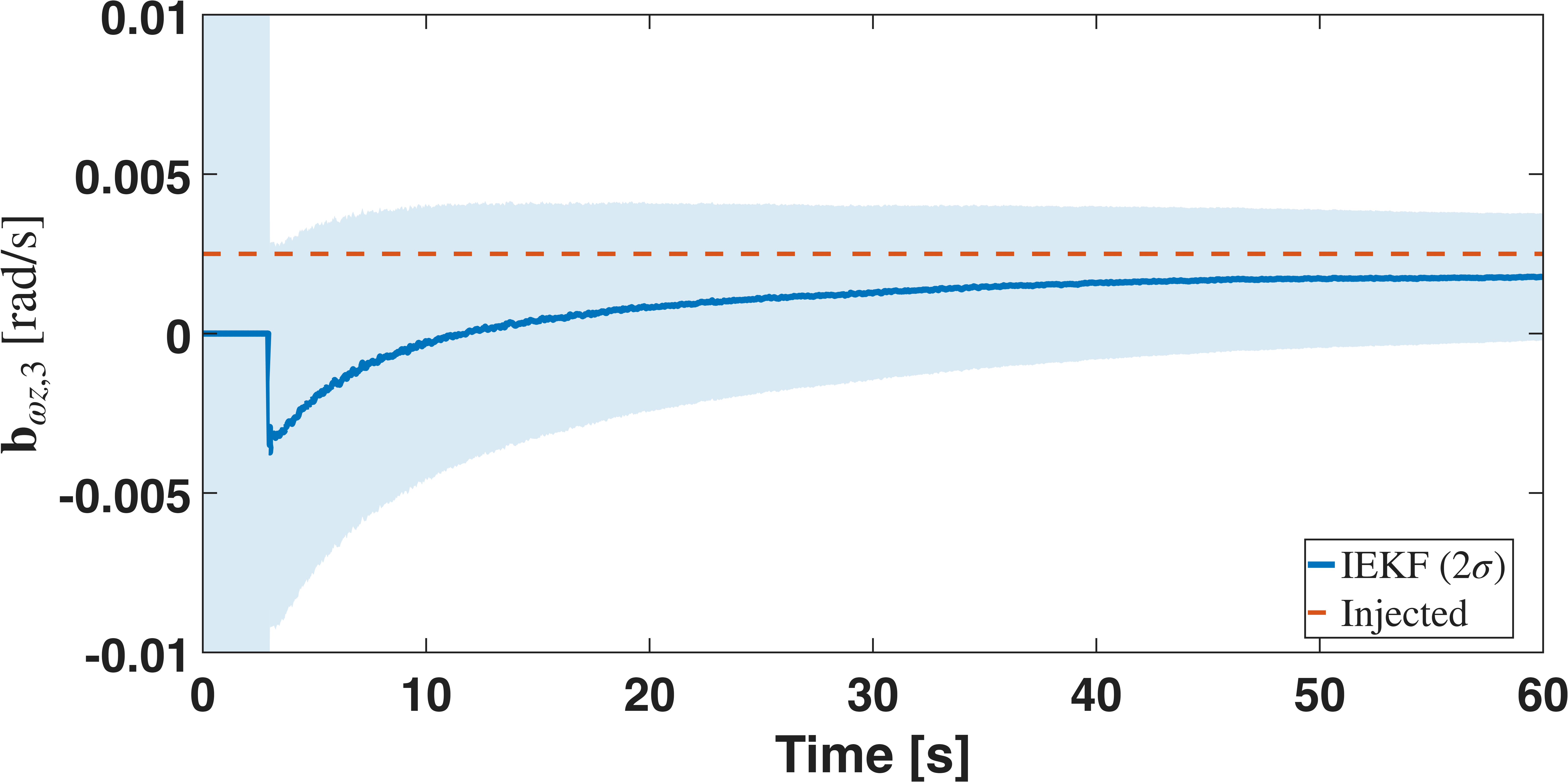}
  \label{fig:sil_b3}}
\caption{Gyroscope-bias estimates about the joint axes (blue solid) with
  $2\sigma$ band (blue shaded) and injected bias (orange dashed),
  (a)~lift, (b)~tilt, (c)~breaker link. Bias estimation is released at
  $t=3$~s.}
\label{fig:sil_bias}
\end{figure*}

\begin{figure*}[!t]
\centering
\subfloat[]{%
  \includegraphics[width=0.32\textwidth]{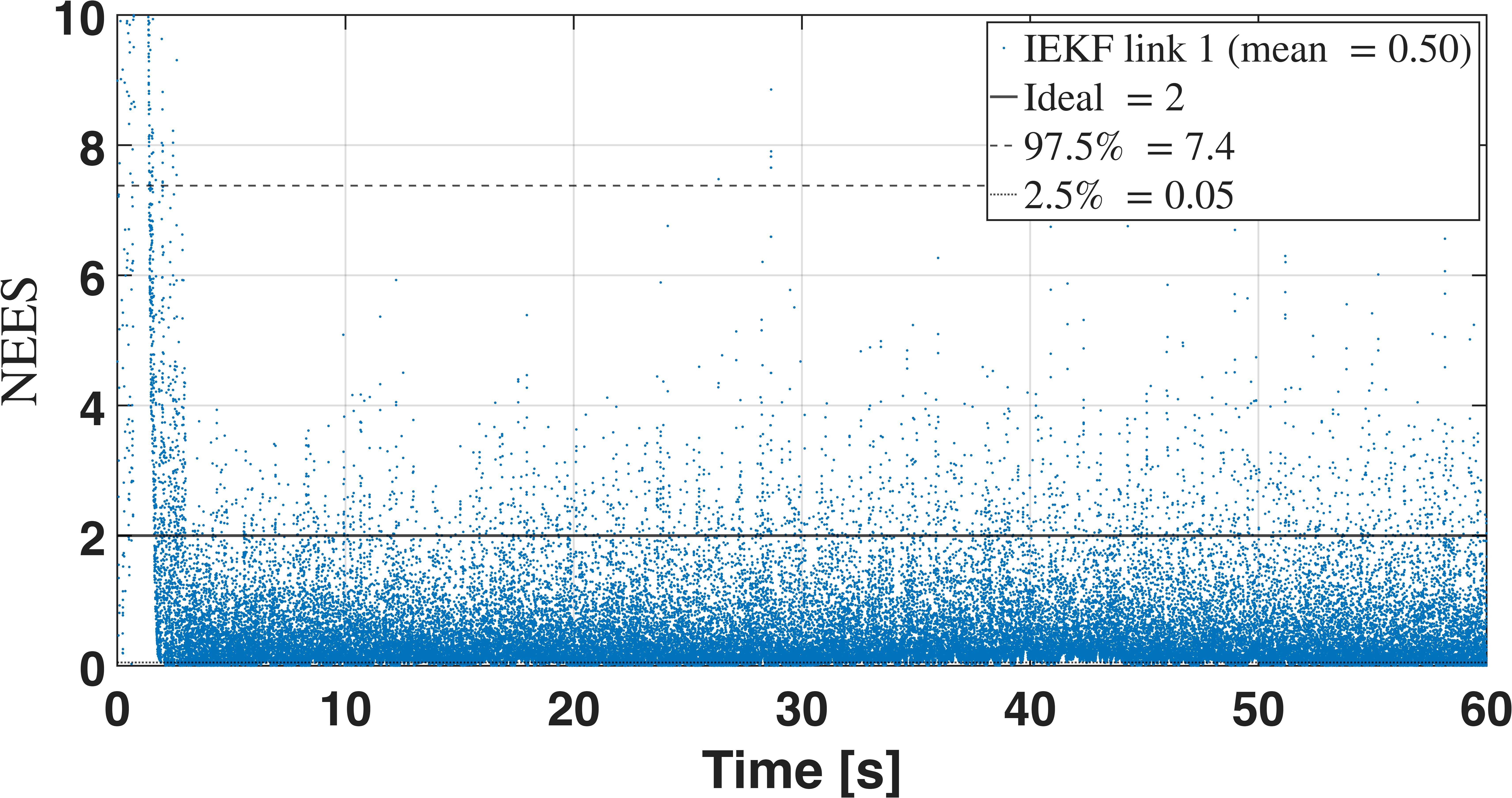}
  \label{fig:sil_nees1}}
\hfill
\subfloat[]{%
  \includegraphics[width=0.32\textwidth]{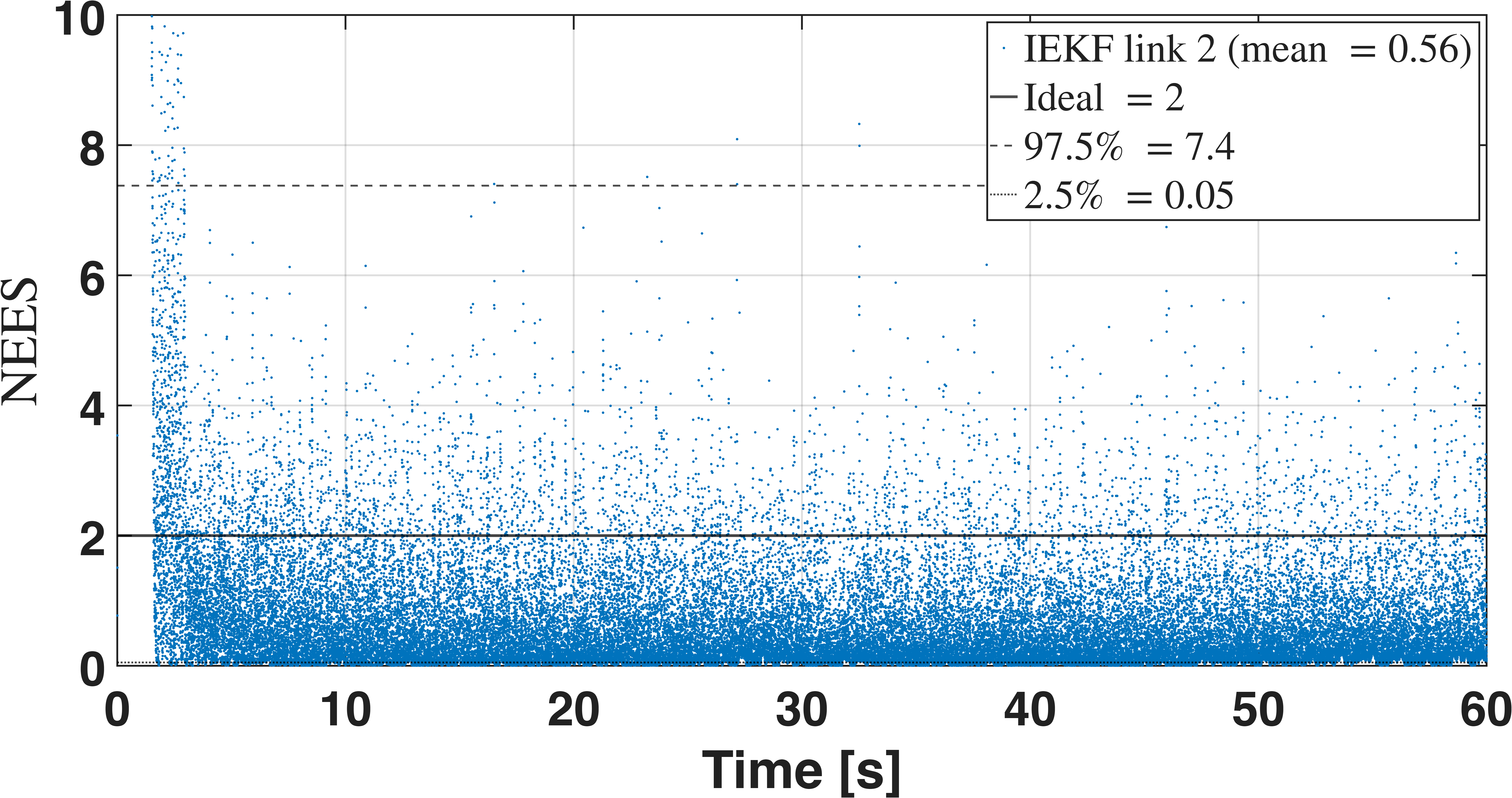}
  \label{fig:sil_nees2}}
\hfill
\subfloat[]{%
  \includegraphics[width=0.32\textwidth]{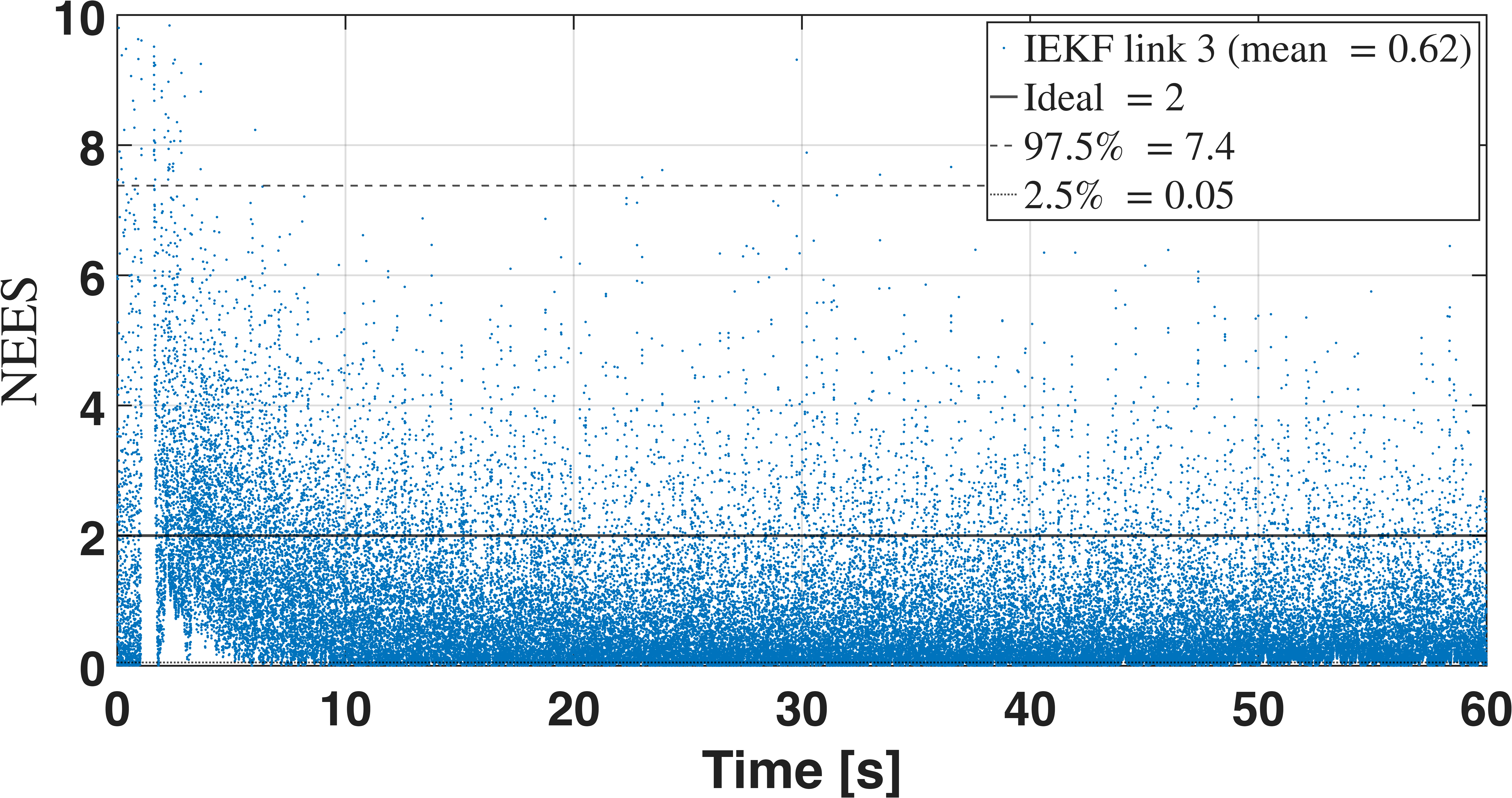}
  \label{fig:sil_nees3}}
\caption{NEES of the per-link IEKF on the $n_f = 2$ free error
  coordinates $[\phi_z,\,\tilde{\omega}_z]^\top$, encoder as reference,
  (a)~lift, (b)~tilt, (c)~breaker link. Solid: ideal ($2$); dashed:
  $97.5\%$ bound ($7.4$); dotted: $2.5\%$ bound ($0.05$).}
\label{fig:sil_nees}
\end{figure*}

\begin{table}[htbp]
\centering
\caption{Example~2 accuracy and gyroscope consistency: joint-angle RMSE
  in steady state ($t \geq 15$~s) and over the full run including the
  start-up transient, joint-rate RMSE in steady state, and mean NIS of
  the gyroscope innovation (ideal $1$).}
\label{tab:sil}
\renewcommand{\arraystretch}{1.3}
\begin{tabular}{l cc c c}
\toprule
& \multicolumn{2}{c}{RMSE [deg]} & RMSE [deg/s] & NIS \\
\cmidrule(lr){2-3}
Joint & $t \geq 15$~s & full & $t \geq 15$~s & gyro \\
\midrule
$\theta_1$ (lift)    & 0.045 & 0.24 & 0.095 & 0.90 \\
$\theta_2$ (tilt)    & 0.059 & 1.76 & 0.149 & 0.90 \\
$\theta_3$ (breaker) & 0.064 & 1.51 & 0.172 & 0.90 \\
\bottomrule
\end{tabular}
\end{table}

\section{Conclusion}
\label{sec:10}

An invariant extended Kalman filter on $SE(3)$ has been developed for
serial rigid manipulators with an arbitrary number of links. The
group-affine pose kinematics make the linearized error dynamics
autonomous; a physically separated noise model accounts for the
integrated accelerometer channel, for gyroscope noise propagated through
the Coriolis term and for gyroscope bias; and the per-link filters form
a modular chain, coupled only through the Adjoint-transformed upstream
posterior, with cost linear in the number of links. Exponential ultimate
boundedness in mean square holds per link and along the chain, with a
residual ball that grows with the Adjoint operator norms fixed by the
link geometry.

In a three-dimensional two-link simulation the filter outperforms a
coordinate EKF on all joints with a consistent covariance. In a
software-in-the-loop study of an industrial hydraulic rammer-boom
manipulator, replacing the encoder row by a gravity-referenced virtual
encoder yields IMU-only joint estimates with steady-state errors below
$0.07^{\circ}$ and gyroscope biases that converge to their true values.

Open directions are the extension to flexible links, where the elastic
deformation enters through the strain-rate PDE model of~\citep{Yaqubi2026};
additional measurement channels such as strain gauges or cameras, which
change only $h_i$ and $\mathbf{N}_i^{\mathrm{obs}}$; and explicit
conditions on geometry and motion that guarantee
Assumption~\ref{ass:observability}.

\bibliographystyle{elsarticle-harv}
\bibliography{references}

\appendix

\section{$SE(3)$ Logarithmic Map}
\label{app:logSE3}

For $\tilde{g}_i$ with rotation $\tilde{\mathbf{R}}_i =
\hat{\mathbf{R}}_i^\top\mathbf{R}_i$ and translation
$\tilde{\mathbf{p}}_i = \hat{\mathbf{R}}_i^\top(\mathbf{p}_i-\hat{\mathbf{p}}_i)$,
the closed-form logarithm~\citep{murray1994,lynch2017} is
\begin{align}
  \tilde{\boldsymbol{\xi}}_i^g =
  \begin{bmatrix}\boldsymbol{\phi}_i\\ \mathbf{J}_l^{-1}(\boldsymbol{\phi}_i)\,
  \tilde{\mathbf{p}}_i\end{bmatrix},
  \quad
  \boldsymbol{\phi}_i = \frac{\theta_i}{2\sin\theta_i}
  (\tilde{\mathbf{R}}_i - \tilde{\mathbf{R}}_i^\top)^\vee,
  \quad
  \theta_i = \cos^{-1}\!\tfrac{\mathrm{tr}(\tilde{\mathbf{R}}_i)-1}{2},
  \label{eq:xi_pose_expanded}
\end{align}
with
\begin{align}
  \mathbf{J}_l^{-1}(\boldsymbol{\phi}) =
  \frac{\theta/2}{\tan(\theta/2)}\mathbf{I}_3
  + \Big(1 - \frac{\theta/2}{\tan(\theta/2)}\Big)
    \frac{\boldsymbol{\phi}\boldsymbol{\phi}^\top}{\theta^2}
  - \tfrac{1}{2}[\boldsymbol{\phi}]^\times,
  \label{eq:Jlinv}
\end{align}
$\boldsymbol{\phi}_i = \mathbf{0}$ and $\mathbf{J}_l = \mathbf{I}_3$ at
$\theta_i = 0$, and the axis taken from the symmetric part of
$\tilde{\mathbf{R}}_i$ as $\theta_i \to \pi$.

\section{Proof of Theorem~\ref{thm:euubs}}
\label{app:proof_eub}
\textbf{Step 1: One-step conditional expectation of $V_i$.}
From~\eqref{eq:error_closed_loop}, dropping the
$O(\|\tilde{\boldsymbol{\xi}}_{i,k}\|^2)$ term (valid
under the small-error assumption), and using
$\mathbb{E}[\mathbf{w}_{i,k}] = \mathbf{0}$ and
independence of $\mathbf{w}_{i,k}$ from
$\tilde{\boldsymbol{\xi}}_{i,k}$:

\begin{align}
  \mathbb{E}\!\left[
    \tilde{\boldsymbol{\xi}}_{i,k+1}
    \tilde{\boldsymbol{\xi}}_{i,k+1}^\top
    \mid \tilde{\boldsymbol{\xi}}_{i,k}
  \right]
  =
  \mathbf{A}_{i,k}\boldsymbol{\Phi}_{i,k}
  \tilde{\boldsymbol{\xi}}_{i,k}
  \tilde{\boldsymbol{\xi}}_{i,k}^\top
  \boldsymbol{\Phi}_{i,k}^\top\mathbf{A}_{i,k}^\top
  +
  \mathbf{A}_{i,k}
  \tilde{\mathbf{Q}}_{d,i}^{\mathrm{eff}}
  \mathbf{A}_{i,k}^\top.
  \label{eq:cond_cov}
\end{align}

Substituting $V_{i,k+1} =
\tilde{\boldsymbol{\xi}}_{i,k+1}^\top
(\mathbf{P}_{i,k+1}^-)^{-1}
\tilde{\boldsymbol{\xi}}_{i,k+1}$,
expanding the quadratic form using~\eqref{eq:cond_cov},
applying the standard identity
$\mathbb{E}[\mathbf{w}^\top\mathbf{C}\mathbf{w}]
= \mathrm{tr}(\mathbf{C}\,\mathbb{E}[\mathbf{w}\mathbf{w}^\top])$
with cyclicity of the trace, and using
$\mathbb{E}[\mathbf{w}_{i,k}\mid\tilde{\boldsymbol{\xi}}_{i,k}]
= \mathbf{0}$ to eliminate cross-terms:
\begin{align}
  \mathbb{E}[V_{i,k+1} \mid \tilde{\boldsymbol{\xi}}_{i,k}]
  &= \tilde{\boldsymbol{\xi}}_{i,k}^\top
    \boldsymbol{\Phi}_{i,k}^\top\mathbf{A}_{i,k}^\top
    (\mathbf{P}_{i,k+1}^-)^{-1}
    \mathbf{A}_{i,k}\boldsymbol{\Phi}_{i,k}
    \tilde{\boldsymbol{\xi}}_{i,k} \notag\\
  &\quad
  + \mathrm{tr}\!\left(
      (\mathbf{P}_{i,k+1}^-)^{-1}
      \mathbf{A}_{i,k}
      \tilde{\mathbf{Q}}_{d,i}^{\mathrm{eff}}
      \mathbf{A}_{i,k}^\top
    \right).
  \label{eq:V_step}
\end{align}

\textbf{Step 2: Bounding the quadratic term (contraction).}
The discrete Riccati recursion with the optimal
simplified-form update $\mathbf{P}_{i,k}^+ =
\mathbf{A}_{i,k}\mathbf{P}_{i,k}^-$ gives:

\begin{align}
  \mathbf{P}_{i,k+1}^-
  = \boldsymbol{\Phi}_{i,k}\mathbf{A}_{i,k}
    \mathbf{P}_{i,k}^-
    \boldsymbol{\Phi}_{i,k}^\top
  + \tilde{\mathbf{Q}}_{d,i}^{\mathrm{eff}}.
  \label{eq:riccati_step}
\end{align}

Since $\tilde{\mathbf{Q}}_{d,i}^{\mathrm{eff}} \succeq
\mathbf{0}$, we have $\mathbf{P}_{i,k+1}^- \succeq
\boldsymbol{\Phi}_{i,k}\mathbf{A}_{i,k}\mathbf{P}_{i,k}^-
\mathbf{A}_{i,k}^\top\boldsymbol{\Phi}_{i,k}^\top$,
and therefore (reversing the positive-semidefinite
ordering for inverses):

\begin{align}
  (\mathbf{P}_{i,k+1}^-)^{-1}
  \preceq
  (\boldsymbol{\Phi}_{i,k}\mathbf{A}_{i,k}
   \mathbf{P}_{i,k}^-
   \mathbf{A}_{i,k}^\top
   \boldsymbol{\Phi}_{i,k}^\top)^{-1}.
  \label{eq:Pinv_bound}
\end{align}

Substituting into the quadratic term of~\eqref{eq:V_step}:

\begin{align}
  &\tilde{\boldsymbol{\xi}}_{i,k}^\top
    \boldsymbol{\Phi}_{i,k}^\top\mathbf{A}_{i,k}^\top
    (\mathbf{P}_{i,k+1}^-)^{-1}
    \mathbf{A}_{i,k}\boldsymbol{\Phi}_{i,k}
    \tilde{\boldsymbol{\xi}}_{i,k} \notag\\
  &\leq
  \tilde{\boldsymbol{\xi}}_{i,k}^\top
    \boldsymbol{\Phi}_{i,k}^\top\mathbf{A}_{i,k}^\top
    (\boldsymbol{\Phi}_{i,k}\mathbf{A}_{i,k}
     \mathbf{P}_{i,k}^-
     \mathbf{A}_{i,k}^\top
     \boldsymbol{\Phi}_{i,k}^\top)^{-1}
    \mathbf{A}_{i,k}\boldsymbol{\Phi}_{i,k}
    \tilde{\boldsymbol{\xi}}_{i,k} \notag\\
  &=
  \tilde{\boldsymbol{\xi}}_{i,k}^\top
    (\mathbf{P}_{i,k}^-)^{-1}
    \tilde{\boldsymbol{\xi}}_{i,k}
  = V_{i,k}.
  \label{eq:quad_bound_raw}
\end{align}

This gives $\mathbb{E}[V_{i,k+1} | \tilde{\boldsymbol{\xi}}_{i,k}]
\leq V_{i,k} + \beta_i$, which is boundedness but
not contraction. The contraction factor arises from the
information injected by the measurement at each step.
By the matrix inversion lemma applied to the
Joseph-form update~\eqref{eq:P_update}:
\begin{align}
  (\mathbf{P}_{i,k}^+)^{-1}
  = (\mathbf{P}_{i,k}^-)^{-1}
  + \mathbf{H}_{\mathrm{obs},i}^\top
    (\mathbf{N}_i^{\mathrm{obs}})^{-1}
    \mathbf{H}_{\mathrm{obs},i},
  \label{eq:info_update}
\end{align}
so $(\mathbf{P}_{i,k}^+)^{-1} \succeq
(\mathbf{P}_{i,k}^-)^{-1}$:
each measurement step strictly increases the information
matrix.
Propagating~\eqref{eq:info_update} forward $N$ steps
through $\boldsymbol{\Phi}_{i,k}$ and accumulating gives:
\begin{align}
  (\mathbf{P}_{i,k+N}^-)^{-1}
  \succeq
  (\mathbf{P}_{i,k}^-)^{-1}
  + \mathbf{W}_i(k,N),
  \label{eq:info_accumulate}
\end{align}
where $\mathbf{W}_i(k,N) \succeq \delta\mathbf{I}_{12}$
uniformly by Assumption~\ref{ass:observability}.
Since $\mathbf{P}_{i,k}^- \preceq \overline{p}\mathbf{I}$
by Assumption~\ref{ass:bounded},
combining~\eqref{eq:info_accumulate} with
the Riccati recursion~\eqref{eq:riccati_step} and
the identity $\mathbf{P}_{i,k}^+ = \mathbf{A}_{i,k}\mathbf{P}_{i,k}^-$
yields, via Theorem~1 of~\citet{reif1999},
that there exists $\alpha_i \in (0,1)$,
\begin{align}
  \alpha_i
  = \frac{\delta/N}{1/\overline{p} + \delta/N}
  > 0,
  \label{eq:alpha_explicit}
\end{align}
depending only on $\delta$, $\overline{p}$, and $N$,
such that the Riccati solution satisfies:
\begin{align}
  \boldsymbol{\Phi}_{i,k}^\top\mathbf{A}_{i,k}^\top
  (\mathbf{P}_{i,k+1}^-)^{-1}
  \mathbf{A}_{i,k}\boldsymbol{\Phi}_{i,k}
  \preceq
  (1-\alpha_i)\,(\mathbf{P}_{i,k}^-)^{-1}.
  \label{eq:contraction}
\end{align}

Substituting~\eqref{eq:contraction}
into~\eqref{eq:V_step}:

\begin{align}
  \mathbb{E}[V_{i,k+1} \mid \tilde{\boldsymbol{\xi}}_{i,k}]
  \leq (1-\alpha_i)\,V_{i,k}
  + \mathrm{tr}\!\left(
      (\mathbf{P}_{i,k+1}^-)^{-1}
      \mathbf{A}_{i,k}
      \tilde{\mathbf{Q}}_{d,i}^{\mathrm{eff}}
      \mathbf{A}_{i,k}^\top
    \right).
  \label{eq:V_decrease}
\end{align}

\textbf{Step 3: Bounding the trace term.}
Since $\mathbf{P}_{i,k}^- \preceq \overline{p}\mathbf{I}$
uniformly, $(\mathbf{P}_{i,k+1}^-)^{-1} \preceq
(\underline{p})^{-1}\mathbf{I}$. Using
$\|\mathbf{A}_{i,k}\| \leq 1 + \|\mathbf{K}_{i,k}\|
\|\mathbf{H}_{\mathrm{obs},i}\|$
which is uniformly bounded under
Assumption~\ref{ass:bounded}, define:

\begin{align}
  \beta_i
  \triangleq
  \frac{1}{\underline{p}}
  \sup_k
  \mathrm{tr}\!\left(
    \mathbf{A}_{i,k}
    \tilde{\mathbf{Q}}_{d,i}^{\mathrm{eff}}
    \mathbf{A}_{i,k}^\top
  \right)
  < \infty.
  \label{eq:beta_def}
\end{align}

\textbf{Step 4: Iteration.}
From~\eqref{eq:V_decrease}:

\begin{align}
  \mathbb{E}[V_{i,k+1} \mid \tilde{\boldsymbol{\xi}}_{i,k}]
  \leq (1-\alpha_i)\,V_{i,k} + \beta_i.
  \label{eq:V_recursion}
\end{align}

Taking unconditional expectations and iterating:

\begin{align}
  \mathbb{E}[V_{i,k}]
  \leq (1-\alpha_i)^k \mathbb{E}[V_{i,0}]
  + \frac{\beta_i}{\alpha_i}.
  \label{eq:V_bound}
\end{align}

\textbf{Step 5: Conversion to $\|\tilde{\boldsymbol{\xi}}_{i,k}\|^2$.}
Since $\mathbf{P}_{i,k}^- \preceq \overline{p}\mathbf{I}$:

\begin{align}
  \|\tilde{\boldsymbol{\xi}}_{i,k}\|^2
  \leq \overline{p}\,V_{i,k},
\end{align}

and since $\mathbf{P}_{i,0}^- \preceq \overline{p}\mathbf{I}$
and $(\mathbf{P}_{i,0}^-)^{-1} \preceq
\underline{p}^{-1}\mathbf{I}$:

\begin{align}
  V_{i,0}
  = \tilde{\boldsymbol{\xi}}_{i,0}^\top
    (\mathbf{P}_{i,0}^-)^{-1}
    \tilde{\boldsymbol{\xi}}_{i,0}
  \leq \frac{1}{\underline{p}}
    \|\tilde{\boldsymbol{\xi}}_{i,0}\|^2.
\end{align}

Combining with~\eqref{eq:V_bound} and using
$(1-\alpha_i)^k \leq e^{-\alpha_i k}$
gives~\eqref{eq:exp_bound} with
$c_i = \overline{p}/\underline{p}$.

\section{Gyroscope Bias Augmentation}
\label{app:bias}

MEMS gyroscopes exhibit a slowly varying bias that the white-noise
model~\eqref{eq:gyro_meas} does not capture. The gyroscope measurement
is extended to
\begin{align}
  \mathbf{y}_i^{\mathrm{gyro}}
  = \boldsymbol{\omega}_i + \mathbf{b}_{\omega,i}
  + \boldsymbol{\eta}_{\omega,i},
  \qquad
  \dot{\mathbf{b}}_{\omega,i} = \mathbf{w}_{b,i},
  \quad
  \mathbf{w}_{b,i} \sim \mathcal{N}(\mathbf{0},\,\mathbf{Q}_{b,i}),
  \label{eq:gyro_bias}
\end{align}
with $\mathbf{Q}_{b,i}$ the bias random-walk spectral density from the
sensor datasheet. The filter state becomes
$(g_i,\mathbf{V}_i,\mathbf{b}_{\omega,i})$ and the error state
\begin{align}
  \tilde{\boldsymbol{\xi}}_i^{b} =
  \big[(\tilde{\boldsymbol{\xi}}_i^g)^\top,\,
  \tilde{\mathbf{V}}_i^\top,\,
  \tilde{\mathbf{b}}_{\omega,i}^\top\big]^\top \in \mathbb{R}^{15},
  \qquad
  \tilde{\mathbf{b}}_{\omega,i} = \mathbf{b}_{\omega,i} - \hat{\mathbf{b}}_{\omega,i}.
  \label{eq:xi_bias}
\end{align}
The bias is a body-frame quantity that enters the measurement
additively, so the drift matrix~\eqref{eq:Fci} and the diffusion
matrix~\eqref{eq:Di} are extended by one block each,
\begin{align}
  \mathbf{F}_{c,i}^{b} =
  \begin{bmatrix}
    \mathbf{F}_{c,i} & \mathbf{0}_{12\times3} \\
    \mathbf{0}_{3\times12} & \mathbf{0}_{3\times3}
  \end{bmatrix},
  \qquad
  \mathbf{D}_i^{b} = \mathrm{blkdiag}\big(\mathbf{D}_i,\,\mathbf{Q}_{b,i}\big),
  \label{eq:F_bias}
\end{align}
and the observation matrix~\eqref{eq:Hobs_explicit} by the bias column
of the gyroscope rows,
\begin{align}
  \mathbf{H}_{\mathrm{obs},i}^{b} =
  \begin{bmatrix}
    \mathbf{0}_{6\times6} & \mathbf{I}_6 & \mathbf{E} \\[2pt]
    \mathbf{j}_{q,i}^\top & \mathbf{0}_{1\times6} & \mathbf{0}_{1\times3}
  \end{bmatrix},
  \qquad
  \mathbf{E} = \begin{bmatrix}\mathbf{I}_3\\ \mathbf{0}_{3\times3}\end{bmatrix}.
  \label{eq:H_bias}
\end{align}
The prediction~\eqref{eq:P_pred_expanded}, gain~\eqref{eq:K_i} and
Joseph update~\eqref{eq:P_update} apply unchanged to the
$15$-dimensional error, and the bias is corrected additively,
$\hat{\mathbf{b}}_{\omega,i}^+ = \hat{\mathbf{b}}_{\omega,i}^-
+ [\mathbf{K}_i\boldsymbol{\nu}_i]_{13:15}$.

\textbf{Error dynamics and modularity.}
$\mathbf{F}_{c,i}^{b}$ depends on $\hat{\mathbf{V}}_i$ only, so the
autonomy of the error dynamics (Proposition~\ref{prop:group_affine}) is
retained. The bias is local to each link and does not enter
$\mathbf{N}_i^{\mathrm{total}}$~\eqref{eq:Ntotal}, so the information
flow~\eqref{eq:info_flow} and Proposition~\ref{prop:modularity} are
unchanged.

\textbf{Observability and stability.}
The gyroscope rows of~\eqref{eq:H_bias} measure
$\tilde{\boldsymbol{\omega}}_i + \tilde{\mathbf{b}}_{\omega,i}$ and do
not separate the two. The bias becomes observable through the
joint-angle row, which constrains the integral of $\tilde{\mathbf{V}}_i$
through the coupling $\mathbf{I}_6$ in~\eqref{eq:Fci}.
Theorems~\ref{thm:euubs} and~\ref{thm:chained} therefore hold for
$\tilde{\boldsymbol{\xi}}_i^{b}$ whenever
Assumption~\ref{ass:observability} holds for the augmented pair
$(\mathbf{F}_{c,i}^{b},\mathbf{H}_{\mathrm{obs},i}^{b})$, which requires
the joint-angle row to be available over every window of $N$ steps.
When that row is temporarily unavailable, the bias is unobservable and its gain is set to
zero.

\end{document}